\documentclass{article}

\usepackage{microtype}
\usepackage{graphicx}
\usepackage{subcaption}
\usepackage{booktabs} 

\usepackage{hyperref}

 \usepackage[preprint]{icml2026}

\usepackage{amsmath}
\usepackage{amssymb}
\usepackage{mathtools}
\usepackage{amsthm}

\usepackage{makecell}

\newcommand{\eqdef}{\; { := }\;}

\newcommand{\R}{\mathbb{R}}

\newcommand{\E}{\mathbb{E}}
\newcommand{\C}{\mathcal{C}}

\newcommand{\cD}{\mathcal{D}}

\usepackage[capitalize,noabbrev]{cleveref}

\theoremstyle{plain}
\newtheorem{theorem}{Theorem}[section]

\newtheorem{lemma}[theorem]{Lemma}

\theoremstyle{definition}
\newtheorem{definition}[theorem]{Definition}
\newtheorem{assumption}[theorem]{Assumption}
\theoremstyle{remark}

\usepackage[textsize=tiny]{todonotes}

\icmltitlerunning{Communication-Efficient Gluon in Federated Learning}

\begin{document}

\twocolumn[
  \icmltitle{Communication-Efficient Gluon in Federated Learning}



  \icmlsetsymbol{equal}{*}

  \begin{icmlauthorlist}
    \icmlauthor{Xun Qian}{yyy}
    \icmlauthor{Alexander Gaponov}{yyy}
    \icmlauthor{Grigory Malinovsky}{yyy}
    \icmlauthor{Peter Richt\'{a}rik}{yyy}
 
  \end{icmlauthorlist}

  \icmlaffiliation{yyy}{Center of Excellence for Generative AI, King Abdullah University of Science and Technology, Thuwal, Saudi Arabia}

  \icmlcorrespondingauthor{Xun Qian}{xun.qian@kaust.edu.sa}

  \icmlkeywords{Muon, Compressed Gluon, MVR, Error Feedback}

  \vskip 0.3in
]



\printAffiliationsAndNotice{}  

\begin{abstract}
  Recent developments have shown that Muon-type optimizers based on linear minimization oracles (LMOs) over non-Euclidean norm balls have the potential to get superior practical performance than Adam-type methods in the training of large language models. Since large-scale neural networks are trained across massive machines, communication cost becomes the bottleneck. To address this bottleneck, we investigate Gluon, which is an extension of Muon under the more general layer-wise $(L^0, L^1)$-smooth setting, with both unbiased and contraction compressors. In order to reduce the compression error, we employ the variance reduced technique in SARAH in our compressed methods. The convergence rates and improved communication cost are achieved under certain conditions. As a byproduct, a new variance reduced algorithm with faster convergence rate than Gluon is obtained. We also incorporate momentum variance reduction (MVR) to these compressed algorithms and comparable communication cost is derived under weaker conditions when $L_i^1 \neq 0$. Finally, several numerical experiments are conducted to verify the superior performance of our compressed algorithms in terms of communication cost. 
\end{abstract}

\section{Introduction}

For the training of large-scale neural networks, Adam-type optimizers \cite{kinga2015method, reddi2019convergence, loshchilovdecoupled} have been the dominant methods for the past decade. However, this paradigm was challenged by the introduction of Muon \cite{jordan6muon}, which demonstrated superior empirical performance over AdamW \cite{loshchilovdecoupled} in training a 1.5B-parameter NanoGPT model. Subsequently, the efficiency advantage of Muon was further confirmed in training even larger Mixture-of-Experts language models with up to 16B parameters \cite{liu2025muon}. The theoretical convergence of Muon has since been established in several works \cite{li2025note, kovalev2025understanding, shen2025convergence}, inspiring numerous extensions \cite{pethick2025training, riabinin2025gluon, shulgin2025beyond, kovalev2025non}. Among these, Gluon \cite{riabinin2025gluon} is a notable layer-wise framework based on the Linear Minimization Oracle (LMO), designed to better capture the layer-wise geometry of neural networks.

With the rapid growth of both training data and model parameters in modern deep learning \cite{team2025kimi,comanici2025gemini}, single-machine training has become infeasible due to computational and memory constraints. This has necessitated the adoption of distributed training paradigms \cite{you2017large,konevcny2016federated}. The corresponding optimization task can be formulated as the following distributed non-convex problem.
\begin{equation}\label{p:p}
	\min_{X \in {\cal S}} \left\{  f(X) \eqdef \frac{1}{n} \sum_{\tau=1}^n f^\tau (X) \right\}, 
\end{equation}
where $f^\tau (X) \eqdef \E_{\xi_\tau \sim \cD_\tau} [f_{\xi_\tau} (X) ]$, $n$ is the number of workers, ${\cal S}$ is a $d$-dimensional vector space which can also be represented by the product space ${\cal S}_1 \otimes \cdots \otimes {\cal S}_p$,  $f_{\xi_\tau}: {\cal S} \to \R$ are potentially non-convex and non-smooth but continuously differentiable functions, and $\cD_\tau$ is the probability distribution of the training data stored on the $\tau$-th worker. We assume that $f(\cdot)$ has a lower bound, i.e., $\inf_{X \in {\cal S}} f(X) > -\infty$ in this work.

\paragraph{Compression} In the distributed setting, the communication of parameters and gradients can be much slower than the local computation for large-scale neural networks \cite{kairouz2021advances}. In order to remedy this communication bottleneck, compression is a well-known and efficient strategy \cite{seide20141,alistarh2017qsgd,qian2021error}. There are mainly two types of compressors:  the unbiased compressor and contraction compressor defined as follows. 
\begin{definition}[Unbiased compressor]\label{df:Qi}
	A random operator $Q_i: \mathcal{S}_i \to \mathcal{S}_i$ with the properties 
	\begin{equation}\label{eq:Qi}
		\E_{Q_i} [Q_i (X)] = X, \ \E_{Q_i} [\|Q_i(X)\|_2^2] \leq (\omega + 1) \|X\|_2^2, 
	\end{equation}
	for all $X \in \mathcal{S}_i$ is an $\omega$-compression operator. 
\end{definition}

\begin{definition}[Contraction compressor]\label{df:Ci}
	A random operator $\C_i: \mathcal{S}_i \to \mathcal{S}_i$ is called a {contraction compressor} if there exist a constant $0\leq \delta \leq 1$ such that for all $X \in \mathcal{S}_i$, 
	\begin{equation}\label{eq:Ci}
		\E \left[  \| X - \C_i(X)\|_2^2 \right] \leq (1-\delta) \|X\|_2^2. 
	\end{equation}
\end{definition}

Common unbiased compressors include random sparsification \cite{stich2018sparsified}, random dithering \cite{alistarh2017qsgd}, and natural compression \cite{horvoth2022natural}. RandK and TopK \cite{stich2018sparsified} are frequently used contraction compressors. For more examples, see \cite{beznosikov2023biased, demidovich2023guide}. When the contraction compressor is directly used, gradient descent algorithm could diverge \cite{beznosikov2023biased}, and error feedback mechanism \cite{seide20141,qian2021error,richtarik2021ef21} need to be used to guarantee the convergence.

\paragraph{Momentum variance reduction} MVR is first proposed in STORM \cite{cutkosky2019momentum} to get faster convergence rate than SGD with momentum. There are also several works that incorporate MVR into Muon-type methods, for instance, \cite{sfyraki2025lions,chang2025convergence,huang2025limuon,qian2025muon}, and the improved convergence rates are achieved under different conditions.

While there are several distributed muon-type methods \cite{liu2025muon,therien2025muloco,ahn2025dion} showing the superior empirical performance, the theoretical convergence is only guaranteed in \cite{gruntkowska2025error} as far as we know, where the contraction compressor and EF21 \cite{richtarik2021ef21} are applied to Muon. In this work, we aim to study distributed Gluon with both unbiased compressors and contraction compressors comprehensively. Our contributions are summarized below. 

\paragraph{Contributions}  1) We propose Compressed Gluon and Compressed Gluon with Error Feedback, with unbiased compressors and contraction compressors, respectively. The convergence rates and communication cost are analyzed. Specifically, the expected totally communication cost of them is ${\cal O}(\frac{\sqrt{n}d}{\epsilon^2})$ and ${\cal O}(\frac{nd}{\epsilon^2})$, respectively\footnote{We only list the dominant terms here. The full expressions can be found in corresponding sections.}, better than the ${\cal O}(\frac{nd}{\epsilon^3})$ cost of EF21-Muon with stochastic gradients in \cite{gruntkowska2025error}. Compressed Gluon reduces to VR-MARINA in certain specific settings, except for a different choice of step size, and achieves the same convergence rate as VR-MARINA in \cite{gorbunov2022marinafasternonconvexdistributed}. 

2) We incorporate MVR into Compressed Gluon and Compressed Gluon with Error Feedback. Comparable results are achieved, except that weaker conditions are needed when $L_i^1 \neq 0$. The communication cost and oracle complexity results for these four methods are summarized in Table \ref{tab:comparison}. 

3) We recover several new algorithms as special cases, including a new variance reduced algorithm with ${\cal O} (\frac{1}{n^{1/3}K^{1/3}})$ convergence rate, Local Gluon, and Local Gluon with Error Compensated MVR.

\begin{table*}[t] 
	\centering
	\caption{Summary of the results to reach the $\epsilon$-precision for the problem (\ref{p:p}), i.e., $\min_{k=0, ..., K-1} \sum_{i=1}^p t_i \E \left[  \|\nabla_i f(X^k) \|_{(i)\star} \right] \leq \epsilon$. Dependences on the numerical constants,``quality" of the starting point, smoothness constants, and Hessian variance constants, etc., are omitted in the complexity bounds. We assume the compression ratios of unbiased compressor and contraction compressor are $\Theta(\omega+1)$ and $\Theta(\frac{1}{\delta})$, respectively.  Abbreviations: ``Communication cost" = the total communication cost needed to reach the $\epsilon$-precision in expectation;  ``Oracle complexity" = the total number of (stochastic) first-order oracle calls needed to reach the $\epsilon$-precision in expectation. The value of parameter $c$ in the Communication cost is the same as in the corresponding Oracle complexity. }
	\label{tab:comparison}
	
	\small 
	
	\begin{tabular}{lcc}
		\toprule
		\textbf{Method} & \textbf{Communication Cost} & \textbf{Oracle Complexity} \\
		\midrule

		\textbf{Compressed Gluon} & ${\cal O}\left(  nd + \frac{1}{c^2 \epsilon^4 \sqrt{n}} + \frac{\sqrt{nd}}{\epsilon^2}  \right)$ & ${\cal O}\left(  \frac{1}{\epsilon^2} + (1 + \frac{1}{c^2}) \frac{1}{\epsilon^4 \sqrt{n}}  \right)$, for any $0< c^2 \leq \frac{1}{\epsilon^2 n^{3/2}}$ \\
		\midrule

		\textbf{Compressed Gluon with MVR} & ${\cal O}\left(  nd + \frac{1}{c^2 \epsilon^4 \sqrt{n}} + \frac{\sqrt{nd}}{\epsilon^2}  \right)$ & ${\cal O}\left(  \frac{1}{\epsilon^2} + (1 + \frac{1}{c^2}) \frac{1}{\epsilon^4 \sqrt{n}}  \right)$, for any $0< c^2 \leq \frac{1}{\epsilon^2 n^{3/2}}$ \\
		\midrule

		\textbf{Compressed Gluon with Error Feedback} & ${\cal O}\left(  \frac{1}{c^2 \epsilon^4}  + \frac{nd}{\epsilon^2}  \right)$& ${\cal O}\left(  \frac{1}{c^2 \epsilon^4}  +  \frac{c^2}{\epsilon^2} \right)$, for any $0< c^2 \leq \frac{1}{\epsilon^2 n}$ \\
		\midrule

		\textbf{\makecell[l]{\textbf{Compressed Gluon with Error } \\ \textbf{Feedback and MVR}}} & ${\cal O}\left(  \frac{1}{c^2 \epsilon^4}  + \frac{nd}{\epsilon^2}  \right)$& ${\cal O}\left(  \frac{1}{c^2 \epsilon^4}  +  \frac{c^2}{\epsilon^2} \right)$, for any $0< c^2 \leq \frac{1}{\epsilon^2 n}$ \\
		\bottomrule
	\end{tabular}
\end{table*}

\section{Compressed Gluon}

	First we introduce the assumptions. To explore the layer-wise geometry of neural networks, the layer-wise $(L^0, L^1)$-smoothness assumption \citep{riabinin2025gluon} is used. 
	
	\begin{assumption}\label{as:L0L1smooth}
		The function $f: \mathcal{S} \mapsto \mathbb{R}$ is layer-wise $(L^0, L^1)$-smooth with constants $L^0 := (L^0_1, \ldots, L^0_p) \in \mathbb{R}_+^p$ and $L^1 := (L^1_1, \ldots, L^1_p) \in \mathbb{R}_+^p$. That is, the inequality 
		\begin{eqnarray*}
			&&\quad  \|\nabla_i f(X) - \nabla_i f(Y)\|_{(i)\star} \\ 
			&& \leq \left(L^0_i + L^1_i \|\nabla_i f(X)\|_{(i)\star}\right) \|X_i - Y_i\|_{(i)}
			\label{eq:layer_smooth}
		\end{eqnarray*}
		holds for all $i = 1, \ldots, p$ and all $X = [X_1, \ldots, X_p] \in \mathcal{S}$, $Y = [Y_1, \ldots, Y_p] \in \mathcal{S}$, where $\|\cdot\|_{(i)\star}$ is the dual norm associated with $\|\cdot\|_{(i)}$, , $\nabla f(X) = [\nabla_1 f(X), \dots, \nabla_p f(X)]$. 
	\end{assumption}

	\begin{assumption}\label{as:Lismooth}
		There exists $L_i \geq 0$ such that 
		$$
		\| \nabla_i f^\tau(X) - \nabla_i f^\tau (Y) \|_2^2 \leq L_i^2 \|X_i-Y_i\|_{(i)}^2, 
		$$
		$\forall X, Y \in {\cal S}, \ i = 1, ..., p$, and $\tau = 1,..., n$. 
	\end{assumption}
	
	\begin{assumption}\label{as:boundedvariance}
		The stochastic gradient estimator $\nabla f_{\xi_\tau} : {\cal S} \to {\cal S}$ is unbiased and has bounded variance for each $\tau \in \{  1, ..., n  \}$. That is, $\E_{\xi_\tau \sim {\cal D}_\tau} \left[  \nabla f_{\xi_\tau} (X)  \right] = \nabla f^\tau (X)$ for all $X \in {\cal S}$ and there exists $\sigma \geq 0$ such that
		$$
		\E_{\xi_\tau \sim {\cal D}_\tau} \left[  \|\nabla_i f_{\xi_\tau} (X) - \nabla_i f^\tau(X) \|_2^2 \right] \leq \sigma^2,
		$$
		$\forall X \in {\cal S}, \ i = 1, ..., p$, and $\tau = 1,..., n$. 
	\end{assumption}
	
	\begin{assumption}\label{as:rho}
		$\| X\|_{(i)\star} \leq \rho\|X\|_2$ for any $X$ and $i \in \{1, ..., p\}$. 
	\end{assumption}

	We also make the following Hessian Variance assumption \cite{qian2025muon}. 
	\begin{assumption}\label{as:HV}
		There exists $\delta_i \geq 0$ for all $i \in \{  1, ..., p  \}$ and $\tau \in \{1, ..., n\}$ such that 
		\begin{align}
			&\E [ \| \nabla_i f_{\xi_\tau} (X)  - \nabla_i f_{\xi_\tau} (Y) - (\nabla_i f^\tau(X) - \nabla_i f^\tau(Y))\|_2^2 ]  \nonumber  \\ 
			&\leq \delta_i^2 \|X_i-Y_i\|_{(i)}^2, \label{eq:HV}
		\end{align}
		where $\| \cdot\|_2$ is the Euclidean norm. 
	\end{assumption}

	The above $\delta_i^2$ can be bounded by the expected smoothness constant and norm equivalence constant. In fact, in the Euclidean case, Assumption \ref{as:HV} reduces to the average ${\cal L}$-smoothness Assumption (Assumption 3.2 there) in \cite{gorbunov2022marinafasternonconvexdistributed}. However, we prefer to call it Hessian varaince assumption since the left-hand-side in (\ref{eq:HV}) can be estimated by the variance of the Hessian matrix.

	Gluon \cite{riabinin2025gluon} extends Muon \cite{jordan6muon} to the layer-wise case, and the momentum for layer $i$ is $M_i^k = \beta M_i^{k-1} + (1-\beta) \nabla_i f_{\xi^k}(X^k)$, where $\beta$ is the momentum decay factor, $\xi^k$ is randomly sampled from the data points, and $X^k$ is the current model parameter. In the distributed setting, where $n$ worker nodes are available, the momentum would be $M_i^k = \beta M_i^{k-1} + (1-\beta) \cdot \frac{1}{n} \sum_{\tau=1}^n  \nabla_i f_{\xi_{\tau}^k}(X^k)$, where the stochastic gradient is calculated on each node and then aggregated. To reduce the communication cost, the stochastic gradient can be compressed by compressor $Q_i$. However, the compression error could be large since $\nabla_i f_{\xi_{\tau}^k}(X^k)$ does not converge to zero generally. A natural way is to introduce two auxiliary vectors $v_{1, i}$ and $v_{2, i}$, such that $v_{1, i}$ is close to $\nabla_i f_{\xi_{\tau}^k}(X^k)$ and $\E[v_{1, i}] = \E[v_{2, i}]$. Then the compression error of $Q_i(\nabla_i f_{\xi_{\tau}^k}(X^k) - v_{1, i})$ could be small, and 
	$$
	g_i^{\tau, k} = Q_i(\nabla_i f_{\xi_{\tau}^k}(X^k) - v_{1, i}) + v_{2, i}
	$$
	can be an stochastic estimator of $\nabla_i f^{\tau} (X^k)$ since $\E [g_i^{\tau, k}] = \nabla_i f^{\tau} (X^k)$. Moreover, since the uncompressed vector $v_{2, i}$ need to be transferred, $v_{2, i}$ should be updated with low probability (controlled by $u^k$ in our case).

	The rest is how to choose $v_{1, i}$ and $v_{2, i}$. A possible choice is $v_{1, i} = \nabla_i f_{\xi_{\tau}^k}({\tilde W})$ and $v_{2, i} = \nabla_i f^{\tau} ({\tilde W})$, where ${\tilde W}$ is a checkpoint, which comes from the search direction in SVRG \cite{johnson2013accelerating} and loopless SVRG \cite{kovalev2020don}. However, a more suitable choice which we use is from the search direction in SARAH \cite{nguyen2017sarah} and PAGE \cite{li2021page}, where $v_{1, i} = \nabla_i f_{\xi_{\tau}^k}(X^{k-1})$, $v_{2, i} = g_i^{\tau, k-1}$, and $v_{2, i}$ is initialized with $\nabla_i f^{\tau}(X^0)$ and updated by the full gradient of $f^\tau ({\tilde W})$ at the checkpoint with low probability. The reason is that in this choice, the compression error of $Q_i(\nabla_i f_{\xi_{\tau}^k}(X^k) - v_{1, i})$ can be upper-bounded by some factor times $\|X^k - X^{k-1}\|$, which can be controlled directly by the radius in the linear minimization oracle.

	The full gradient is calculated in SARAH for finite-sum problem. While for the more general problem (\ref{p:p}) we considered, the calculation of full gradient may not be available. Hence, we use a large batch of stochastic gradients for $v_{2, i}$ when needed instead, which leads to Compressed Gluon, i.e., Algorithm \ref{alg:q-cgluon}. In the Euclidean case and with $p=1$, $t_i =1$, $\eta= \gamma \|M_i^k\|_2$ at the $k$-th step, and $\beta \equiv 0$, Algorithm  \ref{alg:q-cgluon} actually reduces to VR-MARINA \cite{gorbunov2022marinafasternonconvexdistributed}.  We have the following convergence results for Compressed Gluon.

	\begin{theorem}\label{th:cs-Gluon}
	Let Assumptions \ref{as:L0L1smooth}, \ref{as:Lismooth}, \ref{as:boundedvariance}, \ref{as:rho}, and \ref{as:HV} hold. Assume each $Q_i$ in Algorithm \ref{alg:q-cgluon} is the unbiased compressor satisfying (\ref{eq:Qi}).   Let $X^0, ..., X^{K-1}$ be the iterates of Algorithm \ref{alg:q-cgluon}, and $M_i^0 = \frac{1}{Bn} \sum_{\tau=1}^n \sum_{j=1}^B \nabla_i f_{\xi_{\tau, j}^k} (X^0)$. 
	
	1. If $L_i^1 =0$, then for $0<q\leq 1$, 
	\begin{eqnarray}
		&& \min_{k=0, ..., K-1} \sum_{i=1}^p t_i \E \left[  \|\nabla_i f(X^k) \|_{(i)\star} \right]  \nonumber \\ 
		&\leq&  \frac{\Delta^0}{\eta K} + \frac{2\sum_{i=1}^p t_i \rho \sigma}{\alpha K \sqrt{Bn}} +   \frac{4\sqrt{\alpha}\sum_{i=1}^pt_i \rho \sigma}{\sqrt{(2-\alpha)(\alpha+\beta q)} \sqrt{Bn}}  \nonumber  \\ 
		&& +    \frac{2\sqrt{2(1-q)\alpha} \sum_{i=1}^p\sqrt{(\omega+1)\delta_i^2 + \omega L_i^2} t_i^2 \rho \eta}{\sqrt{(2-\alpha)(\alpha+\beta q)} \sqrt{qn}}  \nonumber  \\ 
		&&  + \frac{2\sum_{i=1}^p L_i^0t_i^2 \eta}{\alpha} + \frac{\sum_{i=1}^p L_i^0 t_i^2 \eta}{2};  \label{eq:th-cs-Gluon-1}
	\end{eqnarray}
	for $q=0$, 
	\begin{eqnarray}
		&& \min_{k=0, ..., K-1} \sum_{i=1}^p t_i \E \left[  \|\nabla_i f(X^k) \|_{(i)\star} \right]  \nonumber \\ 
		&\leq&  \frac{\Delta^0}{\eta K} + \frac{2\sum_{i=1}^p t_i \rho \sigma}{\alpha K \sqrt{Bn}}  +   \frac{2\sqrt{2}\sum_{i=1}^pt_i \rho \sigma}{ \sqrt{Bn}}  \nonumber  \\ 
		&& +   \frac{2\sqrt{2} \sqrt{K}\sum_{i=1}^p\sqrt{(\omega+1)\delta_i^2 + \omega L_i^2} t_i^2 \rho \eta}{ \sqrt{n}}  \nonumber  \\ 
		&&  + \frac{2\sum_{i=1}^p L_i^0t_i^2 \eta}{\alpha} + \frac{\sum_{i=1}^p L_i^0 t_i^2 \eta}{2}.  \label{eq:th-cs-Gluon-2}
	\end{eqnarray}
	
	2. If $L_i^1 \neq 0$, we let $\frac{\eta}{\alpha} \leq \min_{i}  \frac{1}{5L_i^1 t_i}$. Then $\left(  \frac{2}{\alpha} + \frac{1}{2}  \right) L_i^1 t_i \eta \leq \frac{1}{2}$ for all $i$, and the inequalities (\ref{eq:th-cs-Gluon-1}) and (\ref{eq:th-cs-Gluon-2}) remain valid when their right-hand sides are multiplied by a factor $2$. 
\end{theorem}

\begin{algorithm}[H]
	\caption{Compressed Gluon}
	\label{alg:q-cgluon}
	\begin{algorithmic}[1]
		\STATE \textbf{Input:} Initial model parameters $X^0 = [X_1^0, \dots, X_p^0] \in \mathcal{S}$, momentum $M^0 = [M_1^0, \dots, M_p^0] \in \mathcal{S}$, momentum decay factors $\beta \in [0, 1)$ for all iterations $k \geq 0$, probability $q \in (0, 1]$, batch size $B$, stepsize parameter $\eta>0$, $u^0=1 \in \R$ 
		\FOR{ $k = 0, 1, 2, ..., K-1$}
		\FOR{ $\tau = 1, ..., n$} 
		\STATE $u^{k+1}_\tau = 0$ for $\tau = 2, ..., n$
		\STATE $
		u^{k+1}_1 = \left\{ \begin{array}{rl}
			1 & \mbox{ with probability $q$} \\
			0 &\mbox{ with probability $1-q$}
		\end{array} \right.
		$\\
		\IF{$u^k = 1$}
		\STATE Sample $\xi^k_{\tau, j} \sim \mathcal{D_\tau}$ independently for $j = 1, ..., B$
		\STATE $g_i^{\tau, k} = \frac{1}{B} \sum_{j=1}^B \nabla_i f_{\xi^k_{\tau, j}} (X^k)$ for $i = 1, ..., p$ 
		\STATE Send $g_i^{\tau, k}$ and $u^{k+1}_\tau$ to the other nodes
		\STATE Receive $g_i^{\tau, k}$ and $u^{k+1}_\tau$ from the other nodes
		\ELSE
		\STATE Sample $\xi^k_\tau \sim \mathcal{D_\tau}$ 
		\STATE  $y_i^{\tau, k} = Q_i\left(  \nabla_i f_{\xi^k_\tau} (X^k) - \nabla_i f_{\xi^k_\tau} (X^{k-1})  \right)$ for $i = 1, ..., p$ 
		\STATE Send $y_i^{\tau, k}$ and $u^{k+1}_\tau$ to the other nodes
		\STATE Receive $y_i^{\tau, k}$ and $u^{k+1}_\tau$ from the other nodes
		\ENDIF

		\ENDFOR
		\STATE 
		$
		g_i^k= \left\{ \begin{array}{cl}
			\frac{1}{n} \sum_{\tau=1}^n g_i^{\tau, k} & \mbox{ if $u^k = 1$} \\
			g_i^{k-1} + \frac{1}{n} \sum_{\tau=1}^n  y_i^{\tau, k}&\mbox{ otherwise}
		\end{array} \right.
		$
		\STATE $u^{k+1} = \sum_{\tau=1}^n u^{k+1}_\tau$
		\STATE Update momentum $M_i^k = \beta M_i^{k-1} + (1-\beta) g_i^k$ for layer $i$ 
		\STATE Choose adaptive stepsize/radius $t_i \eta > 0$ for layer $i$
		\STATE Update parameters for layer $i$ via LMO over $\mathcal{B}_i^k := \{X_i \in \mathcal{S}_i : \|X_i - X_i^k\|_{(i)} \leq t_i \eta\}$:
		\begin{equation}\label{eq:updateCon}
			X_i^{k+1} = \text{LMO}_{\mathcal{B}_i^k}(M_i^k) := \arg \min_{X_i \in \mathcal{B}_i^k} \langle M_i^k, X_i \rangle_{(i)}
		\end{equation}
		\STATE Update full parameter vector $X^{k+1} = [X_1^{k+1}, \dots, X_p^{k+1}]$
		\ENDFOR
	\end{algorithmic}
\end{algorithm}

From Theorem \ref{th:cs-Gluon}, we can get the following convergence rate under suitable choices of parameters. We defer the analysis of communication cost to Subsection \ref{subsec:comc-cG}. 

\begin{theorem}\label{th:cs-Gluon-rate}
	Under the premise of Theorem \ref{th:cs-Gluon}, let $\alpha = \Theta(1)$, $q>0$, and $\frac{\Delta^0}{\eta^2 K} = {\cal L}_1 \eqdef  \sum_{i=1}^p L_i^0t_i^2 + \frac{\sqrt{1-q}}{\sqrt{qn}} \sum_{i=1}^p \rho \sqrt{(\omega+1)\delta_i^2 + \omega L_i^2}t_i^2$. Then the upper-bound in (\ref{eq:th-cs-Gluon-1}) becomes 
	$$
	{\cal O} \left(  \frac{\sqrt{\Delta^0 {\cal L}_1} }{\sqrt{K}}  +  \frac{\sum_it_i \rho \sigma}{\sqrt{Bn}}  \right). 
	$$
\end{theorem}

\paragraph{Comparison with VR-MARINA} 

As mentioned earlier, Compressed Gluon reduces to VR-MARINA in certain specific settings, except for a different choice of step size. Next we show that in the Euclidean case, we also recover the convergence rate of VR-MARINA in \cite{gorbunov2022marinafasternonconvexdistributed}. In fact, when $p=1$, $\|\cdot\|_{(i)} = \|\cdot\|_2$, $t_i=1$, $\rho=1$, and under the parameter setting in Theorem \ref{th:cs-Gluon-rate}, to reach the precision $\min_{k=0, ..., K-1} \E [\|\nabla f(X^k)\|_2] \leq \epsilon$, it is sufficient to choose $B = \Theta(\frac{\sigma^2}{n\epsilon^2})$ and $K = {\cal O} \left(  \frac{\Delta^0}{\epsilon^2} \left(  L_i^0 + \frac{\sqrt{1-q}}{\sqrt{qn}} \sqrt{\omega L_i^2 + (\omega+1) \delta_i^2}  \right)  \right)$, which is the same as that of VR-MARINA (Theorem 3.2 there) in \cite{gorbunov2022marinafasternonconvexdistributed}. Our results can be extended to the minibatch case easily. In fact, if we choose random samples uniformly from ${\cal D}_\tau$ for $u^k \neq 1$ with minibatch size $b$, then the results in Theorems \ref{th:cs-Gluon} and \ref{th:cs-Gluon-rate} also hold by replacing $\delta_i^2$ with $\frac{\delta_i^2}{b}$.

\subsection{Gluon in the Minibatch Case}
Let $q=1$ and $B=1$ in Algorithm \ref{alg:q-cgluon}. Then we recover Gluon with batch size $n$. If $L_i^1 = 0$, from Theorem \ref{th:cs-Gluon}, we can get 
\begin{eqnarray}
	&& \min_{k=0, ..., K-1} \sum_{i=1}^p t_i \E \left[  \|\nabla_i f(X^k) \|_{(i)\star} \right] \nonumber \\ 
	&\leq&  \frac{\Delta^0}{\eta K} + \frac{2\sum_{i=1}^p t_i \rho \sigma}{\alpha K \sqrt{n}} +   \frac{4\sqrt{\alpha}\sum_{i=1}^pt_i \rho \sigma}{\sqrt{(2-\alpha)} \sqrt{n}} \nonumber \\ 
	&&   + \frac{2\sum_{i=1}^p L_i^0t_i^2 \eta}{\alpha} + \frac{\sum_{i=1}^p L_i^0 t_i^2 \eta}{2}. \label{eq:Gluon-minibatch}
\end{eqnarray}
It is easy to verify that the upper-bound in (\ref{eq:Gluon-minibatch}) is at least of order $\frac{1}{(Kn)^{1/4}}$. By choosing $\eta = n^{\frac{1}{4}} K^{-\frac{3}{4}}$ and $\alpha = n^{\frac{1}{2}} K^{-\frac{1}{2}}$, we can obtain 
\begin{align}
	& \min_{k=0, ..., K-1} \sum_{i=1}^p t_i \E \left[  \|\nabla_i f(X^k) \|_{(i)\star} \right] \nonumber \\ 
	&\leq  \frac{\Delta^0}{n^{1/4} K^{1/4}} + \frac{2\sum_{i=1}^p t_i \rho \sigma}{n \sqrt{K}} +   \frac{4 \sum_{i=1}^pt_i \rho \sigma}{\sqrt{(2-\alpha)} n^{1/4}K^{1/4}} \nonumber \\
	&  + \frac{2\sum_{i=1}^p L_i^0t_i^2 }{n^{1/4}K^{1/4}} + \frac{\sum_{i=1}^p L_i^0 t_i^2 n^{1/4}}{2K^{3/4}}. \label{eq:Gluon-minibatch-2}
\end{align}

	\paragraph{Linear speed up with batch size $n$}

For $n\leq K$, from (\ref{eq:Gluon-minibatch-2}), we know that to reach the precision $ \min_{k=0, ..., K-1} \sum_{i=1}^p t_i \E \left[  \|\nabla_i f(X^k) \|_{(i)\star} \right] \leq \epsilon$, it is sufficient to choose $K = \frac{1}{\epsilon^4 n}$ \footnote{We omit the constant factors for simplicity, a convention that applies throughout the rest of the paper.}. From $n\leq K$, we have $n \leq \frac{1}{\epsilon^2}$. Thus, we have linear speed up for $1\leq n \leq \frac{1}{\epsilon^2}$, and in this case, $\eta = n^{\frac{1}{4}} K^{-\frac{3}{4}} = n \epsilon^3$ and $\alpha = n^{\frac{1}{2}} K^{-\frac{1}{2}} = n\epsilon^2$. For $n\geq K$, by choosing $\eta = K^{-\frac{1}{2}}$ and $\alpha = \Theta(1)$, from (\ref{eq:Gluon-minibatch}), we have $\min_{k=0, ..., K-1} \sum_{i=1}^p t_i \E \left[  \|\nabla_i f(X^k) \|_{(i)\star} \right]  \leq  {\cal O} \left(  \frac{1}{ K^{1/2}}  \right)$, and to reach the $\epsilon$ precision, we can choose $K = \frac{1}{\epsilon^2}$. 

For $n \leq K$, the stochastic gradient oracle is $Kn = \frac{1}{\epsilon^4}$. Assume the data sampled from distribution $\mathcal{D_\tau}$ each time is different, which is reasonable for a sufficient large data distribution. Then the number of data $N$ should be $Kn$, which is $ \frac{1}{\epsilon^4}$.

If $L_i^1 \neq 0$, from Theorem \ref{th:cs-Gluon}, we can obtain the above results similarly as long as $K n\geq \left( \max_i 5L_i^1t_i \right)^4$ for $n\leq K$ and $K \geq \left( \max_i 5L_i^1t_i \right)^2$ for $n\geq K$.

\subsection{A New Variance Reduced Algorithm with ${\cal O}\left(  \frac{1}{n^{1/3}K^{1/3}}  \right)$ Convergence Rate}\label{subsec:new-vr-alg}

We consider the uncompressed case, i.e., $\omega=0$. First we let $qB=1$ such that the expected stochastic gradient oracle is $qB + 2(1-q) \leq 3$ at each step on each working node. If $L_i^1=0$, then by choosing $\eta = \frac{n^{1/3}}{K^{2/3}}$, $q = \frac{n^{1/3}}{K^{2/3}}$, $B = \frac{1}{q} = \frac{K^{2/3}}{n^{1/3}}$, and $\alpha \geq \min\{  \frac{n^{2/3}}{K^{1/3}}, 1 \}$, from Theorem \ref{th:cs-Gluon}, we can get 
\begin{align*}
	& \min_{k=0, ..., K-1} \sum_{i=1}^p t_i \E \left[  \|\nabla_i f(X^k) \|_{(i)\star} \right] \\ 
	&\leq \frac{\Delta^0}{n^{1/3}K^{1/3}} + \max\left\{  \frac{2\sum_{i=1}^p t_i \rho \sigma}{n^{1/3}K^{4/3}},  \frac{2\sum_{i=1}^p t_i \rho \sigma}{nK}  \right\} \\ 
	& + \frac{4\sum_{i=1}^p t_i \rho \sigma}{n^{1/3}K^{1/3}} + \frac{2\sqrt{2} \sum_{i=1}^p \delta_i t_i^2\rho}{n^{1/3}K^{1/3}} + \frac{n^{1/3}\sum_{i=1}^p L_i^0 t_i^2}{2K^{2/3}}   \\ 
	& + 2\sum_{i=1}^p L_i^0 t_i^2 \max\left\{ \frac{n^{1/3}}{K^{2/3}},  \frac{1}{n^{1/3}K^{1/3}}  \right\}. 
\end{align*}

\paragraph{Linear speed up with $n$} From the above inequality, for $n\leq \sqrt{K}$, to reach $\epsilon$ precision, it is sufficient to choose $K = \frac{1}{\epsilon^3 n}$, and $n\leq \sqrt{K}$ is equivalent to $n \leq \frac{1}{\epsilon}$. In this case, $\eta = q = n \epsilon^2$, $B = \frac{1}{n\epsilon^2}$, and $\alpha \geq n\epsilon$. For $n\geq \sqrt{K}$, from Theorem \ref{th:cs-Gluon}, by choosing $\eta = q = K^{-\frac{1}{2}}$, $B = \sqrt{K}$, and $\alpha = \Theta(1)$, we have $\min_{k=0, ..., K-1} \sum_{i=1}^p t_i \E \left[  \|\nabla_i f(X^k) \|_{(i)\star} \right]  \leq  {\cal O} \left(  \frac{1}{ K^{1/2}}  \right)$, and to reach the $\epsilon$ precision, we can choose $K = \frac{1}{\epsilon^2}$. 

For $n \leq \sqrt{K}$, the stochastic gradient oracle is $Kn = \frac{1}{\epsilon^3}$. Assume the data sampled from distribution $\mathcal{D_\tau}$ each time is different. Then the number of data $N$ should be $Kn$, which is $ \frac{1}{\epsilon^3}$. 

If $L_i^1 \neq 0$, we can get the above results as long as $\frac{\eta}{\alpha} \leq \min_i \frac{1}{5L_i^1 t_i}$, i.e., $Kn  \geq  (\max_i 5L_i^1 t_i)^3$ for $n\leq \sqrt{K}$, and $K \geq \left( \max_i 5L_i^1t_i \right)^2$ for $n\geq \sqrt{K}$.

\subsection{Communication Cost Analysis}\label{subsec:comc-cG}

In this subsection, we consider the compressed case. For the rand-r compressor \cite{stich2018sparsified}, i.e., the random sparsification with sparsity parameter $r$, we have $\omega = \frac{d}{r} - 1$. The compression ratio for rand-r compressor is $\frac{d}{r} = \omega+1$. Hence in the following analysis, we assume the compression ratio of the unbiased compressor $Q_i$ in Definition \ref{df:Qi} is $\Theta(\omega+1)$. Then the expected communication cost of each node at each step is 
\begin{equation}\label{eq:comcost-k-cg}
	\Theta\left(n \left(1 + qd + \frac{(1-q)d}{\omega+1} \right) \right), 
\end{equation}
where $1$ comes from the communication of $u_\tau^{k+1}$, and $d$ is the dimension of $ \mathcal{S}$. If $q=\Theta(1)$, then (\ref{eq:comcost-k-cg}) is $\Theta(nd)$, which is the same as that of uncompressed algorithms. Hence we assume $(1-q) = \Theta(1)$ in the compressed case. 

It can be  shown that from the upper-bounds in Theorem \ref{th:cs-Gluon}, the expected totally communication cost is at least ${\cal O}\left(  \frac{\sqrt{n}d}{\epsilon^2}  \right)$ when  $n\leq \frac{1}{\epsilon^4}$, which generally holds in practice. We give the parameter settings that can achieve this communication cost here. The detailed anallysis can be found in the Appendix.

For the case where $L_i^1 = 0$, by choosing $q = \frac{cn^{1/4}}{K^{1/2}}$, $\omega = \frac{1}{q} - 1$, $\eta = \frac{\sqrt{c}n^{3/8}}{K^{3/4}}$, $B = \frac{cK^{1/2}}{n^{1/4}}$, and $\alpha \geq \max\{  \frac{cn^{3/4}}{K^{1/2}},  K^{-1}  \}$ for $0<c \leq  \frac{K^{1/2}}{n^{3/4}}$, we have $ \min_{k=0, ..., K-1} \sum_{i=1}^p t_i \E \left[  \|\nabla_i f(X^k) \|_{(i)\star} \right] \leq  {\cal O} \left(  \frac{1}{\sqrt{c} n^{3/8}K^{1/4}} \right)$. To achieve the $\epsilon$ precision by Algorithm \ref{alg:q-cgluon}, it is sufficient to choose $K = \frac{1}{c^2 \epsilon^4 n^{3/2}}$, and the expected totally communication cost is 
\begin{equation}\label{eq:comcost-cG}
\Theta\left(  nd + \frac{1}{c^2 \epsilon^4 \sqrt{n}} + \frac{\sqrt{n}d}{\epsilon^2}  \right). 
\end{equation}
In this case, $q=c^2n\epsilon^2$, $B = \frac{1}{\epsilon^2 n}$, $\eta = c^2 n^{\frac{3}{2}} \epsilon^3$, $\alpha \geq \max\{  cn^{\frac{3}{2}}\epsilon^2,  c^2n^{\frac{3}{2}}\epsilon^4 \}$, and the expected stochastic gradient oracle is $\Theta\left(  \frac{1}{\epsilon^2} +  \left( 1 + \frac{1}{c^2}\right) \frac{1}{\epsilon^4 \sqrt{n}} \right)$. 
The dominate term in (\ref{eq:comcost-cG}) is $\frac{\sqrt{n}d}{\epsilon^2}$ for large $d$ in LLMs. For the uncompressed algorithm in subsection \ref{subsec:new-vr-alg}, the totally communication cost is $\Theta\left(  \frac{d}{\epsilon^3}  \right)$, which is larger than that in (\ref{eq:comcost-cG}) when $n \leq \frac{1}{\epsilon^2}$.

For the case where $L_i^1 \neq 0$, the above results also hold as long as $\frac{\eta}{\alpha} \leq \min_{i}  \frac{1}{5L_i^1 t_i}$, which generally holds for large $K$ in practice. 

\subsection{The Special Case Where $L_i^0=0$}

It is shown in \cite{riabinin2025gluon} that, $L_i^0$ is close to zero in practice. Hence we are interested to see how the communication cost changes in the special case where $L_i^0=0$. When $\frac{\sqrt{n}d}{\epsilon^2}  \geq \frac{n^{2/3}d}{\epsilon^{5/3} c_0^{1/3}}$, i.e., $n \leq \frac{c_0^2}{\epsilon^2}$, where $c_0 \eqdef \min_i \frac{1}{5L_i^1 t_i}$, the expected totally communication cost can be reduced to 
$
\Theta\left( nd + \frac{1}{\epsilon^3 c_0} + \frac{n^{2/3}d}{\epsilon^{5/3} c_0^{1/3}}  \right)
$. The detailed analysis is available in the Appendix.

\section{Compressed Gluon with MVR}

Momentum variance reduction is proposed in STORM \cite{cutkosky2019momentum} to improve the convergence rate for non-convex problems, where the term $\beta (\nabla_i f_{\xi^k}(X^k) - \nabla_i f_{\xi^k}(X^{k-1}))$ is added to the search direction in SGD with momentum. Since we calculate $(\nabla_i f_{\xi_{\tau}^k}(X^k) - \nabla_i f_{\xi_{\tau}^k}(X^{k-1}))$ on each node in Compressed Gluon, it is convenient to incorporate MVR into Compressed Gluon. There are only two revisions. One is that the compressed vector $y_i^{\tau, k} = Q_i\left(  \nabla_i f_{\xi^k_\tau} (X^k) - \nabla_i f_{\xi^k_\tau} (X^{k-1})  \right)$ is always calculated on each node and aggregated. The other is that $\beta y_i^k$ is added to the momentum $M_i^k$, where $y_i^k$ is the average of $y_i^{\tau, k}$. Then we get Algorithm \ref{alg:q-cgluon-mvr}, i.e., Compressed Gluon with MVR.

For limited space, we summarize the main results here. The algorithm and detailed analysis can be found in the Appendix. We only list the results for the $L_i^1=0$ case. For the case where $L_i^1 \neq 0$, the results also hold as long as $\eta \leq \min_{i}  \frac{1}{L_i^1 t_i}$, which generally holds for large $K$ in practice, and this condition is weaker than the one in Compressed Gluon, i.e., $\frac{\eta}{\alpha} \leq \min_{i}  \frac{1}{5L_i^1 t_i}$. 

\paragraph{Gluon-MVR-1 in the minibatch case} By setting $q=1$, $B=1$,  $\omega=0$, and $\xi_{\tau, 1}^k \equiv \xi_\tau^k$, we recover the Gluon-MVR-1 \cite{qian2025muon} with batch size $n$. When $n \leq \sqrt{K}$, the iteration complexity is ${\cal O} \left(  \frac{1}{n^{1/3} K^{1/3}}  \right)$, which implies linear speed up with respect to the batch size $n$.

\paragraph{Communication cost} For the case where $L_i^0=0$, by choosing $q = \frac{cn^{1/4}}{K^{1/2}}$, $\omega = \frac{1}{q}-1$, $\eta = \frac{\sqrt{c}n^{3/8}}{K^{3/4}}$, $B = \frac{cK^{1/2}}{n^{1/4}}$, and $\alpha \geq \max\{ q, K^{-1}\}$ for $0<c \leq \frac{K^{1/2}}{n^{3/4}}$, the expected totally communication cost is 
\begin{equation}\label{eq:comcost-cG-mvr}
 \Theta\left( nd +  \frac{1}{c^2 \epsilon^4 \sqrt{n}} + \frac{\sqrt{n}d}{\epsilon^2}  \right), 
\end{equation}
which is the same as (\ref{eq:comcost-cG}).  In this case, $K = \frac{1}{c^2 \epsilon^4 n^{3/2}}$, $q=c^2n\epsilon^2$, $B = \frac{1}{\epsilon^2 n}$, $\eta = c^2 n^{\frac{3}{2}} \epsilon^3$, $\alpha \geq \max\{  c^2n \epsilon^2,  c^2n^{\frac{3}{2}}\epsilon^4 \}$, and the expected stochastic gradient oracle is $\Theta\left(  \frac{1}{\epsilon^2}  +  \left( 1 + \frac{1}{c^2}\right) \frac{1}{\epsilon^4 \sqrt{n}} \right)$.

\section{Compressed Gluon with Error Feedback}

In addition to the unbiased compressor used in Compressed Gluon, there is another type of compressor, i.e., the contraction compressor in Definition \ref{df:Ci}. Contraction compressor includes some well known compressors, such as TopK compressor \cite{stich2018sparsified}, which could be more efficient than the unbiased compressor in practice. However, contraction compressor could be biased and lead to diverge when directly used for gradient decent \cite{beznosikov2023biased}. Fortunately, the error compensation or error feedback mechanism \cite{seide20141} can resolve this issue. The key point in this mechanism is to use a vector ($e_i^{\tau, k}$ in our case) to store the compression error in each step, and add this vector to the transferred vector before compression in next step. Replacing the unbiased compressor $Q_i$ in Compressed Gluon with the contraction compressor $\C_i$ and applying the error feedback mechanism give rise to Algorithm \ref{alg:q-deltagluon} which we call Compressed Gluon with Error Feedback. It should be noticed that when $u^k=1$, the uncompressed vectors are communicated and there are no compression error in this case.  Thus, $e_i^{\tau, k+1}$ will be reset to zero when $u^k=1$.

For Compressed Gluon with Error Feedback, we can obtain the following convergence results. 

\begin{theorem}\label{th:deltaGluon}
	Let Assumptions \ref{as:L0L1smooth}, \ref{as:Lismooth}, \ref{as:boundedvariance}, \ref{as:rho}, and \ref{as:HV} hold. Assume each $\C_i$ in Algorithm \ref{alg:q-deltagluon} is the contraction compressor satisfying (\ref{eq:Ci}) and $q + \delta - q\delta >0$.   Let $X^0, ..., X^{K-1}$ be the iterates of Algorithm \ref{alg:q-deltagluon}, and $M_i^0 = \frac{1}{Bn} \sum_{\tau=1}^n \sum_{j=1}^B \nabla_i f_{\xi_{\tau, j}^k} (X^0)$. 
	
	1. If $L_i^1 =0$, then for $0<q\leq 1$, 
	\begin{align}
		& \min_{k=0, ..., K-1} \sum_{i=1}^p t_i \E \left[  \|\nabla_i f(X^k) \|_{(i)\star} \right]  \label{eq:th-deltaGluon-1} \\ 
		&\leq  \frac{\Delta^0}{\eta K} + \frac{2\sum_{i=1}^p t_i \rho \sigma}{\alpha K \sqrt{Bn}} +   \frac{4\sqrt{\alpha}\sum_{i=1}^pt_i \rho \sigma}{\sqrt{(2-\alpha)(\alpha+\beta q)} \sqrt{Bn}}  \nonumber  \\ 
		& +    \frac{2\sqrt{2(1-q)\alpha} \sum_{i=1}^p\delta_i t_i^2 \rho \eta}{\sqrt{(2-\alpha)(\alpha+\beta q)} \sqrt{qn}}  \nonumber \\ 
		&  + \frac{2\sum_{i=1}^p L_i^0t_i^2 \eta}{\alpha}   + \frac{\sum_{i=1}^p L_i^0 t_i^2 \eta}{2} \nonumber \\ 
		& + \frac{2\sqrt{2(1-q)(1-\delta)} \sum_i \sqrt{L_i^2 + (q+\delta-q\delta)\delta_i^2 }  t_i^2 \rho \eta}{q + \delta -q\delta}; \nonumber
	\end{align}
	for $q=0$, 
	\begin{align}
		& \min_{k=0, ..., K-1} \sum_{i=1}^p t_i \E \left[  \|\nabla_i f(X^k) \|_{(i)\star} \right]   \label{eq:th-deltaGluon-2}  \\ 
		&\leq   \frac{\Delta^0}{\eta K} + \frac{2\sum_{i=1}^p t_i \rho \sigma}{\alpha K \sqrt{Bn}} +  \frac{2\sqrt{2}\sum_{i=1}^p t_i \rho \sigma}{\sqrt{Bn}}  \nonumber   \\ 
		& +  \frac{2\sqrt{2}\sqrt{K} \sum_{i=1}^p \delta_i t_i^2 \rho \eta}{\sqrt{n}}    + \frac{2\sum_{i=1}^p L_i^0t_i^2 \eta}{\alpha} \nonumber  \\ 
		&  + \frac{\sum_{i=1}^p L_i^0 t_i^2 \eta}{2}  + \frac{2\sqrt{2(1-\delta)} \sum_{i=1}^p \sqrt{L_i^2 + \delta\delta_i^2 }  t_i^2 \rho \eta}{\delta}.  \nonumber 
	\end{align}
	
	2. If $L_i^1 \neq 0$, we let $\frac{\eta}{\alpha} \leq \min_{i}  \frac{1}{5L_i^1 t_i}$. Then $\left(  \frac{2}{\alpha} + \frac{1}{2}  \right) L_i^1 t_i \eta \leq \frac{1}{2}$ for all $i$, and the inequalities (\ref{eq:th-deltaGluon-1}) and (\ref{eq:th-deltaGluon-2}) remain valid when their right-hand sides are multiplied by a factor $2$. 
	
\end{theorem}

From Theorem \ref{th:deltaGluon}, we can get the following convergence rate under suitable choices of parameters. We defer the analysis of communication cost to Subsection \ref{subsec:comcost-dG}. 

\begin{theorem}\label{th:deltaGluon-rate}
Under the premise of Theorem \ref{th:deltaGluon}, let $\alpha = \Theta(1)$, $q>0$, and $\frac{\Delta^0}{\eta^2 K} = {\cal L}_3$, where ${\cal L}_3 \eqdef  \sum_{i=1}^p L_i^0t_i^2 + \frac{\sqrt{1-q}}{\sqrt{qn}} \sum_{i=1}^p \rho \delta_i t_i^2 + \frac{\sqrt{(1-q)(1-\delta)}}{q+\delta-q\delta} \cdot \sum_{i=1}^p \rho \sqrt{L_i^2 + (q+\delta-q\delta)\delta_i^2}t_i^2$. Then the upper-bound in (\ref{eq:th-deltaGluon-1}) becomes 
$$
{\cal O} \left(  \frac{\sqrt{\Delta^0 {\cal L}_3} }{\sqrt{K}}  +  \frac{\sum_it_i \rho \sigma}{\sqrt{Bn}}  \right). 
$$
\end{theorem}

\subsection{Local Gluon}

In this subsection, we consider the case where $\C_i (X_i) \equiv 0$ for all $X_i \in {\cal S}_i$, which indicates that $\delta \equiv 0$. From the update rule of $g_i^k$ in Algorithm \ref{alg:q-deltagluon}, it is easy to see that in this case we actually do not need to calculate $y_i^{\tau, k}$ and $e_i^{\tau, k}$. First we consider the case where $L_i^1=0$. By choosing $\eta = \frac{n^{1/4}}{K^{3/4}}$, $q = \frac{n^{1/2}}{K^{1/2}}$, $B = \frac{K^{1/2}}{n^{1/2}}$, and $\alpha \geq \max\{  \frac{n^{1/2}}{K^{1/2}},  K^{-1}  \}$ (We assume $K\geq n$, which generally holds in practice), from Theorem \ref{th:deltaGluon}, we have 
$$
\min_{k=0, ..., K-1} \sum_{i=1}^p t_i \E \left[  \|\nabla_i f(X^k) \|_{(i)\star} \right]  \leq {\cal O} \left(  \frac{1}{n^{1/4} K^{1/4}}  \right). 
$$
To achieve the $\epsilon$ precision, it is sufficient to choose $K = \frac{1}{\epsilon^4 n}$, and the expected totally communication cost is 
$$
\Theta\left( nd + Kn \left(1 + qd  \right) \right)  = \Theta\left(  \frac{1}{ \epsilon^4 } + \frac{n d}{\epsilon^2}  \right). 
$$
In this case, $n\leq K$ is equivalent to $n\leq \frac{1}{\epsilon^2}$, $\eta = n\epsilon^3$, $q=n \epsilon^2$, $B = \frac{1}{\epsilon^2 n}$, $\alpha \geq n\epsilon^2$, and the expected stochastic gradient oracle is $\Theta(\frac{1}{\epsilon^4})$. Assume the data sampled from distribution $\mathcal{D_\tau}$ each time is different. Then the number of data $N$ should be $ \frac{1}{\epsilon^4}$.
In this parameter setting, if we omit the communication cost of $u_\tau^k$, then the worker nodes only need to communicate with each other when $u^k = 1$. Hence we call Algorithm \ref{alg:q-deltagluon} in this parameter setting Local Gluon.

For the case where $L_i^1 \neq 0$, the above results also hold as long as $\frac{\eta}{\alpha} \leq \min_{i}  \frac{1}{5L_i^1 t_i}$, which generally holds for large $K$ in practice.

\subsection{Communication Cost Analysis}\label{subsec:comcost-dG}

In this subsection, we consider the compressed case. For the RandK contraction compressor \cite{qian2021error}, where $\delta = K/d$, the compression ratio is $\Theta(1/\delta)$. Thus we assume the compression ratio of $\C_i$ is $\Theta(1/\delta)$. Then the expected communication cost of each node at each step is 

\begin{algorithm}[H]
	\caption{Compressed Gluon with Error Feedback}
	\label{alg:q-deltagluon}
	\begin{algorithmic}[1]
		\STATE \textbf{Input:} Initial model parameters $X^0 = [X_1^0, \dots, X_p^0] \in \mathcal{S}$, momentum $M^0 = [M_1^0, \dots, M_p^0] \in \mathcal{S}$, momentum decay factors $\beta \in [0, 1)$ for all iterations $k \geq 0$, probability $q \in (0, 1]$, batch size $B$, stepsize parameter $\eta>0$, $u^0=1 \in \R$, $e_i^{\tau, 1} = 0 \in {\cal S}_i$
		\FOR{ $k = 0, 1, 2, ..., K-1$}
		\FOR{ $\tau = 1, ..., n$} 
		\STATE  $u^{k+1}_\tau = 0$ for $\tau = 2, ..., n$
		\STATE $
		u^{k+1}_1 = \left\{ \begin{array}{rl}
			1 & \mbox{ with probability $q$} \\
			0 &\mbox{ with probability $1-q$}
		\end{array} \right.
		$\\
		\IF{$u^k = 1$}
		\STATE $e_i^{\tau, k+1} = 0$ for $i = 1, ..., p$ 
		\STATE Sample $\xi^k_{\tau, j} \sim \mathcal{D_\tau}$ independently for $j = 1, ..., B$
		\STATE $g_i^{\tau, k} = \frac{1}{B} \sum_{j=1}^B \nabla_i f_{\xi^k_{\tau, j}} (X^k)$ for $i = 1, ..., p$ 
		\STATE Send $g_i^{\tau, k}$ and $u^{k+1}_\tau$ to the other nodes
		\STATE Receive $g_i^{\tau, k}$ and $u^{k+1}_\tau$ from the other nodes
		\ELSE
		\STATE Sample $\xi^k_\tau \sim \mathcal{D_\tau}$ 
		\STATE  $y_i^{\tau, k} = \C_i\left(  \nabla_i f_{\xi^k_\tau} (X^k) - \nabla_i f_{\xi^k_\tau} (X^{k-1} )+ e_i^{\tau, k}  \right)$ for $i = 1, ..., p$ 
		\STATE $e_i^{\tau, k+1} = e_i^{\tau, k} + \nabla_i f_{\xi^k_\tau} (X^k) - \nabla_i f_{\xi^k_\tau} (X^{k-1}) - y_i^{\tau, k}$ for $i = 1, ..., p$ 
		\STATE Send $y_i^{\tau, k}$ and $u^{k+1}_\tau$ to the other nodes
		\STATE Receive $y_i^{\tau, k}$ and $u^{k+1}_\tau$ from the other nodes
		\ENDIF
		
		\ENDFOR
		\STATE 
		$
		g_i^k= \left\{ \begin{array}{cl}
			\frac{1}{n} \sum_{\tau=1}^n g_i^{\tau, k} & \mbox{ if $u^k = 1$} \\
			g_i^{k-1} + \frac{1}{n} \sum_{\tau=1}^n  y_i^{\tau, k}&\mbox{ otherwise}
		\end{array} \right.
		$
		\STATE $u^{k+1} = \sum_{\tau=1}^n u^{k+1}_\tau$
		\STATE Update momentum $M_i^k = \beta M_i^{k-1} + (1-\beta) g_i^k$ for layer $i$ 
		\STATE Choose adaptive stepsize/radius $t_i \eta > 0$ for layer $i$
		\STATE Update parameters for layer $i$ via LMO over $\mathcal{B}_i^k := \{X_i \in \mathcal{S}_i : \|X_i - X_i^k\|_{(i)} \leq t_i \eta\}$:
		\begin{equation}\label{eq:updateCon-dG}
			X_i^{k+1} = \text{LMO}_{\mathcal{B}_i^k}(M_i^k) := \arg \min_{X_i \in \mathcal{B}_i^k} \langle M_i^k, X_i \rangle_{(i)}
		\end{equation}
		\STATE Update full parameter vector $X^{k+1} = [X_1^{k+1}, \dots, X_p^{k+1}]$
		\ENDFOR
	\end{algorithmic}
\end{algorithm}

\begin{equation}\label{eq:comc-dG-each}
	\Theta\left(n \left(1 + qd + (1-q)\delta d \right) \right), 
\end{equation}
where $1$ comes from the communication of $u_\tau^{k+1}$, and $d$ is the dimension of $ \mathcal{S}$. If $q=\Theta(1)$ or $\delta = \Theta(1)$, then (\ref{eq:comc-dG-each}) is $\Theta(nd)$, which is the same as that of uncompressed algorithms. Hence we assume $(1-q)(1-\delta) = \Theta(1)$ in the compressed case. From Theorem \ref{th:deltaGluon}, the upper-bound of $\min_{k=0, ..., K-1} \sum_{i=1}^p t_i \E \left[  \|\nabla_i f(X^k) \|_{(i)\star} \right] $ is at least 
$$
{\cal O} \left(  \frac{1}{\eta K}  +  \frac{\eta}{q+\delta-q\delta}  \right) \geq {\cal O} \left(  \frac{1}{\sqrt{(q+\delta-q\delta) K}}  \right). 
$$
Hence, to achieve $\epsilon$ precision, $K$ is at least $\frac{1}{\epsilon^2 (q+\delta-q\delta)}$, and  the expected totally communication cost is at least 
$$
\Theta\left(  \frac{n}{\epsilon^2 (q+\delta-q\delta) } + \frac{n d}{\epsilon^2}  \right). 
$$

Next we consider the datailed parameter settings. First we consider the case where $L_i^1=0$. By choosing $\eta = \frac{\sqrt{c} n^{1/4}}{K^{3/4}}$, $\delta = \frac{cn^{1/2}}{K^{1/2}}$, $B = \frac{cK^{1/2}}{n^{1/2}}$, $\frac{c^2}{K} \leq q \leq \delta$, and $\alpha \geq \{  q, \frac{1}{K}, \frac{cn^{1/2}}{K^{1/2}}  \}$ for $0<c \leq \frac{K^{1/2}}{n^{1/2}}$, from Theorem \ref{th:deltaGluon}, we can obtain 
$
\min_{k=0, ..., K-1} \sum_{i=1}^p t_i \E \left[  \|\nabla_i f(X^k) \|_{(i)\star} \right]  \leq {\cal O} \left(  \frac{1}{\sqrt{c} n^{1/4} K^{1/4}}  \right). 
$
To achieve the $\epsilon$ precision, it is sufficient to choose $K = \frac{1}{c^2 \epsilon^4 n}$, and the expected totally communication cost is $$\Theta\left(  \frac{1}{ c^2 \epsilon^4 } + \frac{n d}{\epsilon^2}  \right). $$
Compared to the totally communication cost of $\Theta\left(  \frac{d}{\epsilon^3}  \right)$ for the uncompressed algorithm in subsection \ref{subsec:new-vr-alg}, the above communication cost is better when $n \leq \frac{1}{\epsilon}$. In this parameter setting, the batch size $B = \frac{1}{\epsilon^2 n}$ and by choosing $q = \frac{c^2}{K}$, the expected stochastic gradient oracle is 
$$
\Theta\left(  Bn + (1-q + qB)Kn \right) = \Theta \left(  \frac{1}{c^2 \epsilon^4}  +  \frac{c^2}{\epsilon^2}  \right), 
$$
with $0<c^2 \leq \frac{1}{\epsilon^2 n}$. If $n \leq \frac{1}{\epsilon}$,  then by choosing $c^2 = \frac{1}{\epsilon} = (Kn)^{1/3}$, we have $q=n\epsilon^2$, $B=\frac{1}{\epsilon^2 n}$, $\eta=n\epsilon^2$, $\delta=n\epsilon$, $\alpha \geq n \epsilon$, $K = \frac{1}{\epsilon^3 n}$, and the above expected stochastic gradient oracle becomes $\Theta \left(  \frac{1}{\epsilon^3}  \right)$. 

When $q=0$, by choosing $\eta = \frac{\sqrt{c} n^{1/4}}{K^{3/4}}$, $\delta = \frac{cn^{1/2}}{K^{1/2}}$, $B = \frac{cK^{1/2}}{n^{1/2}}$, and $\alpha \geq \max\{ \frac{1}{K}, \frac{cn^{1/2}}{K^{1/2}}  \}$ for $0<c \leq \min \{\frac{K^{1/2}}{n^{1/2}}, 1\}$, from Theorem \ref{th:deltaGluon}, we can get 
$
\min_{k=0, ..., K-1} \sum_{i=1}^p t_i \E \left[  \|\nabla_i f(X^k) \|_{(i)\star} \right]  \leq {\cal O} \left(  \frac{1}{\sqrt{c} n^{1/4} K^{1/4}}  \right). 
$
To achieve the $\epsilon$ precision, it is sufficient to choose $K = \frac{1}{c^2 \epsilon^4 n}$, and the expected totally communication cost is 
$$
\Theta\left( nd + Kn \left( \delta d  \right) \right)  = \Theta\left(  \frac{n d}{\epsilon^2}  \right). 
$$
The batch size $B = \frac{1}{\epsilon^2 n}$ and the expected stochastic gradient oracle is $\Theta\left(  \frac{1}{c^2 \epsilon^4}  \right)$. 

For the case where $L_i^1 \neq 0$, the above results also hold as long as $\frac{\eta}{\alpha} \leq \min_{i}  \frac{1}{5L_i^1 t_i}$, which generally holds for large $K$ in practice.

\paragraph{Comparison with EF21-Muon} From Theorem 6 in \cite{gruntkowska2025error}, the convergence rate is at least ${\cal O}(\frac{1}{K^{1/4}})$. To achieve this rate, $\delta$ need to be at least $\Theta(\frac{1}{K^{1/4}})$. Thus, to obtain the $\epsilon$ precision, $K = \frac{1}{\epsilon^4}$ and $\delta$ is at least $\Theta(\epsilon)$, which implies that the expected totally communication cost is at least $\Theta(Knd\delta) = \Theta(\frac{nd}{\epsilon^3})$.

\section{Compressed Gluon with Error Feedback and MVR}

Similar to Compressed Gluon with MVR, we can incorporate MVR into Compressed Gluon with Error Feedback easily. For limited space, we summarize the main results here. The algorithm and detailed analysis can be found in the Appendix. We only list the results for the $L_i^1=0$ case. For the case where $L_i^1 \neq 0$, the results also hold as long as $\eta \leq \min_{i}  \frac{1}{L_i^1 t_i}$, which generally holds for large $K$ in practice. 

\paragraph{Local Gluon with Error Compensated MVR} For the resulting algorithm with $\C_i (X_i) \equiv 0$ for all $X_i \in {\cal S}_i$, we call it the Local Gluon with Error Compensated MVR. By choosing $\eta = \frac{n^{1/3}}{K^{2/3}}$, $q = \frac{n^{2/3}}{K^{1/3}}$, $B = \frac{K^{1/3}}{n^{2/3}}$, and $\alpha = \frac{n^{1/3}}{K^{2/3}}$, when $n \leq \sqrt{K}$,  
$
\min_{k=0, ..., K-1} \sum_{i=1}^p t_i \E \left[  \|\nabla_i f(X^k) \|_{(i)\star} \right]  \leq {\cal O} \left(  \frac{1}{n^{1/3} K^{1/3}}  \right)
$, and the expected totally communication cost is $\Theta\left(  \frac{1}{ \epsilon^3 } + \frac{n d}{\epsilon^2}  \right)$,  the expected stochastic gradient oracle is $\Theta(\frac{1}{\epsilon^3})$. 

\paragraph{Communication cost} By choosing $\eta = \frac{\sqrt{c} n^{1/4}}{K^{3/4}}$, $\delta = \frac{cn^{1/2}}{K^{1/2}}$, $B = \frac{cK^{1/2}}{n^{1/2}}$, $\frac{c^2}{K} \leq q \leq \delta$, and $\alpha \geq \{  q, \frac{1}{K}  \}$ for $0<c \leq \frac{K^{1/2}}{n^{1/2}}$, the expected totally communication cost is $\Theta\left(  \frac{1}{ c^2 \epsilon^4 } + \frac{n d}{\epsilon^2}  \right)$. Compared to the totally communication cost of $\Theta\left(  \frac{d}{\epsilon^3}  \right)$ for the uncompressed algorithm in subsection \ref{subsec:new-vr-alg}, the above communication cost is better when $n \leq \frac{1}{\epsilon}$. By further choosing $q = \frac{c^2}{K}$, the expected stochastic gradient oracle is $\Theta \left(  \frac{1}{c^2 \epsilon^4}  +  \frac{c^2}{\epsilon^2}  \right)$, which becomes $\Theta \left(  \frac{1}{\epsilon^3}  \right)$ if we choose $c^2 = \frac{1}{\epsilon} = (Kn)^{1/3}$ and $n \leq \frac{1}{\epsilon}$. More parameter setting where $\approx \Theta\left(  \frac{n d}{\epsilon^2}   \right)$ expected totally communication cost can be achieved can be found in the Appendix.

\section{Experiments}

\subsection{General Details}

Algorithm \ref{alg:q-cgluon} recovers non-Euclidean SGD with momentum ($q=1$) and VR-MARINA \cite{gorbunov2022marinafasternonconvexdistributed} with momentum (by setting $B$ to compute the full gradient and replacing the LMO step with a standard GD step), which serve as our baselines. We employ unbiased random sparsification Rand$K\%$ (Alg \ref{alg:q-cgluon}, VR-MARINA) and contractive greedy compression Top$K\%$ (Alg \ref{alg:q-deltagluon}) compressors. While we assume a theoretical $K\%$ communication cost for both, practically this is exact only for Rand$K\%$, as it avoids explicit index transmission.

We adopt the standard cost model where uplink communication is the primary bottleneck, neglecting downlink and compute. We report training loss against normilized communication cost, which is proportional to the number of bytes in the uplink. Additionally, we empirically improve Algorithms \ref{alg:q-cgluon} and \ref{alg:q-deltagluon} by scaling $y_i^{\tau,k}$ by $1/B$ when $u^k=0$.

\subsection{Logistic Regression}
\label{section:main_logreg}

We train\footnote{Code is available \href{https://anonymous.4open.science/r/CompressedGluonFL-D197}{here}.} a logistic regression on the a5a dataset \cite{a5a} using cross-entropy loss, weight decay $10^{-4}$, batch size $64$, and $\tau=4$. We evaluate block sizes $B \in \{1, 4, 16\}$ and random compressor densities $\{100\%, 10\%, 1\%\}$. Using the Euclidean norm for the LMO reduces the uncompressed baseline ($q=1$) to Normalized SGD with momentum (using $B=1$). We report $f(x) - f^*$ using optimum loss, found via the gradient descent. Figure \ref{fig:main_logreg} shows that Algorithm \ref{alg:q-cgluon} outperforms baselines (tuned in Figure \ref{fig:logreg_marina_ablation}) with a $\approx60\%$ communication cost reduction.

We present detailed ablation in the Appendix \ref{appendix:logreg}. The key observation is that setting $q<1$ leads to a significant slowdown in the method's convergence in terms of per-iteration loss (in comparison to the $q=1$ baseline) however, this is compensated for by the tolerance to the compression technique. Setting a larger value $B$ is crucial for accelerating the convergence of Algorithms \ref{alg:q-cgluon}, \ref{alg:q-deltagluon} and VR-MARINA. Note that the computationally intensive large batch is processed only with probability $q$.

\begin{figure}[t!]
    \centering
    
    \includegraphics[width=0.82\linewidth]{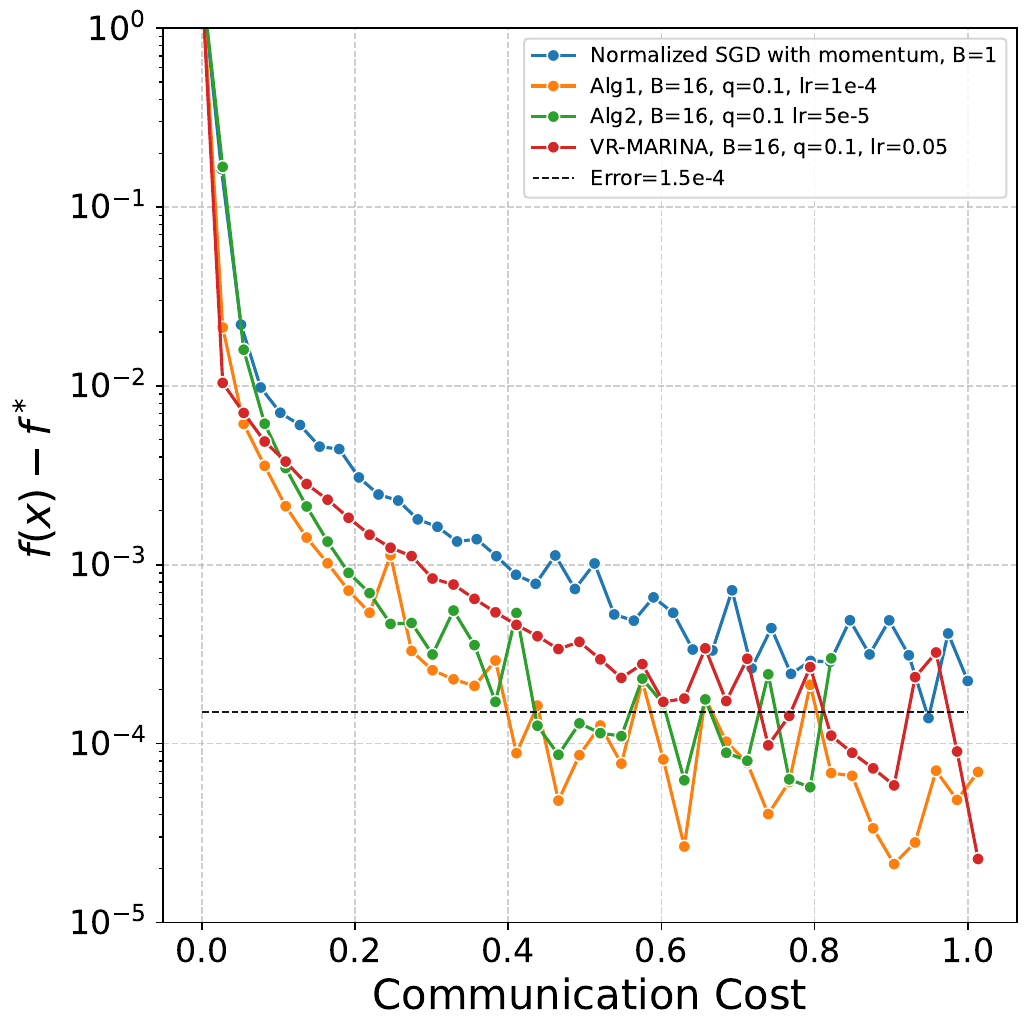}
    \caption{Logreg, a5a. $f(x)-f^*$ vs. communication cost. Algorithms \ref{alg:q-cgluon}, \ref{alg:q-deltagluon} and VR-MARINA utilize compression parameter $K=1\%$}
    \label{fig:main_logreg}
    
    \vspace{1.5em}
    
    \includegraphics[width=0.82\linewidth]{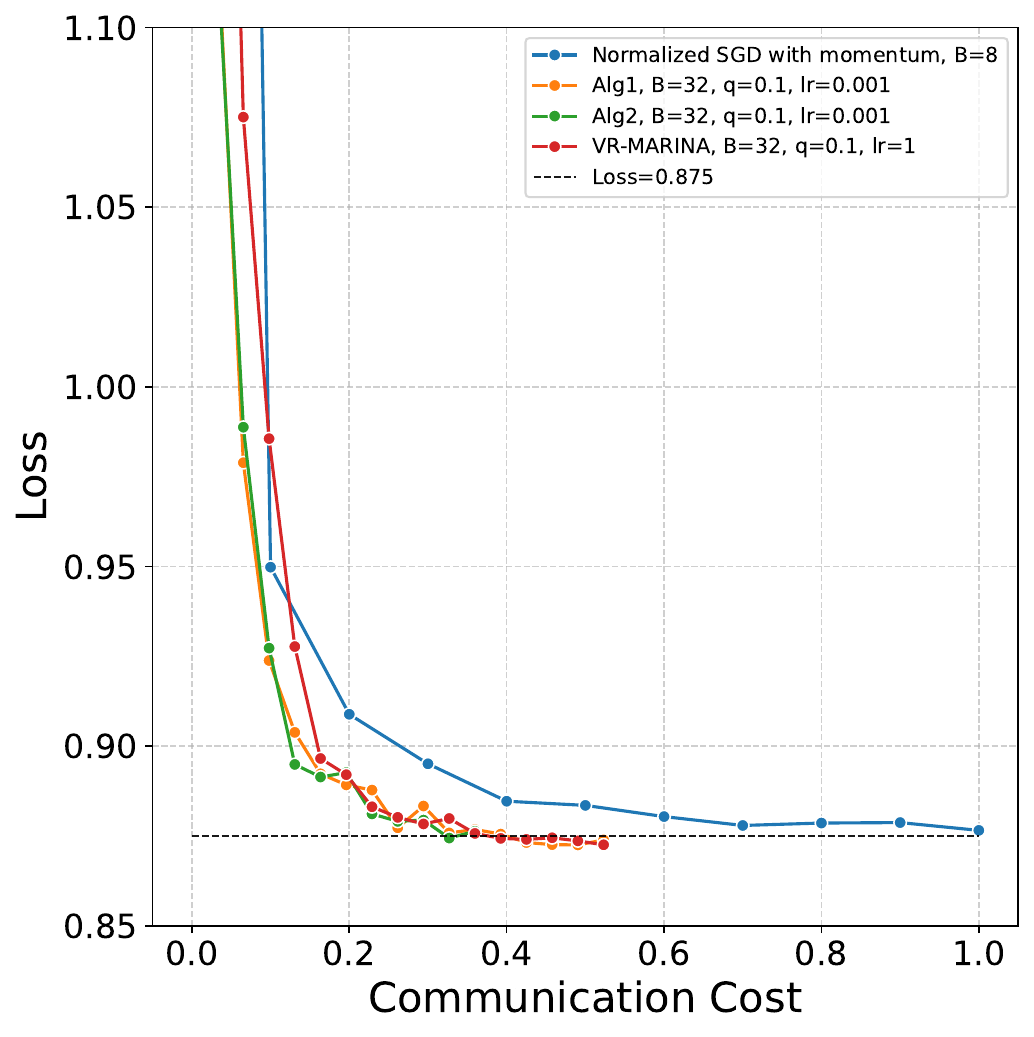}
    \caption{CNN, CIFAR10. Loss vs. normalized communication cost. Algorithms \ref{alg:q-cgluon}, \ref{alg:q-deltagluon} and VR-MARINA utilize compression parameter $K=1\%$}
    \label{fig:main_cnn}
\end{figure}

\subsection{Convolutional Neural Network}

We train a CNN model based on \cite{jordan2024cifar10}, on the CIFAR-10 dataset \cite{cifar10}. We utilize a batch size of $16$ and $\tau=4$. We employ spectral norm LMOs for the hidden layers and the infinity norm for the head layer, following \cite{riabinin2025gluon}. For the spectral norm LMOs, inexact updates are computed using 5 Newton–Schulz iterations, consistent with \cite{jordan6muon}. The baseline ($q=1$), corresponds to non-Euclidean SGD with momentum. Hyperparameter tuning is presented in the Appendix \ref{appendix:cifar10}. Figure \ref{fig:main_cnn} displays tuned trajectories, demonstrating $\approx65\%$ communication cost reduction.

\label{section:main_cnn}

%
%


\section*{Impact Statement}

This paper presents work whose goal is to advance the field of Machine
Learning. There are many potential societal consequences of our work, none
which we feel must be specifically highlighted here.


\bibliography{com_gluon_ref}
\bibliographystyle{icml2026}

\newpage
\appendix
\onecolumn

\tableofcontents

\newpage

\section{Proofs for Compressed Gluon}

\subsection{Two Lemmas}

\begin{lemma}\label{lm:yikbound}
	For the sum of $y_i^{\tau, j}$, we have the following estimation. 
	$$
	\E \left[  \| \frac{1}{n}\sum_{\tau=1}^n y_i^{\tau, j} - (\nabla_i f(X^j) - \nabla_i f(X^{j-1}) )\|_2^2   \right] \leq  \left(  \frac{(\omega+1) \delta_i^2}{n} + \frac{\omega L_i^2}{n}  \right) \E \left[  \|X_i^j - X_i^{j-1}\|_{(i)}^2  \right]. 
	$$
\end{lemma}

\begin{proof}
	From the definition of $y_i^{\tau,j}$ in Algorithm \ref{alg:q-cgluon}, we can get 
	\begin{eqnarray*}
		&& \E \left[  \| \frac{1}{n}\sum_{\tau=1}^n y_i^{\tau, j} - (\nabla_i f(X^j) - \nabla_i f(X^{j-1}) )\|_2^2   \right] \\ 
		&\overset{(\ref{eq:Qi})}{=}& \E \left[  \| \frac{1}{n}\sum_{\tau=1}^n y_i^{\tau, j} - \frac{1}{n} \sum_{\tau=1}^n ( \nabla_i f_{\xi^k_\tau}(X^j) - \nabla_i f_{\xi^k_\tau}(X^{j-1}) ) \|_2^2 \right] \\ 
		&& + \E \left[ \| \frac{1}{n} \sum_{\tau=1}^n ( \nabla_i f_{\xi^k_\tau}(X^j) - \nabla_i f_{\xi^k_\tau}(X^{j-1}) )  - (\nabla_i f(X^j) - \nabla_i f(X^{j-1}) )\|_2^2   \right] \\ 
		&\overset{(a)}{\leq}&  \frac{\omega}{n^2} \sum_{\tau=1}^n \E \left[  \|\nabla_i f_{\xi^k_\tau} (X^j) - \nabla_i f_{\xi^k_\tau} (X^{j-1})\|_2^2  \right]  + \frac{\delta_i^2}{n} \E \left[  \|X_i^j - X_i^{j-1}\|_{(i)}^2  \right] \\ 
		&\overset{(b)}{\leq}&  \frac{\omega}{n^2} \sum_{\tau=1}^n \E \left[  \| \nabla_i f_{\xi^k_\tau} (X^j) - \nabla_i f_{\xi^k_\tau} (X^{j-1}) - ( \nabla_i f^\tau(X^j) - \nabla_i f^\tau(X^{j-1}) ) \|_2^2 \right] \\ 
		&& +  \frac{\omega}{n^2} \sum_{\tau=1}^n \E \left[  \|  \nabla_i f^\tau(X^j) - \nabla_i f^\tau(X^{j-1})  \|_2^2 \right] +  \frac{\delta_i^2}{n} \E \left[  \|X_i^j - X_i^{j-1}\|_{(i)}^2  \right] \\ 
		&\overset{(c)}{\leq}& \left(  \frac{(\omega+1) \delta_i^2}{n} + \frac{\omega L_i^2}{n}  \right) \E \left[  \|X_i^j - X_i^{j-1}\|_{(i)}^2  \right], 
	\end{eqnarray*}
	where (a) uses (\ref{eq:Qi}) and Assumption \ref{as:HV}, (b) uses Assumption \ref{as:boundedvariance} and the fact that samples $\xi^k_\tau \sim {\cal D}_\tau$ are i.i.d, (c) uses Assumptions \ref{as:Lismooth} and \ref{as:HV}.
	
\end{proof}

\begin{lemma}\label{lm:sumgammaij}
	Let $\gamma_i^k \eqdef g_i^k - \nabla_i f(X^k)$ for $k \geq 0$ and $\alpha \eqdef 1-\beta$. Then for $0<q\leq 1$, we have 
	\begin{equation}\label{eq:sumgammaijboundq>0}
		\E \left[  \| \sum_{j=1}^k \alpha \beta^{k-j} \gamma_i^j\|_2^2  \right]  \leq \frac{2\alpha}{(2-\alpha)(\alpha+\beta q)} \left(   \frac{2\sigma^2}{Bn}  + \frac{(1-q)t_i^2 \eta^2}{q} \left(  \frac{(\omega+1) \delta_i^2}{n} + \frac{\omega L_i^2}{n}  \right)  \right); 
	\end{equation}
	for $q=0$, we have 
	\begin{equation}\label{eq:sumgammaijboundq=0}
		\E \left[  \| \sum_{j=1}^k \alpha \beta^{k-j} \gamma_i^j\|_2^2  \right]  \leq \frac{2\sigma^2}{Bn}  + 2k \left(  \frac{(\omega+1) \delta_i^2}{n} + \frac{\omega L_i^2}{n}  \right) t_i^2 \eta^2. 
	\end{equation}
\end{lemma}

\begin{proof}
	
	For $j<s$, we have 
	
	\begin{eqnarray}
		\E \langle \gamma_i^j, \gamma_i^s \rangle &=& \E \langle g_i^j - \nabla_i f(X^j), g_i^s - \nabla_i f(X^s) \rangle \nonumber \\ 
		&=& q \E \left\langle g_i^j - \nabla_i f(X^j), \frac{1}{n} \sum_{\tau=1}^n g_i^{\tau,s} - \nabla_i f(X^s) \right\rangle \nonumber \\ 
		&& + (1-q) \E \left\langle g_i^j - \nabla_i f(X^j), g_i^{s-1} + \frac{1}{n} \sum_{\tau=1}^n y_i^{\tau,s} - \nabla_i f(X^s) \right\rangle \nonumber \\ 
		&=& (1-q) \E \langle g_i^j - \nabla_i f(X^j), g_i^{s-1} - \nabla_i f(X^{s-1}) \rangle \nonumber \\ 
		&=& (1-q)^{s-j} \E [ \|g_i^j - \nabla_i f(X^j) \|_2^2 ], \label{eq:gammajgammas}
	\end{eqnarray}
	where in the second equality we use the definition of $g_i^s$, in the third equality we use Assumption \ref{as:boundedvariance} and (\ref{eq:Qi}). 
	
	For $\E[\|\gamma_i^j\|_2^2]$, we have the following estimation. 
	\begin{eqnarray*}
		\E [\|\gamma_i^j\|_2^2] &=& \E [ \|g_i^j - \nabla_i f(X^j) \|_2^2 ] \\ 
		&=& q \E \left[  \| \frac{1}{n} \sum_{\tau=1}^n g_i^{\tau, j} - \nabla_i f(X^j) \|_2^2  \right] + (1-q) \E \left[  \| g_i^{j-1} + \frac{1}{n} \sum_{\tau=1}^n y_i^{\tau, j} - \nabla_i f(X^j) \|_2^2  \right] \\ 
		&\leq&  \frac{q \sigma^2}{Bn} + (1-q) \E \left[  \|g_i^{j-1} - \nabla_i f(X^{j-1})\|_2^2  \right] \\ 
		&& + (1-q) \E \left[  \| \frac{1}{n}\sum_{\tau=1}^n y_i^{\tau, j} - (\nabla_i f(X^j) - \nabla_i f(X^{j-1}) )\|_2^2   \right], 
	\end{eqnarray*}
	where we use Assumption \ref{as:boundedvariance}, the fact that samples $\xi^k_\tau \sim {\cal D}_\tau$ are i.i.d, and the unbiased property of compressor $Q_i$.

	Furthermore, from Lemma \ref{lm:yikbound}, we have 
	\begin{eqnarray*}
		\E \left[  \| \frac{1}{n}\sum_{\tau=1}^n y_i^{\tau, j} - (\nabla_i f(X^j) - \nabla_i f(X^{j-1}) )\|_2^2   \right] &\leq& \left(  \frac{(\omega+1) \delta_i^2}{n} + \frac{\omega L_i^2}{n}  \right) \E \left[  \|X_i^j - X_i^{j-1}\|_{(i)}^2  \right] \\ 
		&\leq&  \left(  \frac{(\omega+1) \delta_i^2}{n} + \frac{\omega L_i^2}{n}  \right) t_i^2 \eta^2, 
	\end{eqnarray*}
	where in the last inequality, we use the fact that $\|X_i^{k+1} - X_i^k\|_{(i)} \leq t_i \eta$ from (\ref{eq:updateCon}). Therefore, for $0<q\leq 1$, we arrive at 
	
	\begin{eqnarray}
		\E [\|\gamma_i^j\|_2^2] &\leq& (1-q) \E \left[  \|\gamma_i^{j-1}\|_2^2  \right] + (1-q)  \left(  \frac{(\omega+1) \delta_i^2}{n} + \frac{\omega L_i^2}{n}  \right) t_i^2 \eta^2 + \frac{q\sigma^2}{Bn} \nonumber \\ 
		&\leq& (1-q)^j \E \left[  \|\gamma_i^0\|_2^2  \right]	+ \sum_{s=0}^{j-1} (1-q)^s \left(  (1-q)  \left(  \frac{(\omega+1) \delta_i^2}{n} + \frac{\omega L_i^2}{n}  \right) t_i^2 \eta^2 + \frac{q\sigma^2}{Bn}  \right) \nonumber \\ 
		&\leq& (1-q)^j \frac{\sigma^2}{Bn} + \frac{\sigma^2}{Bn}  + \frac{1-q}{q} \left(  \frac{(\omega+1) \delta_i^2}{n} + \frac{\omega L_i^2}{n}  \right) t_i^2 \eta^2 \nonumber \\ 
		&\leq&  \frac{2\sigma^2}{Bn}  + \frac{(1-q)t_i^2 \eta^2}{q} \left(  \frac{(\omega+1) \delta_i^2}{n} + \frac{\omega L_i^2}{n}  \right), \label{eq:gammaj2}
	\end{eqnarray}
	where in the third inequality we use Assumption \ref{as:boundedvariance} and the fact that samples $\xi^k_\tau \sim {\cal D}_\tau$ are i.i.d. For $q=0$, we have 
	\begin{eqnarray}
		\E [\|\gamma_i^j\|_2^2] &\leq&  \E \left[  \|\gamma_i^{j-1}\|_2^2  \right] +   \left(  \frac{(\omega+1) \delta_i^2}{n} + \frac{\omega L_i^2}{n}  \right) t_i^2 \eta^2  \nonumber \\ 
		&\leq& \E \left[  \|\gamma_i^0\|_2^2  \right]	+ \sum_{s=0}^{j-1}   \left(  \frac{(\omega+1) \delta_i^2}{n} + \frac{\omega L_i^2}{n}  \right) t_i^2 \eta^2  \nonumber \\ 
		&\leq& \frac{\sigma^2}{Bn}   + j \left(  \frac{(\omega+1) \delta_i^2}{n} + \frac{\omega L_i^2}{n}  \right) t_i^2 \eta^2. \label{eq:gammaj2q=0}
	\end{eqnarray}
	
	Firstly, we consider the case where $0<q\leq 1$. Combining (\ref{eq:gammajgammas}) and (\ref{eq:gammaj2}) yields that 
	\begin{eqnarray}
		&& \E \left[  \| \sum_{j=1}^k \alpha \beta^{k-j} \gamma_i^j\|_2^2  \right] \nonumber \\ 
		&=& \E \left[  \sum_{j=1}^k \sum_{s=1}^k \langle \alpha \beta^{k-j} \gamma_i^j, \alpha \beta^{k-s}\gamma_i^s \rangle  \right] \nonumber  \\ 
		&=& \E \left[  \sum_{j=1}^k \alpha^2 \beta^{2k-2j} \|\gamma_i^j\|_2^2  \right] + 2 \E \left[  \sum_{j<s} \langle \alpha \beta^{k-j} \gamma_i^j, \alpha \beta^{k-s}\gamma_i^s \rangle \right] \nonumber  \\ 
		&\overset{(\ref{eq:gammajgammas})}{\leq}&  \E \left[  \sum_{j=1}^k \alpha^2 \beta^{2k-2j} \|\gamma_i^j\|_2^2  \right] + 2 \E \left[  \sum_{j<s} \alpha^2 \beta^{2k-2j} v^{s-j} \| \gamma_i^j\|_2^2  \right] \nonumber  \\ 
		&\overset{(\ref{eq:gammaj2})}{\leq}&  \left(   \frac{2\sigma^2}{Bn}  + \frac{(1-q)t_i^2 \eta^2}{q} \left(  \frac{(\omega+1) \delta_i^2}{n} + \frac{\omega L_i^2}{n}  \right)  \right) \left( \sum_{j=1}^k \alpha^2 \beta^{2k-2j} + 2\sum_{j<s} \alpha^2 \beta^{2k-2j} v^{s-j}  \right), \label{eq:sumgammaij}
	\end{eqnarray}
	where we denote $v\eqdef \frac{1-q}{\beta} = \frac{1-q}{1-\alpha}$. From the Lemma 5 in \cite{qian2025muon}, we have 
	$$
	\sum_{j=1}^k \alpha^2 \beta^{2k-2j} + 2\sum_{j<s} \alpha^2 \beta^{2k-2j} v^{s-j}  \leq \frac{2\alpha}{(2-\alpha)(\alpha+\beta q)}. 
	$$
	
	Combining tha above inequality with (\ref{eq:sumgammaij}), we can obtain 
	\begin{equation}\label{eq:sumgammaijbound}
		\E \left[  \| \sum_{j=1}^k \alpha \beta^{k-j} \gamma_i^j\|_2^2  \right]  \leq \frac{2\alpha}{(2-\alpha)(\alpha+\beta q)} \left(   \frac{2\sigma^2}{Bn}  + \frac{(1-q)t_i^2 \eta^2}{q} \left(  \frac{(\omega+1) \delta_i^2}{n} + \frac{\omega L_i^2}{n}  \right)  \right). 
	\end{equation}
	
	Now we consider the case where $q=0$. Combining (\ref{eq:gammajgammas}) and (\ref{eq:gammaj2q=0}), we can get 
	\begin{eqnarray*}
		&& \E \left[  \| \sum_{j=1}^k \alpha \beta^{k-j} \gamma_i^j\|_2^2  \right] \nonumber \\ 
		&=& \E \left[  \sum_{j=1}^k \sum_{s=1}^k \langle \alpha \beta^{k-j} \gamma_i^j, \alpha \beta^{k-s}\gamma_i^s \rangle  \right] \nonumber  \\ 
		&=& \E \left[  \sum_{j=1}^k \alpha^2 \beta^{2k-2j} \|\gamma_i^j\|_2^2  \right] + 2 \E \left[  \sum_{j<s} \langle \alpha \beta^{k-j} \gamma_i^j, \alpha \beta^{k-s}\gamma_i^s \rangle \right] \nonumber  \\ 
		&\overset{(\ref{eq:gammajgammas})}{=}& \E \left[  \sum_{j=1}^k \alpha^2 \beta^{2k-2j} \|\gamma_i^j\|_2^2  \right] + 2\sum_{j=1}^k \sum_{s=j+1}^k \alpha^2 \beta^{2k-j-s} \E \left[  \|\gamma_i^j\|_2^2  \right] \nonumber  \\ 
		&\leq&  2\sum_{j=1}^k \sum_{s=j}^k \alpha^2 \beta^{2k-j-s} \E \left[  \|\gamma_i^j\|_2^2  \right] \nonumber  \\ 
		&\leq& 2\sum_{j=1}^k \alpha \beta^{k-j} \E \left[  \|\gamma_i^j\|_2^2  \right] \nonumber  \\ 
		&\overset{(\ref{eq:gammaj2q=0})}{\leq}& 2\alpha \sum_{j=1}^k \beta^{k-j} \frac{\sigma^2}{Bn} + 2\alpha \sum_{j=1}^k \beta^{k-j} j \left(  \frac{(\omega+1) \delta_i^2}{n} + \frac{\omega L_i^2}{n}  \right) t_i^2 \eta^2 \nonumber \\ 
		&\leq& \frac{2\sigma^2}{Bn} + 2\alpha \left(  \frac{(\omega+1) \delta_i^2}{n} + \frac{\omega L_i^2}{n}  \right) t_i^2 \eta^2 \frac{k-(k+1)\beta + \beta^{k+1}}{(1-\beta)^2} \nonumber \\ 
		&\leq&  \frac{2\sigma^2}{Bn}  + 2k \left(  \frac{(\omega+1) \delta_i^2}{n} + \frac{\omega L_i^2}{n}  \right) t_i^2 \eta^2. 
	\end{eqnarray*}
	
\end{proof}

\subsection{Proof of Theorem \ref{th:cs-Gluon}}

Firstly, similar to the proof of (28) in \cite{riabinin2025gluon}, we can obtain 
\begin{align}
	\sum_{i=1}^p \sum_{k=0}^{K-1} t_i \eta \E \left[  \| \nabla_i f(X^k)\|_{(i)\star}  \right] \leq \Delta^0 + &\sum_{i=1}^p \left[  2 \sum_{k=0}^{K-1} t_i \eta \E \left[  \|M_i^k - \nabla_i f(X^k)\|_{(i)\star}   \right] \right. \label{eq:sumfgrad-cs} \\
	& \quad \quad \left.  +  \sum_{k=0}^{K-1} \frac{L_i^0}{2} t_i^2 \eta^2 + \sum_{k=0}^{K-1} \frac{L_i^1 t_i^2 \eta^2}{2} \E \left[  \| \nabla_i f(X^k)\|_{(i)\star}  \right] \right]. \nonumber 
\end{align}

We introduce the following notation: $\mu_i^k \eqdef M_i^k - \nabla_i f(X^k)$, $\gamma_i^k \eqdef g_i^k - \nabla_i f(X^k)$, $\alpha = 1- \beta$, and $S_i^k \eqdef \nabla_i f(X^{k-1}) - \nabla_i f(X^k)$. Then we have 
\begin{eqnarray*}
	\mu_i^k &=& M_i^k - \nabla_i f(X^k) \\ 
	&=& \beta M_i^{k-1} + \alpha g_i^k - \nabla_i f(X^k) \\
	&=& \beta \mu_i^{k-1} + \alpha \gamma_i^k + \beta S_i^k  \\ 
	&=& \beta^k \mu_i^0 + \sum_{j =1}^k \beta^{k-j} \alpha \gamma_i^j + \sum_{j=1}^k \beta^{k+1-j} S_i^j. 
\end{eqnarray*}

Hence we can obtain 
\begin{eqnarray}
	&& \E \left[  \|M_i^k - \nabla_i f(X^k)\|_{(i)\star}  \right] \nonumber \\ 
	&=& \E \left[  \|\mu_i^k\|_{(i)\star}  \right] \nonumber \\ 
	&\overset{(a)}{\leq}& \beta^k \E \left[  \|\mu_i^0\|_{(i)\star}  \right] + \E \left[  \| \sum_{j =1}^k \beta^{k-j} \alpha \gamma_i^j\|_{(i)\star} \right] + \sum_{j=1}^k \beta^{k+1-j} \E \left[  \|S_i^j \|_{(i)\star} \right] \nonumber \\ 
	&\overset{(b)}{\leq}& \beta^k \rho \E \left[  \|\mu_i^0\|_2  \right] + \rho \E \left[  \|\sum_{j=1}^k \beta^{k-j} \alpha \gamma_i^j \|_2 \right] + \sum_{j=1}^k \beta^{k+1-j} \left(  L_i^0 + L_i^1 \E \left[  \nabla_i f(X^j)\|_{(i)\star}  \right]  \right) t_i \eta \nonumber \\ 
	&\overset{(c)}{\leq}& \beta^k \rho \sqrt{\E \left[  \|\mu_i^0\|_2^2  \right]} + \rho \sqrt{\E \left[  \| \sum_{j=1}^k \beta^{k-j} \alpha \gamma_i^j \|_2^2 \right]} + \frac{L_i^0t_i\eta}{\alpha} + L_i^1 t_i \eta \sum_{j=1}^k \beta^{k+1-j} \E \left[  \|\nabla_i f(X^j)\|_{(i)\star}  \right] \nonumber \\ 
	&\overset{(d)}{\leq}& \frac{(1-\alpha)^k \rho \sigma}{\sqrt{Bn}} + \rho \sqrt{\E \left[  \| \sum_{j=1}^k \beta^{k-j} \alpha \gamma_i^j \|_2^2 \right]}  +  \frac{L_i^0t_i\eta}{\alpha} + L_i^1 t_i \eta \sum_{j=1}^k \beta^{k+1-j} \E \left[  \|\nabla_i f(X^j)\|_{(i)\star}  \right], \label{eq:Mik-}
\end{eqnarray}
where (a) uses the triangle inequality, (b) uses Assumptions \ref{as:L0L1smooth} and \ref{as:rho}, (c) uses Jensen’s inequality, (d) uses Assumption \ref{as:boundedvariance} and the fact that samples $\xi^k_\tau \sim {\cal D}_\tau$ are i.i.d.

Next we consider the case where $0<q\leq 1$. (\ref{eq:sumgammaijboundq>0}) in Lemma \ref{lm:sumgammaij} along with (\ref{eq:Mik-}) indicates that 
\begin{eqnarray*}
	&& \E \left[  \|M_i^k - \nabla_i f(X^k)\|_{(i)\star}  \right]  \\ 
	&\leq& \frac{(1-\alpha)^k \rho \sigma}{\sqrt{Bn}} + \frac{2\sqrt{\alpha}\rho \sigma}{\sqrt{(2-\alpha)(\alpha+\beta q)} \sqrt{Bn}}  + \frac{\sqrt{2(1-q)\alpha} \sqrt{(\omega+1)\delta_i^2 + \omega L_i^2}\rho t_i \eta}{\sqrt{(2-\alpha)(\alpha+\beta q)} \sqrt{qn}}  \\ 
	&& +  \frac{L_i^0t_i\eta}{\alpha} + L_i^1 t_i \eta \sum_{j=1}^k \beta^{k+1-j} \E \left[  \|\nabla_i f(X^j)\|_{(i)\star}  \right]. 
\end{eqnarray*}
Then from (\ref{eq:sumfgrad-cs}), we can get 
\begin{align*}
	\sum_{i=1}^p \sum_{k=0}^{K-1} t_i \eta \E \left[  \| \nabla_i f(X^k)\|_{(i)\star}  \right] \leq \Delta^0 + \sum_{i=1}^p  & \left[   \sum_{k=0}^{K-1} \frac{2(1-\alpha)^kt_i\eta\rho \sigma}{\sqrt{Bn}} + \sum_{k=0}^{K-1} \frac{4\sqrt{\alpha}t_i \eta \rho \sigma}{\sqrt{(2-\alpha)(\alpha+\beta q)} \sqrt{Bn}} \right. \\ 
	& \quad \left. + \sum_{k=0}^{K-1} \frac{2\sqrt{2(1-q)\alpha} \sqrt{(\omega+1)\delta_i^2 + \omega L_i^2}\rho t_i^2 \eta^2}{\sqrt{(2-\alpha)(\alpha+\beta q)} \sqrt{qn}}   + \sum_{k=0}^{K-1}  \frac{2L_i^0 t_i^2 \eta^2}{\alpha}   \right. \\ 
	& \quad \left. + \sum_{k=0}^{K-1} 2L_i^1 t_i^2 \eta^2 \sum_{j=1}^k \beta^{k+1-j} \E \left[  \|\nabla_i f(X^j) \|_{(i)\star}  \right] \right.  \\ 
	& \quad  \left. + \sum_{k=0}^{K-1} \frac{L_i^0 t_i^2 \eta^2}{2} + \sum_{k=0}^{K-1}  \frac{L_i^1 t_i^2 \eta^2}{2} \E \left[  \nabla_i f(X^k)\|_{(i)\star}  \right] \right]. 
\end{align*}
Since 
\begin{eqnarray}
	\sum_{k=0}^{K-1} \sum_{j=1}^k \beta^{k+1-j} \E \left[  \|\nabla_i f(X^j) \|_{(i)\star}  \right]  &=& \sum_{j=1}^{K-1} \sum_{k=j}^{K-1} \beta^{k+1-j}  \E \left[  \|\nabla_i f(X^j) \|_{(i)\star}  \right]  \nonumber \\ 
	&\leq& \frac{1}{\alpha} \sum_{k=0}^{K-1}  \E \left[  \|\nabla_i f(X^k) \|_{(i)\star}  \right], \label{eq:sumsum-betagrad-cG}
\end{eqnarray}
which implies that 
\begin{align*}
	\sum_{i=1}^p \sum_{k=0}^{K-1} t_i \eta \E \left[ \|  \nabla_i f(X^k)\|_{(i)\star}  \right] \leq \Delta^0 + \sum_{i=1}^p & \left[   \frac{2t_i \eta \rho \sigma}{\alpha\sqrt{Bn}} +  \frac{4K\sqrt{\alpha}t_i \eta \rho \sigma}{\sqrt{(2-\alpha)(\alpha+\beta q)} \sqrt{Bn}}  + \frac{2KL_i^0 t_i^2 \eta^2}{\alpha}  + \frac{K L_i^0 t_i^2 \eta^2}{2}   \right. \\ 
	& \quad  \left. +   \frac{2K\sqrt{2(1-q)\alpha} \sqrt{(\omega+1)\delta_i^2 + \omega L_i^2}\rho t_i^2 \eta^2}{\sqrt{(2-\alpha)(\alpha+\beta q)} \sqrt{qn}}     \right. \\ 
	& \quad \left.  +   \sum_{k=0}^{K-1} \left(  \frac{2}{\alpha} + \frac{1}{2}  \right) L_i^1 t_i^2 \eta^2 \E \left[  \|\nabla_i f(X^\tau) \|_{(i)\star} \right]   \right]. 
\end{align*}

Now we consider two options: (1) $L_i^1 = 0$ for all $i \in \{  1, ..., p  \}$ and (2) $L_i^1 \neq 0$, for all $i \in \{  1, ..., p  \}$.  

{\bf Case 1:} $L_i^1 = 0$ for all $i \in \{  1, ..., p  \}$. In this case, 
\begin{eqnarray*}
	&& \min_{k=0, ..., K-1} \sum_{i=1}^p t_i \E \left[  \|\nabla_i f(X^k) \|_{(i)\star} \right] \\ 
	&\leq&  \frac{1}{K} \sum_{k=0}^{K-1} \sum_{i=1}^p t_i \E \left[  \|\nabla_i f(X^k) \|_{(i)\star} \right] \\ 
	&\leq&  \frac{\Delta^0}{\eta K} + \frac{2\sum_{i=1}^p t_i \rho \sigma}{\alpha K \sqrt{Bn}} +   \frac{4\sqrt{\alpha}\sum_{i=1}^pt_i \rho \sigma}{\sqrt{(2-\alpha)(\alpha+\beta q)} \sqrt{Bn}} +    \frac{2\sqrt{2(1-q)\alpha} \sum_{i=1}^p\sqrt{(\omega+1)\delta_i^2 + \omega L_i^2} t_i^2 \rho \eta}{\sqrt{(2-\alpha)(\alpha+\beta q)} \sqrt{qn}}  \\ 
	&&  + \frac{2\sum_{i=1}^p L_i^0t_i^2 \eta}{\alpha} + \frac{\sum_{i=1}^p L_i^0 t_i^2 \eta}{2}. 
\end{eqnarray*}

{\bf Case 2:} $L_i^1 \neq 0$, for all $i \in \{  1, ..., p  \}$. First we let $\frac{\eta}{\alpha} \leq \min_{i}  \frac{1}{5L_i^1 t_i}$. Then $\left(  \frac{2}{\alpha} + \frac{1}{2}  \right) L_i^1 t_i \eta \leq \frac{1}{2}$ for all $i$, and 
\begin{eqnarray*}
	&& \min_{k=0, ..., K-1} \sum_{i=1}^p t_i \E \left[  \|\nabla_i f(X^k) \|_{(i)\star} \right] \\ 
	&\leq&  \frac{1}{K} \sum_{k=0}^{K-1} \sum_{i=1}^p t_i \E \left[  \|\nabla_i f(X^k) \|_{(i)\star} \right] \\ 
	&\leq&  \frac{2 \Delta^0}{\eta K} + \frac{4\sum_{i=1}^p t_i \rho \sigma}{\alpha K \sqrt{Bn}} +   \frac{8\sqrt{\alpha}\sum_{i=1}^pt_i \rho \sigma}{\sqrt{(2-\alpha)(\alpha+\beta q)} \sqrt{Bn}} +    \frac{4\sqrt{2(1-q)\alpha} \sum_{i=1}^p\sqrt{(\omega+1)\delta_i^2 + \omega L_i^2} t_i^2 \rho \eta}{\sqrt{(2-\alpha)(\alpha+\beta q)} \sqrt{qn}}  \\ 
	&&  + \frac{4\sum_{i=1}^p L_i^0t_i^2 \eta}{\alpha} + \sum_{i=1}^p L_i^0 t_i^2 \eta. 
\end{eqnarray*}

Now we consider the case where $q=0$. (\ref{eq:sumgammaijboundq=0}) in Lemma \ref{lm:sumgammaij} along with (\ref{eq:Mik-}) implies that 
\begin{eqnarray*}
	&& \E \left[  \|M_i^k - \nabla_i f(X^k)\|_{(i)\star}  \right]  \\ 
	&\leq& \frac{(1-\alpha)^k \rho \sigma}{\sqrt{Bn}} + \frac{\sqrt{2}\rho \sigma}{ \sqrt{Bn}}  + \frac{\sqrt{2k} \sqrt{(\omega+1)\delta_i^2 + \omega L_i^2}\rho t_i \eta}{\sqrt{n}}  \\ 
	&& +  \frac{L_i^0t_i\eta}{\alpha} + L_i^1 t_i \eta \sum_{j=1}^k \beta^{k+1-j} \E \left[  \|\nabla_i f(X^j)\|_{(i)\star}  \right]. 
\end{eqnarray*}
Then similar to the $0<q\leq 1$ case, we can obtain 
\begin{align*}
	\sum_{i=1}^p \sum_{k=0}^{K-1} t_i \eta \E \left[ \|  \nabla_i f(X^k)\|_{(i)\star}  \right] \leq \Delta^0 + \sum_{i=1}^p & \left[   \frac{2t_i \eta \rho \sigma}{\alpha\sqrt{Bn}} +  \frac{2\sqrt{2}Kt_i \eta \rho \sigma}{ \sqrt{Bn}}  + \frac{2KL_i^0 t_i^2 \eta^2}{\alpha}  + \frac{K L_i^0 t_i^2 \eta^2}{2}   \right. \\ 
	& \quad  \left. +   \frac{2\sqrt{2} K^{\frac{3}{2}} \sqrt{(\omega+1)\delta_i^2 + \omega L_i^2}\rho t_i^2 \eta^2}{ \sqrt{n}}     \right. \\ 
	& \quad \left.  +   \sum_{k=0}^{K-1} \left(  \frac{2}{\alpha} + \frac{1}{2}  \right) L_i^1 t_i^2 \eta^2 \E \left[  \|\nabla_i f(X^\tau) \|_{(i)\star} \right]   \right]. 
\end{align*}

Now we consider two options: (1) $L_i^1 = 0$ for all $i \in \{  1, ..., p  \}$ and (2) $L_i^1 \neq 0$, for all $i \in \{  1, ..., p  \}$.  

{\bf Case 1:} $L_i^1 = 0$ for all $i \in \{  1, ..., p  \}$. In this case, 
\begin{eqnarray*}
	&& \min_{k=0, ..., K-1} \sum_{i=1}^p t_i \E \left[  \|\nabla_i f(X^k) \|_{(i)\star} \right] \\ 
	&\leq&  \frac{1}{K} \sum_{k=0}^{K-1} \sum_{i=1}^p t_i \E \left[  \|\nabla_i f(X^k) \|_{(i)\star} \right] \\ 
	&\leq&  \frac{\Delta^0}{\eta K} + \frac{2\sum_{i=1}^p t_i \rho \sigma}{\alpha K \sqrt{Bn}} +   \frac{2\sqrt{2}\sum_{i=1}^pt_i \rho \sigma}{ \sqrt{Bn}} +   \frac{2\sqrt{2} \sqrt{K}\sum_{i=1}^p\sqrt{(\omega+1)\delta_i^2 + \omega L_i^2} t_i^2 \rho \eta}{ \sqrt{n}}  \\ 
	&&  + \frac{2\sum_{i=1}^p L_i^0t_i^2 \eta}{\alpha} + \frac{\sum_{i=1}^p L_i^0 t_i^2 \eta}{2}. 
\end{eqnarray*}

{\bf Case 2:} $L_i^1 \neq 0$, for all $i \in \{  1, ..., p  \}$. First we let $\frac{\eta}{\alpha} \leq \min_{i}  \frac{1}{5L_i^1 t_i}$. Then $\left(  \frac{2}{\alpha} + \frac{1}{2}  \right) L_i^1 t_i \eta \leq \frac{1}{2}$ for all $i$, and 
\begin{eqnarray*}
	&& \min_{k=0, ..., K-1} \sum_{i=1}^p t_i \E \left[  \|\nabla_i f(X^k) \|_{(i)\star} \right] \\ 
	&\leq&  \frac{1}{K} \sum_{k=0}^{K-1} \sum_{i=1}^p t_i \E \left[  \|\nabla_i f(X^k) \|_{(i)\star} \right] \\ 
	&\leq&  \frac{2 \Delta^0}{\eta K} + \frac{4\sum_{i=1}^p t_i \rho \sigma}{\alpha K \sqrt{Bn}} +   \frac{4\sqrt{2}\sum_{i=1}^pt_i \rho \sigma}{ \sqrt{Bn}} +   \frac{4\sqrt{2} \sqrt{K}\sum_{i=1}^p\sqrt{(\omega+1)\delta_i^2 + \omega L_i^2} t_i^2 \rho \eta}{ \sqrt{n}} \\ 
	&&  + \frac{4\sum_{i=1}^p L_i^0t_i^2 \eta}{\alpha} + \sum_{i=1}^p L_i^0 t_i^2 \eta. 
\end{eqnarray*}

\newpage

\section{Communication Cost Analysis for Compressed Gluon}

As explained in Subsection \ref{subsec:comc-cG}, we assume $(1-q) = \Theta(1)$ in the compressed case. 

When $q\geq \alpha$, from Theorem \ref{th:cs-Gluon}, we have 
$$
\min_{k=0, ..., K-1} \sum_{i=1}^p t_i \E \left[  \|\nabla_i f(X^k) \|_{(i)\star} \right]  \leq {\cal O} \left(  \frac{1}{\eta K} + \frac{1}{\alpha K\sqrt{Bn}}  +  \frac{\sqrt{\alpha}}{\sqrt{qBn}}   +  \frac{\sqrt{\alpha(\omega+1)} \eta}{q\sqrt{n}}  +  \frac{\eta}{\alpha}   \right), 
$$
where 
\begin{eqnarray*}
	\frac{1}{\eta K}  +   \frac{\sqrt{\alpha(\omega+1)} \eta}{q\sqrt{n}}  +  \frac{\eta}{\alpha} &\geq& 	\frac{1}{\eta K}  +   \frac{\sqrt{\alpha(\omega+1)} \eta}{q\sqrt{n}}  +  \frac{\eta}{\sqrt{\alpha} \sqrt{q}} \geq \frac{1}{\eta K}  +  \frac{2(\omega+1)^{1/4} \eta}{q^{3/4}n^{1/4}}  \\ 
	&\geq& \frac{2\sqrt{2} (\omega+1)^{1/8}}{q^{3/8} n^{1/8} K^{1/2}}. 
\end{eqnarray*}
Hence, to achieve the precision $\min_{k=0, ..., K-1} \sum_{i=1}^p t_i \E \left[  \|\nabla_i f(X^k) \|_{(i)\star} \right]  \leq \epsilon$ by Algorithm \ref{alg:q-cgluon}, $K$ is at least $\frac{(\omega+1)^{1/4}}{\epsilon^2 q^{3/4} n^{1/4}}$, and the expected totally communication cost is at least 
$$
\Theta\left( nd + Kn \left(1 + qd +  \frac{(1-q)d}{\omega+1} \right)   \right)  \geq  \Theta\left( nd +  \frac{(\omega+1)^{1/4} n^{3/4}}{\epsilon^2 q^{3/4}} + \frac{n^{3/4}d}{\epsilon^2}  \right). 
$$
When $q\leq \alpha$, from Theorem \ref{th:cs-Gluon}, we have 
$$
\min_{k=0, ..., K-1} \sum_{i=1}^p t_i \E \left[  \|\nabla_i f(X^k) \|_{(i)\star} \right]  \leq {\cal O} \left(  \frac{1}{\eta K} + \frac{1}{\alpha K\sqrt{Bn}}  +  \frac{1}{\sqrt{Bn}}   +  \frac{\sqrt{\omega+1} \eta}{\sqrt{qn}}  +  \frac{\eta}{\alpha}   \right), 
$$
where 
$$
\frac{1}{\eta K} +  \frac{\sqrt{\omega+1} \eta}{\sqrt{qn}}  \geq \frac{2(\omega+1)^{1/4}}{q^{1/4}n^{1/4}K^{1/2}}. 
$$
Then, to achieve $\epsilon$ precision, $K$ is at least $\frac{(\omega+1)^{1/2}}{\epsilon^2q^{1/2}n^{1/2}}$, and the expected totally communication cost is at least 
$$
\Theta\left( nd + Kn \left(1 + qd +  \frac{(1-q)d}{\omega+1} \right)   \right)  \geq  \Theta\left( nd +  \frac{(\omega+1)^{1/2} n^{1/2}}{\epsilon^2 q^{1/2}} + \frac{\sqrt{n}d}{\epsilon^2}  \right). 
$$
When $q=0$, from Theorem \ref{th:cs-Gluon}, we have 
\begin{eqnarray*}
	&& \min_{k=0, ..., K-1} \sum_{i=1}^p t_i \E \left[  \|\nabla_i f(X^k) \|_{(i)\star} \right] \\ 
	&\leq& {\cal O} \left(  \frac{1}{\eta K} + \frac{1}{\alpha K \sqrt{Bn}}  +   \frac{1}{\sqrt{Bn}}  +  \frac{\sqrt{K(\omega+1)}\eta}{\sqrt{n}} + \frac{\eta}{\alpha} + \frac{\eta}{2}    \right), 
\end{eqnarray*}
where 
$$
\frac{1}{\eta K} +   \frac{\sqrt{K(\omega+1)}\eta}{\sqrt{n}} \geq 2 \left(  \frac{\omega+1}{nK}  \right)^{1/4}. 
$$
Thus, to achieve the $\epsilon$ precision by Algorithm \ref{alg:q-cgluon}, $K$ is at least $\frac{\omega+1}{\epsilon^4 n}$, and the expected totally communication cost is at least 
$$
\Theta\left( nd + Kn \left(1 + \frac{d}{\omega+1} \right)   \right)  = \Theta\left( nd +  \frac{\omega+1}{\epsilon^4} + \frac{d}{\epsilon^4}  \right). 
$$
To summarize, since $n\leq \frac{1}{\epsilon^4}$ generally in practice, the expected totally communication cost is at least ${\cal O}\left(  \frac{\sqrt{n}d}{\epsilon^2}  \right)$ from the upper-bounds in Theorem \ref{th:cs-Gluon}. \\

Next we consider the detailed parameter settings. To balance the last two terms in (\ref{eq:comcost-k-cg}), we assume $q(\omega+1) = \Theta(1)$. We first consider the case where $L_i^1 = 0$. We assume $qB = \Theta(1)$ such that the expected stochastic gradient oracle at each step on each working node is $qB + 2(1-q) = \Theta(1)$. When $q\geq \alpha$, from Theorem \ref{th:cs-Gluon}, we have 
\begin{eqnarray*}
	&& \min_{k=0, ..., K-1} \sum_{i=1}^p t_i \E \left[  \|\nabla_i f(X^k) \|_{(i)\star} \right] \\ 
	&\leq& {\cal O} \left(  \frac{1}{\eta K} + \frac{\sqrt{q}}{\alpha K \sqrt{n}}  +   \frac{\sqrt{\alpha}}{\sqrt{n}}  +  \frac{\sqrt{\alpha }\eta}{q\sqrt{qn}} + \frac{\eta}{\alpha} + \frac{\eta}{2}   \right). 
\end{eqnarray*}
Noticed that 
$$
\frac{1}{\eta K}  + \frac{\sqrt{\alpha}}{\sqrt{n}}  + \frac{\eta}{\alpha} \geq \frac{2}{\sqrt{\alpha K}} + \frac{\sqrt{\alpha}}{\sqrt{n}}   \geq \frac{2\sqrt{2}}{n^{1/4}K^{1/4}}, 
$$
and the minimum is obtained by choosing $\alpha = \frac{n^{1/2}}{K^{1/2}}$, $\eta = \frac{n^{1/4}}{K^{3/4}}$. Since $q\geq \alpha$, we have $\frac{\sqrt{\alpha }\eta}{q\sqrt{qn}} \leq \frac{1}{n^{3/4}K^{1/4}}$, which implies that 
$$
\min_{k=0, ..., K-1} \sum_{i=1}^p t_i \E \left[  \|\nabla_i f(X^k) \|_{(i)\star} \right]  \leq {\cal O} \left(  \frac{1}{n^{1/4}K^{1/4}}  +  \frac{n^{1/4}}{K^{3/4}}  \right). 
$$
Assume $n\leq K$, which generally holds in practice. Then the above upper-bound becomes $ {\cal O} \left(  \frac{1}{n^{1/4}K^{1/4}}  \right)$. Since the communication cost in (\ref{eq:comcost-k-cg}) is proportional to $q$, we choose $q=\alpha$. To achieve the precision $\min_{k=0, ..., K-1} \sum_{i=1}^p t_i \E \left[  \|\nabla_i f(X^k) \|_{(i)\star} \right]  \leq \epsilon$ by Algorithm \ref{alg:q-cgluon}, it is sufficient to choose $K = \frac{1}{\epsilon^4 n}$, and the expected totally communication cost is 
\begin{equation}\label{eq:cmuon-comcost-q>a}
	\Theta\left( nd + Kn \left(1 + qd + \frac{(1-q)d}{\omega+1} \right) \right)  = \Theta\left(  \frac{1}{\epsilon^4} + \frac{nd}{\epsilon^2}  \right). 
\end{equation}
In this case, $B = \Theta(\frac{1}{q}) = \Theta(\frac{1}{\epsilon^2 n})$ and the expected stochastic gradient oracle is $\Theta(Bn + Kn(qB+ (1-q))) = \Theta(\frac{1}{\epsilon^4})$. 
For the above parameter setting, if we increase $\alpha$, then $q$ will increase since $q\geq \alpha$, and meanwhile the upper-bound of $\min_{k=0, ..., K-1} \sum_{i=1}^p t_i \E \left[  \|\nabla_i f(X^k) \|_{(i)\star} \right]$ is at least ${\cal O} \left(  \frac{1}{n^{1/4}K^{1/4}}  \right)$. Hence, the communication cost will increase as well. Next we decrease $\alpha$, i.e., let $\alpha = \frac{cn^{1/2}}{K^{1/2}}$ with $0<c<1$. Then it is easy to verify that the upper-bound is at least ${\cal O} \left(  \frac{1}{\sqrt{c}n^{1/4}K^{1/4}}  \right)$ with the minimum obtained by $\eta = \frac{\sqrt{c}n^{1/4}}{K^{3/4}}$. Similar to the above analysis, to achieve the precision $\min_{k=0, ..., K-1} \sum_{i=1}^p t_i \E \left[  \|\nabla_i f(X^k) \|_{(i)\star} \right]  \leq \epsilon$ by Algorithm \ref{alg:q-cgluon}, by choosing $K = \frac{1}{c^2\epsilon^4 n}$ and $q=\alpha = c^2 n \epsilon^2$, the expected totally communication cost is 
$$
\Theta\left( nd + Kn \left(1 + qd + \frac{(1-q)d}{\omega+1} \right) \right)  = \Theta\left(  \frac{1}{c^2\epsilon^4} + \frac{nd}{\epsilon^2}  \right), 
$$
which is larger than (\ref{eq:cmuon-comcost-q>a}).

When $q\leq \alpha$, from Theorem \ref{th:cs-Gluon}, we have 
\begin{eqnarray*}
	&& \min_{k=0, ..., K-1} \sum_{i=1}^p t_i \E \left[  \|\nabla_i f(X^k) \|_{(i)\star} \right] \\ 
	&\leq& {\cal O} \left(  \frac{1}{\eta K} + \frac{\sqrt{q}}{\alpha K \sqrt{n}}  +   \frac{\sqrt{q}}{\sqrt{n}}  +  \frac{\eta}{q\sqrt{n}} + \frac{\eta}{\alpha} + \frac{\eta}{2}    \right), 
\end{eqnarray*}
where 
$$
\frac{1}{\eta K} +   \frac{\sqrt{q}}{\sqrt{n}}  +  \frac{\eta}{q\sqrt{n}} \geq \frac{2}{n^{1/4} \sqrt{qK}} +  \frac{\sqrt{q}}{\sqrt{n}}  \geq \frac{2\sqrt{2}}{n^{3/8} K^{1/4}}, 
$$
and the minimum is attained by choosing $q = \frac{n^{1/4}}{K^{1/2}}$, $\eta = \frac{n^{3/8}}{K^{3/4}}$. Assume $n\leq K^{2/3}$. Then as long as $\alpha \geq \frac{n^{3/4}}{K^{1/2}}$, we can get $\frac{\eta}{\alpha} \leq \frac{1}{n^{3/8} K^{1/4}}$. Thus we have $\min_{k=0, ..., K-1} \sum_{i=1}^p t_i \E \left[  \|\nabla_i f(X^k) \|_{(i)\star} \right]  \leq {\cal O} \left(  \frac{1}{n^{3/8}K^{1/4}}  \right)$.  To achieve the $\epsilon$ precision by Algorithm \ref{alg:q-cgluon}, it is sufficient to choose $K = \frac{1}{\epsilon^4 n^{3/2}}$, and the expected totally communication cost is 
\begin{equation}\label{eq:cmuon-comcost-q<a}
	\Theta\left( nd + Kn \left(1 + qd + \frac{(1-q)d}{\omega+1} \right) \right)  = \Theta\left( nd +  \frac{1}{\epsilon^4 \sqrt{n}} + \frac{\sqrt{n}d}{\epsilon^2}  \right), 
\end{equation}
which is better than (\ref{eq:cmuon-comcost-q>a}). 
In the above equality, $qd = \frac{n^{1/4} d}{K^{1/2}}$ is generally greater than $1$ for LLMs since $d$ is usually very large. Thus the dominate term is $\frac{\sqrt{n}d}{\epsilon^2}$. 

Next we relax the condition that $qB=\Theta(1)$. By choosing $q = \frac{cn^{1/4}}{K^{1/2}}$, $\eta = \frac{\sqrt{c}n^{3/8}}{K^{3/4}}$, $B = \frac{cK^{1/2}}{n^{1/4}}$, and $\alpha \geq \max\{  \frac{cn^{3/4}}{K^{1/2}},  K^{-1}  \}$ for $0<c \leq  \frac{K^{1/2}}{n^{3/4}}$, we have 
\begin{eqnarray*}
	&& \min_{k=0, ..., K-1} \sum_{i=1}^p t_i \E \left[  \|\nabla_i f(X^k) \|_{(i)\star} \right] \\ 
	&\leq& {\cal O} \left(  \frac{1}{\sqrt{c} n^{3/8}K^{1/4}} + \frac{1}{\sqrt{c} n^{3/8}K^{5/4}} +   \frac{\sqrt{c}n^{3/8}}{K^{3/4}}  \right) \leq  {\cal O} \left(  \frac{1}{\sqrt{c} n^{3/8}K^{1/4}} \right). 
\end{eqnarray*}
To achieve the $\epsilon$ precision by Algorithm \ref{alg:q-cgluon}, it is sufficient to choose $K = \frac{1}{c^2 \epsilon^4 n^{3/2}}$, and the expected totally communication cost is 
\begin{equation}\label{eq:comcost-cG-ap}
	\Theta\left( nd + Kn \left(1 + qd + \frac{(1-q)d}{\omega+1} \right) \right)  = \Theta\left(  nd + \frac{1}{c^2 \epsilon^4 \sqrt{n}} + \frac{\sqrt{n}d}{\epsilon^2}  \right), 
\end{equation}
which is (\ref{eq:cmuon-comcost-q<a}) when $c=1$. In this case, $q=c^2n\epsilon^2$, $B = \frac{1}{\epsilon^2 n}$, $\eta = c^2 n^{\frac{3}{2}} \epsilon^3$, $\alpha \geq \max\{  cn^{\frac{3}{2}}\epsilon^2,  c^2n^{\frac{3}{2}}\epsilon^4 \}$, and the expected stochastic gradient oracle is $\Theta\left(  \frac{1}{\epsilon^2} +  \left( 1 + \frac{1}{c^2}\right) \frac{1}{\epsilon^4 \sqrt{n}} \right)$. 
The dominate term in (\ref{eq:comcost-cG-ap}) is $\frac{\sqrt{n}d}{\epsilon^2}$ for large $d$ in LLMs. For the uncompressed algorithm in subsection \ref{subsec:new-vr-alg}, the totally communication cost is $\Theta\left(  \frac{d}{\epsilon^3}  \right)$, which is larger than that in (\ref{eq:comcost-cG-ap}) when $n \leq \frac{1}{\epsilon^2}$.

Finally, for the case where $L_i^1 \neq 0$, the above results also hold as long as $\frac{\eta}{\alpha} \leq \min_{i}  \frac{1}{5L_i^1 t_i}$, which generally holds for large $K$ in practice.

\newpage 

\section{Communication Cost Analysis for Compressed Gluon in the $L_i^0=0$ Case}

\paragraph{Uncompressed case} First, we consider the uncompressed case, i.e., $\omega=0$. Then the communication cost at each step for each working node is $\Theta(d)$. To reduce the expected totally communication cost, the convergence rate should be as fast as possible. Hence from the upper-bounds in Theorem \ref{th:cs-Gluon}, we should choose $q=\alpha=1$, and have 
$$
\min_{k=0, ..., K-1} \sum_{i=1}^p t_i \E \left[  \|\nabla_i f(X^k) \|_{(i)\star} \right]  \leq \frac{2\Delta^0}{\eta K} + \frac{4\sum_{i=1}^p t_i\rho \sigma}{K\sqrt{Bn}}  +   \frac{8\sum_{i=1}^p t_i\rho \sigma}{\sqrt{Bn}}. 
$$
When $\alpha=1$, the upper-bound of $\sum_{k=0}^{K-1} \sum_{j=1}^k \beta^{k+1-j} \E \left[  \|\nabla_i f(X^j) \|_{(i)\star}  \right]$ in (\ref{eq:sumsum-betagrad-cG}) could be zero. Thus the condition $\left(  \frac{2}{\alpha} + \frac{1}{2}  \right) L_i^1 t_i \eta \leq \frac{1}{2}$ can be replaced by $\frac{1}{2} L_i^1 t_i \eta \leq \frac{1}{2}$, which implies $\eta \leq \min_i \frac{1}{L_i^1t_i}$. To reach $\epsilon$ precision, it is sufficient to choose $\eta = \min_i \frac{1}{L_i^1t_i}$, (assume $K\geq 8$), 
$$
K = \frac{4\Delta^0}{\epsilon \min_i \frac{1}{L_i^1 t_i}}, \quad B = \frac{289(\sum_{i=1}^p t_i) \rho^2 \sigma^2}{\epsilon^2 n}. 
$$
If we set $t_i = \frac{1/L_i^1}{\frac{1}{p} \sum_{j=1}^p 1/L_j^1}$, then $K = \frac{4\Delta^0}{\epsilon (\frac{1}{p} \sum_{j=1}^p 1/L_j^1)}$, which is twice as much as the iteration complexity of deterministic Gluon \cite{riabinin2025gluon} when $L_i^0=0$. The above batch size is large, and we can actually choose small $\alpha$ to reduce it. From Theorem \ref{th:cs-Gluon}, by choosing $\eta = \alpha c_0$, where we denote $c_0 \eqdef \min_i \frac{1}{5L_i^1 t_i}$, we have 
$$
\min_{k=0, ..., K-1} \sum_{i=1}^p t_i \E \left[  \|\nabla_i f(X^k) \|_{(i)\star} \right]  \leq \frac{2\Delta^0}{\alpha K  c_0}  +  \frac{12\sqrt{\alpha} \sum_{i=1}^p t_i \rho \sigma}{\sqrt{Bn}}, 
$$
for $\alpha \geq K^{-\frac{2}{3}}$; and 
$$
\min_{k=0, ..., K-1} \sum_{i=1}^p t_i \E \left[  \|\nabla_i f(X^k) \|_{(i)\star} \right]  \leq \frac{2\Delta^0}{\alpha K  c_0}  +  \frac{12 \sum_{i=1}^p t_i \rho \sigma}{\alpha K\sqrt{Bn}}, 
$$
for $\alpha \leq K^{-\frac{2}{3}}$. Therefore, to reach $\epsilon$ precision, it is sufficient to choose 
$$
K = \frac{4\Delta^0}{\alpha \epsilon c_0}, \quad B = \frac{576\alpha (\sum_{i=1}^p t_i) \rho^2 \sigma^2}{\epsilon^2 n}, 
$$
for $\left(  \frac{\epsilon c_0}{4\Delta^0}  \right)^2 \leq \alpha \leq 1$; and choose 
$$
K = \frac{4\Delta^0}{\alpha \epsilon c_0}, \quad B \geq \frac{36c_0^2 (\sum_{i=1}^p t_i)^2 \rho^2 \sigma^2}{(\Delta^0)^2 n}, 
$$
for $\alpha \leq \left(  \frac{\epsilon c_0}{4\Delta^0}  \right)^2$.

\paragraph{Compressed case}  If $q=\Theta(1)$, then the expected communication cost at each step in (\ref{eq:comcost-k-cg}) is $\Theta(nd)$, which is the same as that of uncompressed algorithms. Hence we assume $(1-q) = \Theta(1)$ in the compressed case. 

When $q\geq \alpha$, from Theorem \ref{th:cs-Gluon}, we can get 
$$
\min_{k=0, ..., K-1} \sum_{i=1}^p t_i \E \left[  \|\nabla_i f(X^k) \|_{(i)\star} \right]  \leq {\cal O} \left(  \frac{1}{\eta K}  +  \frac{1}{\alpha K \sqrt{Bn}}  +  \frac{\sqrt{\alpha}}{\sqrt{qBn}}  +   \frac{\sqrt{\alpha (\omega+1)}\eta}{q\sqrt{n}}   \right), 
$$
where 
\begin{eqnarray*}
	\frac{1}{\eta K}  +    \frac{\sqrt{\alpha (\omega+1)}\eta}{q\sqrt{n}}  &\geq&  \frac{1}{\eta K}  +   \frac{\sqrt{ (\omega+1)}\eta^{\frac{3}{2}}}{q\sqrt{nc_0}} \\ 
	&\geq&  \frac{(\omega+1)^{1/5}}{K^{3/5}(q\sqrt{nc_0})^{2/5}}. 
\end{eqnarray*}
Thus, to reach $\epsilon$ precision, $K$ is at least $\frac{(\omega+1)^{1/3}}{\epsilon^{5/3} q^{2/3}(nc_0)^{1/3}}$, and the expected totally communication cost is at least 
$$
\Theta\left( nd + Kn \left(1 + qd +  \frac{(1-q)d}{\omega+1} \right)   \right)  \geq  \Theta\left( nd +  \frac{(\omega+1)^{1/3} n^{2/3}}{\epsilon^{5/3} q^{2/3} c_0^{1/3}} + \frac{n^{2/3}d}{\epsilon^{5/3} c_0^{1/3}}  \right). 
$$
When $q\leq \alpha$, from Theorem \ref{th:cs-Gluon}, we can obtain 
$$
\min_{k=0, ..., K-1} \sum_{i=1}^p t_i \E \left[  \|\nabla_i f(X^k) \|_{(i)\star} \right]  \leq {\cal O} \left(  \frac{1}{\eta K}  +  \frac{1}{\alpha K \sqrt{Bn}}  +  \frac{1}{\sqrt{Bn}}  +   \frac{\sqrt{\omega+1}\eta}{\sqrt{qn}}   \right). 
$$
Then similar to the analysis in Subsection \ref{subsec:comc-cG}, $K$ is at least $\frac{\sqrt{\omega+1}}{\epsilon^2 \sqrt{qn}}$, and the expected totally communication cost is at least $\Theta\left( nd +  \frac{\sqrt{n} \sqrt{\omega+1}}{\epsilon^2 \sqrt{q}} +  \frac{\sqrt{n}d}{\epsilon^2}  \right)$. 

When $q=0$, similar to the analysis in Subsection \ref{subsec:comc-cG}, $K$ is at least $\frac{\omega+1}{\epsilon^4 n}$, and the expected totally communication cost is at least $\Theta\left( nd +  \frac{\omega+1}{\epsilon^4} + \frac{d}{\epsilon^4}  \right)$. 

If $\frac{\sqrt{n}d}{\epsilon^2}  \leq \frac{n^{2/3}d}{\epsilon^{5/3} c_0^{1/3}}$, then we can use the parameter settings in Subsection \ref{subsec:comc-cG} to achieve $\Theta\left(  \frac{\sqrt{n}d}{\epsilon^2}   \right)$ expected totally communication cost. If $\frac{\sqrt{n}d}{\epsilon^2}  \geq \frac{n^{2/3}d}{\epsilon^{5/3} c_0^{1/3}}$, i.e., $n \leq \frac{c_0^2}{\epsilon^2}$, we will show how to choose the parameters in the following. To balance the last two terms in (\ref{eq:comcost-k-cg}), we also assume $q(\omega+1) = \Theta(1)$. Without loss of generality, we assume $c_0 \leq 1$. By choosing $\eta = c_0^{\frac{1}{3}} n^{\frac{1}{3}} K^{-\frac{2}{3}}$, $\alpha = c_0^{-\frac{2}{3}} n^{\frac{1}{3}} K^{-\frac{2}{3}}$, $q = c_0^{\frac{2}{9}} n^{\frac{2}{9}} K^{-\frac{4}{9}}$, and $B = 1/q$, from Theorem \ref{th:cs-Gluon}, we can obtain 
\begin{eqnarray*}
	\min_{k=0, ..., K-1} \sum_{i=1}^p t_i \E \left[  \|\nabla_i f(X^k) \|_{(i)\star} \right]  &\leq& {\cal O} \left(  \frac{1}{\eta K}  +  \frac{1}{\alpha K \sqrt{B}}  +  \frac{\sqrt{\alpha}}{\sqrt{qBn}}  +  \frac{\sqrt{\alpha} \eta}{q\sqrt{qn}}  \right) \\ 
	&\leq& {\cal O} \left(  \frac{1}{c_0^{1/3} n^{1/3}K^{1/3}}  \right). 
\end{eqnarray*}
Thus, to reach $\epsilon$ precision, it is sufficient to choose $K = \frac{1}{\epsilon^3 c_0 n}$, and the expected totally communication cost is 
$$
\Theta\left( nd + Kn \left(1 + qd +  \frac{(1-q)d}{\omega+1} \right)   \right)  =  \Theta\left( nd + \frac{1}{\epsilon^3 c_0} + \frac{n^{2/3}d}{\epsilon^{5/3} c_0^{1/3}}  \right). 
$$
In this case, $\eta = c_0 n \epsilon^2$, $\alpha = n\epsilon^2$, $q = c_0^{\frac{2}{3}} n^{\frac{2}{3}} \epsilon^{\frac{4}{3}}$, and the expected stochastic gradient oracle is $\Theta \left(  \frac{1}{\epsilon^3 c_0}  \right)$.

\newpage 

\section{Compressed Gluon with MVR}

\begin{algorithm}
	\caption{Compressed Gluon with MVR}
	\label{alg:q-cgluon-mvr}
	\begin{algorithmic}[1]
		\STATE \textbf{Input:} Initial model parameters $X^0 = [X_1^0, \dots, X_p^0] \in \mathcal{S}$, momentum $M^0 = [M_1^0, \dots, M_p^0] \in \mathcal{S}$, momentum decay factors $\beta \in [0, 1)$ for all iterations $k \geq 0$, probability $q \in (0, 1]$, batch size $B$, stepsize parameter $\eta>0$, $u^0=1 \in \R$ 
		\FOR{ $k = 0, 1, 2, ..., K-1$}
		\FOR{ $\tau = 1, ..., n$} 
		\STATE $u^{k+1}_\tau = 0$ for $\tau = 2, ..., n$, \quad 
		$
		u^{k+1}_1 = \left\{ \begin{array}{rl}
			1 & \mbox{ with probability $q$} \\
			0 &\mbox{ with probability $1-q$}
		\end{array} \right.
		$\\
		\IF{$u^k = 1$}
		\STATE Sample $\xi^k_{\tau, j} \sim \mathcal{D_\tau}$ independently for $j = 1, ..., B$
		\STATE $g_i^{\tau, k} = \frac{1}{B} \sum_{j=1}^B \nabla f_{\xi^k_{\tau, j}} (X^k)$ for $i = 1, ..., p$ 
		\STATE Send $g_i^{\tau, k}$ to the other nodes
		\STATE Receive $g_i^{\tau, k}$ from the other nodes
		\ENDIF
		\STATE Sample $\xi^k_\tau \sim \mathcal{D_\tau}$ 
		\STATE  $y_i^{\tau, k} = Q_i\left(  \nabla_i f_{\xi^k_\tau} (X^k) - \nabla_i f_{\xi^k_\tau} (X^{k-1})  \right)$ for $i = 1, ..., p$ 
		\STATE Send $y_i^{\tau, k}$ and $u^{k+1}_\tau$ to the other nodes
		\STATE Receive $y_i^{\tau, k}$ and $u^{k+1}_\tau$ from the other nodes
		
		\ENDFOR
		\STATE $y_i^k = \frac{1}{n}\sum_{\tau=1}^n y_i^{\tau,k}$
		\STATE 
		$
		g_i^k= \left\{ \begin{array}{cl}
			\frac{1}{n} \sum_{\tau=1}^n g_i^{\tau, k} & \mbox{ if $u^k = 1$} \\
			g_i^{k-1} + y_i^k &\mbox{ otherwise}
		\end{array} \right.
		$
		\STATE $u^{k+1} = \sum_{\tau=1}^n u^{k+1}_\tau$
		\STATE Update momentum $M_i^k = \beta M_i^{k-1} + (1-\beta) g_i^k + \beta y_i^k$ for layer $i$ 
		\STATE Choose adaptive stepsize/radius $t_i \eta > 0$ for layer $i$
		\STATE Update parameters for layer $i$ via LMO over $\mathcal{B}_i^k := \{X_i \in \mathcal{S}_i : \|X_i - X_i^k\|_{(i)} \leq t_i \eta\}$:
		\begin{equation}\label{eq:updateCon-mvr}
			X_i^{k+1} = \text{LMO}_{\mathcal{B}_i^k}(M_i^k) := \arg \min_{X_i \in \mathcal{B}_i^k} \langle M_i^k, X_i \rangle_{(i)}
		\end{equation}
		\STATE Update full parameter vector $X^{k+1} = [X_1^{k+1}, \dots, X_p^{k+1}]$
		\ENDFOR
	\end{algorithmic}
\end{algorithm}

\begin{theorem}\label{th:cs-Gluon-mvr}
	Let Assumptions \ref{as:L0L1smooth}, \ref{as:Lismooth}, \ref{as:boundedvariance}, \ref{as:rho}, and Assumption \ref{as:HV} hold. Assume each $Q_i$ in Algorithm \ref{alg:q-cgluon-mvr} is unbiased compressor satisfying (\ref{eq:Qi}).   Let $X^0, ..., X^{K-1}$ be the iterates of Algorithm \ref{alg:q-cgluon-mvr}, and $M_i^0 = \frac{1}{Bn} \sum_{\tau=1}^n \sum_{j=1}^B \nabla_i f_{\xi_{\tau, j}^k} (X^0)$. 
	
	1. If $L_i^1 =0$, then for $0<q\leq 1$, 
	\begin{eqnarray}
		&& \min_{k=0, ..., K-1} \sum_{i=1}^p t_i \E \left[  \|\nabla_i f(X^k) \|_{(i)\star} \right] \nonumber  \\ 
		&\leq& \frac{\Delta^0}{\eta K} + \frac{2\sum_{i=1}^p t_i \rho \sigma}{\alpha K\sqrt{Bn}}  +   \frac{4\sqrt{\alpha}\sum_{i=1}^p t_i  \rho \sigma}{\sqrt{(2-\alpha)(\alpha+\beta q)} \sqrt{Bn}}  +    \frac{2\rho \sum_{i=1}^p t_i^2 \sqrt{(\omega+1)\delta_i^2 + \omega L_i^2}\eta}{\sqrt{\alpha n}}  \nonumber  \\ 
		&&  +    \frac{2\sqrt{2(1-q)\alpha} \sum_{i=1}^p \sqrt{(\omega+1)\delta_i^2 + \omega L_i^2} t_i^2 \rho \eta}{\sqrt{(2-\alpha)(\alpha+\beta q)} \sqrt{qn}}   + \frac{\sum_{i=1}^p L_i^0 t_i^2 \eta}{2};  \label{eq:th-cs-Gluon-mvr-1}
	\end{eqnarray}
	for $q=0$, 
	\begin{eqnarray}
		&& \min_{k=0, ..., K-1} \sum_{i=1}^p t_i \E \left[  \|\nabla_i f(X^k) \|_{(i)\star} \right]  \nonumber  \\ 
		&\leq& \frac{\Delta^0}{\eta K} + \frac{2\sum_{i=1}^p t_i \rho \sigma}{\alpha K\sqrt{Bn}}  +   \frac{2\sqrt{2}\sum_{i=1}^p t_i  \rho \sigma}{\sqrt{Bn}}  +    \frac{2\rho \sum_{i=1}^p t_i^2 \sqrt{(\omega+1)\delta_i^2 + \omega L_i^2}\eta}{\sqrt{\alpha n}}   \nonumber  \\ 
		&&  +    \frac{2\sqrt{2} \sqrt{K} \sum_{i=1}^p \sqrt{(\omega+1)\delta_i^2 + \omega L_i^2} t_i^2 \rho \eta}{\sqrt{n}}   + \frac{\sum_{i=1}^p L_i^0 t_i^2 \eta}{2}. \label{eq:th-cs-Gluon-mvr-2}
	\end{eqnarray}

	2. If $L_i^1 \neq 0$, we let $\eta \leq \min_i \frac{1}{L_i^1 t_i}$. Then the inequalities (\ref{eq:th-cs-Gluon-mvr-1}) and (\ref{eq:th-cs-Gluon-mvr-2}) remain valid when their right-hand sides are multiplied by a factor $2$. 
	
\end{theorem}

From Theorem \ref{th:cs-Gluon-mvr}, we can get the following convergence rate under suitable choices of parameters. 

\begin{theorem}\label{th:cs-Gluon-mvr-rate}
Under the premise of Theorem \ref{th:cs-Gluon-mvr}, let $\alpha = \Theta(1)$, $q>0$, and $\frac{\Delta^0}{\eta^2 K} = {\cal L}_2 \eqdef  \sum_{i=1}^p L_i^0t_i^2 + \frac{\sqrt{1-q} + \sqrt{q}}{\sqrt{qn}} \sum_{i=1}^p \rho \sqrt{(\omega+1)\delta_i^2 + \omega L_i^2}t_i^2$. Then the upper-bound in (\ref{eq:th-cs-Gluon-mvr-1}) becomes 
$$
{\cal O} \left(  \frac{\sqrt{\Delta^0 {\cal L}_2} }{\sqrt{K}}  +  \frac{\sum_it_i \rho \sigma}{\sqrt{Bn}}  \right). 
$$
\end{theorem}

\subsection{Gluon-MVR-1 in the Minibatch Case}

Let $q=1$, $B=1$,  $\omega=0$, and $\xi_{\tau, 1}^k \equiv \xi_\tau^k$ (it is easy to verify the convergence results still hold) in Algorithm \ref{alg:q-cgluon-mvr}. Then we recover Gluon-MVR-1 with batch size $n$. If $L_i^1 = 0$, from Theorem \ref{th:cs-Gluon-mvr}, we have 
\begin{eqnarray*}
	&& \min_{k=0, ..., K-1} \sum_{i=1}^p t_i \E \left[  \|\nabla_i f(X^k) \|_{(i)\star} \right] \\ 
	&\leq& \frac{\Delta^0}{\eta K} + \frac{2\sum_{i=1}^p t_i \rho \sigma}{\alpha K\sqrt{n}}  +   \frac{4\sqrt{\alpha}\sum_{i=1}^p t_i  \rho \sigma}{\sqrt{(2-\alpha)} \sqrt{n}}  +    \frac{2\rho \sum_{i=1}^p t_i^2 \delta_i \eta}{\sqrt{\alpha n}}  + \frac{\sum_{i=1}^p L_i^0 t_i^2 \eta}{2}. 
\end{eqnarray*}
Assume $n \leq \sqrt{K}$. Then by choosing $\eta = \alpha = n^{\frac{1}{3}} K^{-\frac{2}{3}}$, we can get 
\begin{eqnarray*}
	&& \min_{k=0, ..., K-1} \sum_{i=1}^p t_i \E \left[  \|\nabla_i f(X^k) \|_{(i)\star} \right] \\ 
	&\leq& \frac{\Delta^0}{n^{1/3} K^{1/3}} + \frac{2\sum_{i=1}^p t_i \rho \sigma}{n^{5/6} K^{1/3}}  +   \frac{4\sum_{i=1}^p t_i  \rho \sigma}{\sqrt{(2-\alpha)} n^{1/3} K^{1/3}}  +    \frac{2\rho \sum_{i=1}^p t_i^2 \delta_i }{n^{1/3} K^{1/3}}  + \frac{\sum_{i=1}^p L_i^0 t_i^2 n^{1/3}}{2 K^{2/3}} \\ 
	&\leq& {\cal O} \left(  \frac{1}{n^{1/3}K^{1/3}}  \right). 
\end{eqnarray*}

\paragraph{Linear speed up with batch size $n$} From the above inequality, for $n\leq \sqrt{K}$, to reach $\epsilon$ precision, it is sufficient to choose $K = \frac{1}{\epsilon^3 n}$, and $n\leq \sqrt{K}$ is equivalent to $n \leq \frac{1}{\epsilon}$. In this case, $\eta = \alpha = n \epsilon^2$. For $n\geq \sqrt{K}$, from Theorem \ref{th:cs-Gluon-mvr},  by choosing $\eta = \alpha = K^{-\frac{1}{2}}$, we have $\min_{k=0, ..., K-1} \sum_{i=1}^p t_i \E \left[  \|\nabla_i f(X^k) \|_{(i)\star} \right]  \leq  {\cal O} \left(  \frac{1}{ K^{1/2}}  \right)$, and to reach the $\epsilon$ precision, we can choose $K = \frac{1}{\epsilon^2}$. 

For $n \leq \sqrt{K}$, the stochastic gradient oracle is $Kn = \frac{1}{\epsilon^3}$. Assume the data sampled from distribution $\mathcal{D_\tau}$ each time is different. Then the number of data $N$ should be $ \frac{1}{\epsilon^3}$. \\ 

If $L_i^1 \neq 0$, from Theorem \ref{th:cs-Gluon-mvr}, we can get the above convergence results similarly as long as $K \geq \sqrt{n \left( \max_i L_i^1t_i \right)^3}$ for $n\leq \sqrt{K}$ and $K \geq  \left( \max_i L_i^1t_i \right)^2$ for $n\geq \sqrt{K}$.

\subsection{Communication Cost Analysis}

In this subsection, we consider the compressed case.  If $q=\Theta(1)$, then the expected communication cost at each step (\ref{eq:comcost-k-cg}) is $\Theta(nd)$, which is the same as that of uncompressed algorithms. Hence we assume $(1-q) = \Theta(1)$ in the compressed case.  When $q\geq \alpha$, from Theorem \ref{th:cs-Gluon-mvr}, we have 
$$
\min_{k=0, ..., K-1} \sum_{i=1}^p t_i \E \left[  \|\nabla_i f(X^k) \|_{(i)\star} \right]  \leq {\cal O} \left(  \frac{1}{\eta K}  +  \frac{1}{\alpha K \sqrt{Bn}}  +  \frac{\sqrt{\alpha}}{\sqrt{qBn}}   +  \frac{\sqrt{\omega+1} \eta}{\sqrt{\alpha n}}  +  \frac{\sqrt{\alpha (\omega+1)} \eta}{q \sqrt{n}}  +  \eta  \right), 
$$
where 
$$
\frac{1}{\eta K}   +   \frac{\sqrt{\omega+1} \eta}{\sqrt{\alpha n}}  +  \frac{\sqrt{\alpha (\omega+1)} \eta}{q \sqrt{n}}   \geq  \frac{1}{\eta K}   +  \frac{2\sqrt{\omega+1} \eta}{\sqrt{qn}}  \geq  \frac{2\sqrt{2} (\omega+1)^{1/4}}{q^{1/4} n^{1/4} K^{1/2}}. 
$$
Thus, to achieve $\epsilon$ precision, $K$ is at least $\frac{\sqrt{\omega+1}}{\epsilon^2 \sqrt{qn}}$, and the expected totally communication cost is at least 
\begin{equation}\label{eq:cG-lowbcomc}
	\Theta\left( nd + Kn \left(1 + qd +  \frac{(1-q)d}{\omega+1} \right)   \right)  \geq  \Theta\left( nd +  \frac{(\omega+1)^{1/2} n^{1/2}}{\epsilon^2 q^{1/2}} + \frac{\sqrt{n}d}{\epsilon^2}  \right). 
\end{equation}
When $q\leq \alpha$, from Theorem \ref{th:cs-Gluon-mvr}, we have 
$$
\min_{k=0, ..., K-1} \sum_{i=1}^p t_i \E \left[  \|\nabla_i f(X^k) \|_{(i)\star} \right]  \leq {\cal O} \left(  \frac{1}{\eta K}  +  \frac{1}{\alpha K \sqrt{Bn}}  +  \frac{1}{\sqrt{Bn}}   +  \frac{\sqrt{\omega+1} \eta}{\sqrt{\alpha n}}  +  \frac{\sqrt{\omega+1} \eta}{\sqrt{qn}}  +  \eta  \right), 
$$
where 
$$
\frac{1}{\eta K}   + \frac{\sqrt{\omega+1} \eta}{\sqrt{qn}}  \geq  \frac{2(\omega+1)^{1/4}}{q^{1/4} n^{1/4} K^{1/2}}. 
$$
Thus, to achieve $\epsilon$ precision, $K$ is at least $\frac{\sqrt{\omega+1}}{\epsilon^2 \sqrt{qn}}$, and the expected totally communication cost is at least (\ref{eq:cG-lowbcomc}). 

When $q=0$, from Theorem \ref{th:cs-Gluon-mvr}, we have 
$$
\min_{k=0, ..., K-1} \sum_{i=1}^p t_i \E \left[  \|\nabla_i f(X^k) \|_{(i)\star} \right]  \leq {\cal O} \left(  \frac{1}{\eta K}  +  \frac{1}{\alpha K \sqrt{Bn}}  +  \frac{1}{\sqrt{Bn}}   +  \frac{\sqrt{\omega+1} \eta}{\sqrt{\alpha n}}  +  \frac{\sqrt{K} \sqrt{\omega+1} \eta}{\sqrt{n}}  +  \eta  \right), 
$$
where 
$$
\frac{1}{\eta K} +   \frac{\sqrt{K(\omega+1)}\eta}{\sqrt{n}} \geq 2 \left(  \frac{\omega+1}{nK}  \right)^{1/4}. 
$$
Thus, to achieve the $\epsilon$ precision by Algorithm \ref{alg:q-cgluon}, $K$ is at least $\frac{\omega+1}{\epsilon^4 n}$, and the expected totally communication cost is at least 
$$
\Theta\left( nd + Kn \left(1 + \frac{d}{\omega+1} \right)   \right)  = \Theta\left( nd +  \frac{\omega+1}{\epsilon^4} + \frac{d}{\epsilon^4}  \right). 
$$
To summarize, since $n\leq \frac{1}{\epsilon^4}$ generally in practice, the expected totally communication cost is at least ${\cal O}\left(  \frac{\sqrt{n}d}{\epsilon^2}  \right)$ from the upper-bounds in Theorem \ref{th:cs-Gluon-mvr}. \\

Next we consider the detailed parameter settings. First we consider the case where $L_i^1 = 0$. When $q\neq 0$, assume $q(\omega+1) = \Theta(1)$. Then from Theorem \ref{th:cs-Gluon-mvr}, by choosing $q = \frac{cn^{1/4}}{K^{1/2}}$, $\eta = \frac{\sqrt{c}n^{3/8}}{K^{3/4}}$, $B = \frac{cK^{1/2}}{n^{1/4}}$, and $\alpha \geq \max\{ q, K^{-1}\}$ for $0<c \leq \frac{K^{1/2}}{n^{1/4}}$, we have 
\begin{eqnarray*}
	&& \min_{k=0, ..., K-1} \sum_{i=1}^p t_i \E \left[  \|\nabla_i f(X^k) \|_{(i)\star} \right] \\ 
	&\leq& {\cal O} \left(  \frac{1}{\sqrt{c} n^{3/8}K^{1/4}}  +   \frac{\sqrt{c}n^{3/8}}{K^{3/4}}  \right). 
\end{eqnarray*}
Hence for $c\leq \frac{K^{1/2}}{n^{3/4}}$, the above upper-bound becomes $ {\cal O} \left(  \frac{1}{\sqrt{c} n^{3/8}K^{1/4}} \right)$, and to achieve the $\epsilon$ precision by Algorithm \ref{alg:q-cgluon}, it is sufficient to choose $K = \frac{1}{c^2 \epsilon^4 n^{3/2}}$, and the expected totally communication cost is 
\begin{equation}\label{eq:comcost-cG-mvr-ap}
	\Theta\left( nd + Kn \left(1 + qd + \frac{(1-q)d}{\omega+1} \right) \right)  = \Theta\left( nd +  \frac{1}{c^2 \epsilon^4 \sqrt{n}} + \frac{\sqrt{n}d}{\epsilon^2}  \right), 
\end{equation}
which is the same as (\ref{eq:comcost-cG}). In this case, $q=c^2n\epsilon^2$, $B = \frac{1}{\epsilon^2 n}$, $\eta = c^2 n^{\frac{3}{2}} \epsilon^3$, $\alpha \geq \max\{  c^2n \epsilon^2,  c^2n^{\frac{3}{2}}\epsilon^4 \}$, and the expected stochastic gradient oracle is $\Theta\left(  \frac{1}{\epsilon^2}  +  \left( 1 + \frac{1}{c^2}\right) \frac{1}{\epsilon^4 \sqrt{n}} \right)$.

Finally, for the case where $L_i^1 \neq 0$, the above results also hold as long as $\eta \leq \min_{i}  \frac{1}{L_i^1 t_i}$, which generally holds for large $K$ in practice.

\subsection{Proof of Theorem \ref{th:cs-Gluon-mvr}}

For the iterates of Algorithm \ref{alg:q-cgluon-mvr}, we also define $\mu_i^k \eqdef M_i^k - \nabla_i f(X^k)$, $\gamma_i^k \eqdef g_i^k - \nabla_i f(X^k)$, and $\alpha \eqdef 1-\beta$. Moreover, we define $Z_i^k \eqdef y_i^k - (\nabla_i f(X^k) - \nabla_i f(X^{k-1}))$. Then we have 
\begin{eqnarray*}
	\mu_i^k &=& M_i^k - \nabla_i f(X^k) \\ 
	&=& \beta M_i^{k-1} + \alpha g_i^k - \nabla_i f(X^k) \\
	&=& \beta \mu_i^{k-1} + \alpha \gamma_i^k + \beta Z_i^k  \\ 
	&=& \beta^k \mu_i^0 + \sum_{j =1}^k \beta^{k-j} \alpha \gamma_i^j + \sum_{j=1}^k \beta^{k+1-j} Z_i^j. 
\end{eqnarray*}

Next we upper-bound $\E\left[  \|\sum_{j=1}^k \beta^{k+1-j} Z_i^j \|_{(i)\star} \right]$ as follows. 

\begin{eqnarray}
	\E \left[  \|\sum_{j=1}^k \beta^{k+1-j} Z_i^j \|_{(i)\star} \right] &\overset{Assumption \ref{as:rho}}{\leq}& \rho 	\E \left[  \|\sum_{j=1}^k \beta^{k+1-j} Z_i^j \|_2 \right] \nonumber \\ 
	&\leq&  \rho \sqrt{\E \left[  \|\sum_{j=1}^k \beta^{k+1-j} Z_i^j \|_2^2 \right]}  \nonumber  \\ 
	&=& \rho \sqrt{\sum_{j=1}^k \beta^{2k+2-2j}\E \left[  \|Z_i^j\|_2^2  \right] }  \nonumber \\ 
	&\overset{Lemma \ref{lm:yikbound}}{\leq}& \rho \sqrt{\sum_{j=1}^k \beta^{2k+2-2j}  \frac{(\omega+1) \delta_i^2 + \omega L_i^2}{n}  \E \left[  \|X_i^j - X_i^{j-1}\|_{(i)}^2  \right] }  \nonumber  \\ 
	&\overset{(\ref{eq:updateCon-mvr})}{\leq}& \rho \sqrt{\sum_{j=1}^k \beta^{2k+2-2j}  \frac{(\omega+1) \delta_i^2 + \omega L_i^2}{n} t_i^2 \eta^2 }  \nonumber  \\ 
	&\leq& \frac{\rho t_i \eta\sqrt{(\omega+1)\delta_i^2 + \omega L_i^2}}{\sqrt{\alpha n}}, \label{eq:sumZij}
\end{eqnarray}
where in the second inequality we use Jensen's inequality, in the first equality we use (\ref{eq:Qi}), Assumption \ref{as:boundedvariance} and $Q_i$, $\xi_\tau^j$ are both i.i.d. 

Then for $0<q\leq 1$, similar to the proof of Theorem \ref{th:cs-Gluon}, we can get 
\begin{align*}
	\sum_{i=1}^p \sum_{k=0}^{K-1} t_i \eta \E \left[ \|  \nabla_i f(X^k)\|_{(i)\star}  \right] \leq \Delta^0 + \sum_{i=1}^p & \left[   \frac{2t_i \eta \rho \sigma}{\alpha\sqrt{Bn}} +  \frac{4K\sqrt{\alpha}t_i \eta \rho \sigma}{\sqrt{(2-\alpha)(\alpha+\beta q)} \sqrt{Bn}}   + \frac{K L_i^0 t_i^2 \eta^2}{2}   \right. \\ 
	& \quad \left. +  \frac{2K\rho t_i^2 \eta^2\sqrt{(\omega+1)\delta_i^2 + \omega L_i^2}}{\sqrt{\alpha n}}   \right. \\
	& \quad  \left. +   \frac{2K\sqrt{2(1-q)\alpha} \sqrt{(\omega+1)\delta_i^2 + \omega L_i^2}\rho t_i^2 \eta^2}{\sqrt{(2-\alpha)(\alpha+\beta q)} \sqrt{qn}}     \right. \\ 
	& \quad \left.  +   \sum_{k=0}^{K-1} \frac{1}{2} L_i^1 t_i^2 \eta^2 \E \left[  \|\nabla_i f(X^\tau) \|_{(i)\star} \right]   \right]. 
\end{align*}

Now we consider two options: (1) $L_i^1 = 0$ for all $i \in \{  1, ..., p  \}$ and (2) $L_i^1 \neq 0$, for all $i \in \{  1, ..., p  \}$.  

{\bf Case 1:} $L_i^1 = 0$ for all $i \in \{  1, ..., p  \}$. In this case, 
\begin{eqnarray*}
	&& \min_{k=0, ..., K-1} \sum_{i=1}^p t_i \E \left[  \|\nabla_i f(X^k) \|_{(i)\star} \right] \\ 
	&\leq&  \frac{1}{K} \sum_{k=0}^{K-1} \sum_{i=1}^p t_i \E \left[  \|\nabla_i f(X^k) \|_{(i)\star} \right] \\ 
	&\leq&  \frac{\Delta^0}{\eta K} + \frac{2\sum_{i=1}^p t_i \rho \sigma}{\alpha K\sqrt{Bn}}  +   \frac{4\sqrt{\alpha}\sum_{i=1}^p t_i  \rho \sigma}{\sqrt{(2-\alpha)(\alpha+\beta q)} \sqrt{Bn}}  +    \frac{2\rho \sum_{i=1}^p t_i^2 \sqrt{(\omega+1)\delta_i^2 + \omega L_i^2}\eta}{\sqrt{\alpha n}}  \\ 
	&&  +    \frac{2\sqrt{2(1-q)\alpha} \sum_{i=1}^p \sqrt{(\omega+1)\delta_i^2 + \omega L_i^2} t_i^2 \rho \eta}{\sqrt{(2-\alpha)(\alpha+\beta q)} \sqrt{qn}}   + \frac{\sum_{i=1}^p L_i^0 t_i^2 \eta}{2}. 
\end{eqnarray*}

{\bf Case 2:} $L_i^1 \neq 0$, for all $i \in \{  1, ..., p  \}$. First we let $\eta \leq \min_i \frac{1}{L_i^1 t_i}$. Then $\frac{1}{2} L_i^1 t_i \eta \leq \frac{1}{2}$ for all $i$, and 
\begin{eqnarray*}
	&& \min_{k=0, ..., K-1} \sum_{i=1}^p t_i \E \left[  \|\nabla_i f(X^k) \|_{(i)\star} \right] \\ 
	&\leq&  \frac{1}{K} \sum_{k=0}^{K-1} \sum_{i=1}^p t_i \E \left[  \|\nabla_i f(X^k) \|_{(i)\star} \right] \\ 
	&\leq& \frac{2\Delta^0}{\eta K} + \frac{4\sum_{i=1}^p t_i \rho \sigma}{\alpha K\sqrt{Bn}}  +   \frac{8\sqrt{\alpha}\sum_{i=1}^p t_i  \rho \sigma}{\sqrt{(2-\alpha)(\alpha+\beta q)} \sqrt{Bn}}  +    \frac{4\rho \sum_{i=1}^p t_i^2 \sqrt{(\omega+1)\delta_i^2 + \omega L_i^2}\eta}{\sqrt{\alpha n}}  \\ 
	&&  +    \frac{4\sqrt{2(1-q)\alpha} \sum_{i=1}^p \sqrt{(\omega+1)\delta_i^2 + \omega L_i^2} t_i^2 \rho \eta}{\sqrt{(2-\alpha)(\alpha+\beta q)} \sqrt{qn}}   + \sum_{i=1}^p L_i^0 t_i^2 \eta. 
\end{eqnarray*}

For $q=0$, similar to the proof of Theorem \ref{th:cs-Gluon}, we can get 
\begin{align*}
	\sum_{i=1}^p \sum_{k=0}^{K-1} t_i \eta \E \left[ \|  \nabla_i f(X^k)\|_{(i)\star}  \right] \leq \Delta^0 + \sum_{i=1}^p & \left[   \frac{2t_i \eta \rho \sigma}{\alpha\sqrt{Bn}} +  \frac{2\sqrt{2}K t_i \eta \rho \sigma}{ \sqrt{Bn}}   + \frac{K L_i^0 t_i^2 \eta^2}{2}   \right. \\ 
	& \quad \left. +  \frac{2K\rho t_i^2 \eta^2\sqrt{(\omega+1)\delta_i^2 + \omega L_i^2}}{\sqrt{\alpha n}}   \right. \\
	& \quad  \left. +   \frac{2\sqrt{2}K^{\frac{3}{2}}  \sqrt{(\omega+1)\delta_i^2 + \omega L_i^2}\rho t_i^2 \eta^2}{ \sqrt{n}}     \right. \\ 
	& \quad \left.  +   \sum_{k=0}^{K-1} \frac{1}{2} L_i^1 t_i^2 \eta^2 \E \left[  \|\nabla_i f(X^\tau) \|_{(i)\star} \right]   \right]. 
\end{align*}

Now we consider two options: (1) $L_i^1 = 0$ for all $i \in \{  1, ..., p  \}$ and (2) $L_i^1 \neq 0$, for all $i \in \{  1, ..., p  \}$.  

{\bf Case 1:} $L_i^1 = 0$ for all $i \in \{  1, ..., p  \}$. In this case, 
\begin{eqnarray*}
	&& \min_{k=0, ..., K-1} \sum_{i=1}^p t_i \E \left[  \|\nabla_i f(X^k) \|_{(i)\star} \right] \\ 
	&\leq&  \frac{1}{K} \sum_{k=0}^{K-1} \sum_{i=1}^p t_i \E \left[  \|\nabla_i f(X^k) \|_{(i)\star} \right] \\ 
	&\leq&  \frac{\Delta^0}{\eta K} + \frac{2\sum_{i=1}^p t_i \rho \sigma}{\alpha K\sqrt{Bn}}  +   \frac{2\sqrt{2}\sum_{i=1}^p t_i  \rho \sigma}{\sqrt{Bn}}  +    \frac{2\rho \sum_{i=1}^p t_i^2 \sqrt{(\omega+1)\delta_i^2 + \omega L_i^2}\eta}{\sqrt{\alpha n}}  \\ 
	&&  +    \frac{2\sqrt{2} \sqrt{K} \sum_{i=1}^p \sqrt{(\omega+1)\delta_i^2 + \omega L_i^2} t_i^2 \rho \eta}{\sqrt{n}}   + \frac{\sum_{i=1}^p L_i^0 t_i^2 \eta}{2}. 
\end{eqnarray*}

{\bf Case 2:} $L_i^1 \neq 0$, for all $i \in \{  1, ..., p  \}$. First we let $\eta \leq \min_i \frac{1}{L_i^1 t_i}$. Then $\frac{1}{2} L_i^1 t_i \eta \leq \frac{1}{2}$ for all $i$, and 
\begin{eqnarray*}
	&& \min_{k=0, ..., K-1} \sum_{i=1}^p t_i \E \left[  \|\nabla_i f(X^k) \|_{(i)\star} \right] \\ 
	&\leq&  \frac{1}{K} \sum_{k=0}^{K-1} \sum_{i=1}^p t_i \E \left[  \|\nabla_i f(X^k) \|_{(i)\star} \right] \\ 
	&\leq& \frac{2\Delta^0}{\eta K} + \frac{4\sum_{i=1}^p t_i \rho \sigma}{\alpha K\sqrt{Bn}}  +   \frac{4\sqrt{2}\sum_{i=1}^p t_i  \rho \sigma}{\sqrt{Bn}}  +    \frac{4\rho \sum_{i=1}^p t_i^2 \sqrt{(\omega+1)\delta_i^2 + \omega L_i^2}\eta}{\sqrt{\alpha n}}  \\ 
	&&  +    \frac{4\sqrt{2} \sqrt{K} \sum_{i=1}^p \sqrt{(\omega+1)\delta_i^2 + \omega L_i^2} t_i^2 \rho \eta}{\sqrt{n}}   +  + \sum_{i=1}^p L_i^0 t_i^2 \eta. 
\end{eqnarray*}

\newpage

\section{Proofs for Compressed Gluon with Error Feedback}

\subsection{A Lemma}

\begin{lemma}\label{lm:eik}
	Let Assumptions \ref{as:Lismooth} and \ref{as:HV} hold, Assume each $\C_i$ in Algorithm \ref{alg:q-deltagluon} is the contraction compressor satisfying (\ref{eq:Ci}) and $q+\delta-q\delta>0$. Then for $k\geq 0$, the $e_i^{\tau, k+1}$ in Algorithm \ref{alg:q-deltagluon} satisfies 
	\begin{equation}\label{eq:eik}
		\E \left[  \|e_i^{\tau, k+1}\|_2^2 \right] \leq  \frac{2(1-q)(1-\delta)(L_i^2 + (q+\delta-q\delta)\delta_i^2 ) t_i^2 \eta^2}{(q+\delta-q\delta)^2}. 
	\end{equation}
	
\end{lemma}

\begin{proof}
	For $k=0$, (\ref{eq:eik}) clearly holds since $e_i^{\tau, 1} = 0$. 
	For $k \geq 1$, 
	\begin{eqnarray*}
		&& \E \left[  \|e_i^{\tau, k+1}\|_2^2 \right] \nonumber \\
		&=& (1-q) \E \left[  \| e_i^{\tau, k} + \nabla_i f_{\xi^k_\tau} (X^k) - \nabla_i f_{\xi^k_\tau} (X^{k-1}) - y_i^{\tau, k}\|_2^2  \right] \nonumber \\ 
		&\overset{(a)}{\leq}&  (1-q)(1-\delta) \E \left[  \| e_i^{\tau, k} + \nabla_i f_{\xi^k_\tau} (X^k) - \nabla_i f_{\xi^k_\tau} (X^{k-1}) \|_2^2  \right] \nonumber \\ 
		&=&  (1-q)(1-\delta) \E \left[  \| e_i^{\tau, k} + \nabla_i f^\tau (X^k) - \nabla_i f^\tau (X^{k-1}) \|_2^2 \right] \nonumber \\ 
		&&  + (1-q)(1-\delta) \E \left[  \|  \nabla_i f_{\xi^k_\tau} (X^k) - \nabla_i f_{\xi^k_\tau} (X^{k-1}) - \left(  \nabla_i f^\tau (X^k) - \nabla_i f^\tau (X^{k-1})  \right) \|_2^2  \right] \nonumber \\ 
		&\overset{(b)}{\leq}&  (1-q)(1-\delta) \E \left[  \| e_i^{\tau, k} + \nabla_i f^\tau (X^k) - \nabla_i f^\tau (X^{k-1}) \|_2^2 \right]  + (1-q)(1-\delta) \delta_i^2 t_i^2 \eta^2 \nonumber \\ 
		&\overset{(c)}{\leq}& (1-q)(1-\delta)(1 + {\tilde \beta}) \E \left[  \|e_i^{\tau, k}\|_2^2  \right] \nonumber \\ 
		&& + (1-q)(1-\delta)\left(  1 + \frac{1}{\tilde \beta}  \right) \E \left[  \| \nabla_i f^\tau (X^k) - \nabla_i f^\tau (X^{k-1}) \|_2^2  \right] + (1-q)(1-\delta) \delta_i^2 t_i^2 \eta^2 \nonumber \\ 
		&\overset{(d)}{\leq}&  \left(  1 - \frac{q+\delta-q\delta}{2}  \right) \E \left[  \|e_i^{\tau, k}\|_2^2  \right] + \frac{2(1-q)(1-\delta)}{q+\delta-q\delta} \E \left[  \| \nabla_i f^\tau (X^k) - \nabla_i f^\tau (X^{k-1}) \|_2^2  \right] \nonumber \\
		&&  +  (1-q)(1-\delta) \delta_i^2 t_i^2 \eta^2 \nonumber \\ 
		&\overset{(e)}{\leq}& \left(  1 - \frac{q+\delta-q\delta}{2}  \right) \E \left[  \|e_i^{\tau, k}\|_2^2  \right] + \frac{2(1-q)(1-\delta)L_i^2 t_i^2 \eta^2}{q+\delta-q\delta} +  (1-q)(1-\delta) \delta_i^2 t_i^2 \eta^2  \nonumber \\ 
		&\leq& \frac{2(1-q)(1-\delta)L_i^2 t_i^2 \eta^2}{(q+\delta-q\delta)^2} +  \frac{2(1-q)(1-\delta) \delta_i^2 t_i^2 \eta^2}{q+\delta-q\delta} \nonumber \\ 
		&=&  \frac{2(1-q)(1-\delta)(L_i^2 + (q+\delta-q\delta)\delta_i^2 ) t_i^2 \eta^2}{(q+\delta-q\delta)^2}, 
	\end{eqnarray*}
	where (a) uses the contraction property of the compressor $\C_i$ in (\ref{eq:Ci}), (b) uses Assumption \ref{as:HV} and the update rule (\ref{eq:updateCon-dG}), (c) and (d) use Young’s inequality and choose ${\tilde \beta} = \frac{q+\delta-q\delta}{2(1-q-\delta+q\delta)}$ when $q+\delta-q\delta<1$, (e) uses Assumption \ref{as:Lismooth} and (\ref{eq:updateCon-dG}). When $q+\delta-q\delta=1$, it is easy to see that the above inequality also holds. 
	
\end{proof}

\subsection{Proof of Theorem \ref{th:deltaGluon}}

First we introduce some auxiliary notations. We define $e_i^k \eqdef \frac{1}{n} \sum_{\tau=1}^n e_i^{\tau, k}$, ${\tilde M}_i^k \eqdef M_i^k + e_i^{k+1}$, ${\tilde g}_i^k \eqdef g_i^k + e_i^{k+1}$, and $\alpha \eqdef 1-\beta$. Then from the update rule of $M_i^k$ in Algorithm \ref{alg:q-deltagluon}, we can get 
\begin{eqnarray*}
	{\tilde M}_i^k &=& M_i^k + e_i^{k+1} \\ 
	&=& \beta(M_i^{k-1} + e_i^k) + \alpha (g_i^k + e_i^{k+1}) + \beta (e_i^{k+1} - e_i^k) \\ 
	&=& \beta {\tilde M}_i^{k-1} + \alpha {\tilde g}_i^k +  \beta (e_i^{k+1} - e_i^k). 
\end{eqnarray*}
We also denote $\mu_i^k \eqdef {\tilde M}_i^k - \nabla_i f(X^k)$, $\gamma_i^k \eqdef {\tilde g}_i^k - \nabla_i f(X^k)$, and $S_i^k \eqdef \nabla_i f(X^{k-1}) - \nabla_i f(X^k)$. Then we can obtain 
\begin{eqnarray*}
	{\tilde M}_i^k - \nabla_i f(X^k) &=&  \beta ({\tilde M}_i^{k-1} - \nabla_i f(X^{k-1})) + \alpha ({\tilde g}_i^k - \nabla_i f(X^k)) \\ 
	&&  + \beta (\nabla_i f(X^{k-1}) - \nabla_i f(X^k)) + \beta (e_i^{k+1} - e_i^k) \\ 
	&=& \beta^k \mu_i^0 + \sum_{j=1}^k \beta^{k-j} \alpha \gamma_i^j + \sum_{j=1}^k \beta^{k+1-j} S_i^j + \sum_{j=1}^k \beta^{k+1-j} (e_i^{j+1} - e_i^j). 
\end{eqnarray*}
For the term $\sum_{j=1}^k \beta^{k+1-j} (e_i^{j+1} - e_i^j)$, we have 
\begin{eqnarray*}
	\sum_{j=1}^k \beta^{k+1-j} (e_i^{j+1} - e_i^j) - e_i^{k+1}&=&  \sum_{j=2}^{k+1} \beta^{k+2-j} e_i^j - \sum_{j=1}^{k+1} \beta^{k+1-j} e_i^j \\ 
	&=& \sum_{j=1}^{k+1} \beta^{k+2-j} e_i^j - \sum_{j=1}^{k+1} \beta^{k+1-j} e_i^j \\ 
	&=& -\alpha \sum_{j=1}^{k+1} \beta^{k+1-j} e_i^j, 
\end{eqnarray*}
where we use $e_i^1=0$ in the second equality. Combining the above two equalities, we arrive at 
\begin{eqnarray*}
	M_i^k - \nabla_i f(X^k) &=& {\tilde M}_i^k - \nabla_i f(X^k) - e_i^{k+1} \\ 
	&=& \beta^k \mu_i^0 + \sum_{j=1}^k \beta^{k-j} \alpha \gamma_i^j + \sum_{j=1}^k \beta^{k+1-j} S_i^j - \alpha \sum_{j=1}^{k+1} \beta^{k+1-j} e_i^j, 
\end{eqnarray*}
which yields that
\begin{align}
	& \E \left[  \|M_i^k - \nabla_i f(X^k)\|_{(i)\star}  \right] \nonumber \\ 
	\overset{(a)}{\leq}& \beta^k \E \left[  \|\mu_i^0\|_{(i)\star}  \right] + \E \left[  \| \sum_{j =1}^k \beta^{k-j} \alpha \gamma_i^j\|_{(i)\star} \right] + \sum_{j=1}^k \beta^{k+1-j} \E \left[  \|S_i^j \|_{(i)\star} \right] + \sum_{j=1}^{k+1} \alpha \beta^{k+1-j} \E \left[  \|e_i^j\|_{(i)\star}  \right] \nonumber \\ 
	\overset{(b)}{\leq}& \beta^k \rho \E \left[  \|\mu_i^0\|_2  \right] + \rho \E \left[  \|\sum_{j=1}^k \beta^{k-j} \alpha \gamma_i^j \|_2 \right] + \sum_{j=1}^k \beta^{k+1-j} \left(  L_i^0 + L_i^1 \E \left[  \nabla_i f(X^j)\|_{(i)\star}  \right]  \right) t_i \eta \nonumber \\ 
	&  + \sum_{j=1}^{k+1} \alpha \beta^{k+1-j} \rho \E \left[  \|e_i^j\|_2  \right] \nonumber \\ 
	\overset{(c)}{\leq}& \beta^k \rho \sqrt{\E \left[  \|\mu_i^0\|_2^2  \right]} + \rho \sqrt{\E \left[  \| \sum_{j=1}^k \beta^{k-j} \alpha \gamma_i^j \|_2^2 \right]} + \frac{L_i^0t_i\eta}{\alpha} + L_i^1 t_i \eta \sum_{j=1}^k \beta^{k+1-j} \E \left[  \|\nabla_i f(X^j)\|_{(i)\star}  \right] \nonumber \\ 
	& + \sum_{j=1}^{k+1} \alpha \beta^{k+1-j} \rho \sqrt{\E \left[  \|e_i^j\|_2^2  \right]}  \nonumber \\ 
	\overset{(d)}{\leq}& \frac{(1-\alpha)^k \rho \sigma}{\sqrt{Bn}} + \rho \sqrt{\E \left[  \| \sum_{j=1}^k \beta^{k-j} \alpha \gamma_i^j \|_2^2 \right]}  +  \frac{L_i^0t_i\eta}{\alpha} + L_i^1 t_i \eta \sum_{j=1}^k \beta^{k+1-j} \E \left[  \|\nabla_i f(X^j)\|_{(i)\star}  \right] \nonumber \\ 
	& + \sum_{j=1}^{k+1} \alpha \beta^{k+1-j} \rho \sqrt{\E \left[  \|e_i^j\|_2^2  \right]}, \label{eq:Mik-dG}
\end{align}
where (a) uses the triangle inequality, (b) uses Assumptions \ref{as:L0L1smooth} and \ref{as:rho}, (c) uses Jensen’s inequality, (d) uses Assumption \ref{as:boundedvariance} and the fact that samples $\xi^k_\tau \sim {\cal D}_\tau$ are i.i.d. 

Next we estimate $\E \left[  \| \sum_{j=1}^k \beta^{k-j} \alpha \gamma_i^j \|_2^2 \right]$ and $\E \left[  \|e_i^j\|_2^2 \right]$ respectively. Recall that $\gamma_i^k = {\tilde g}_i^k - \nabla_i f(X^k)$ and ${\tilde g}_i^k = g_i^k + e_i^{k+1}$. If $u^k=1$, then $e_i^{\tau, k+1}=0$, which implies that ${\tilde g}_i^k = g_i^k = \frac{1}{n}\sum_{\tau=1}^n g_i^{\tau, k} =  \frac{1}{Bn} \sum_{j=1}^B \sum_{\tau=1}^n \nabla_i f_{\xi^k_{\tau, j}} (X^k)$. If $u^k \neq 1$, then $e_i^{\tau, k+1} = e_i^{\tau, k} + \nabla_i f_{\xi^k_\tau} (X^k) - \nabla_i f_{\xi^k_\tau} (X^{k-1}) - y_i^{\tau, k}$. Thus for $u^k \neq 1$, 
\begin{eqnarray*}
	{\tilde g}_i^k &=& g_i^k + e_i^{k+1} \\ 
	&=& g_i^{k-1} + \frac{1}{n} \sum_{\tau=1}^n y_i^{\tau, k} + e_i^k + \frac{1}{n} \sum_{\tau=1}^n \left( \nabla_i f_{\xi^k_\tau} (X^k) - \nabla_i f_{\xi^k_\tau} (X^{k-1}) \right) - \frac{1}{n} \sum_{\tau=1}^n y_i^{\tau, k}  \\ 
	&=& {\tilde g}_i^{k-1} +  \frac{1}{n} \sum_{\tau=1}^n \left( \nabla_i f_{\xi^k_\tau} (X^k) - \nabla_i f_{\xi^k_\tau} (X^{k-1}) \right). 
\end{eqnarray*}
Then we get the following update rule for ${\tilde g}_i^k$. 
$$
{\tilde g}_i^k= \left\{ \begin{array}{cl}
	\frac{1}{Bn} \sum_{j=1}^B \sum_{\tau=1}^n \nabla_i f_{\xi^k_{\tau, j}} (X^k) & \mbox{ if $u^k = 1$} \\
	{\tilde g}_i^{k-1} + \frac{1}{n} \sum_{\tau=1}^n \left( \nabla_i f_{\xi^k_\tau} (X^k) - \nabla_i f_{\xi^k_\tau} (X^{k-1}) \right) &\mbox{ otherwise}
\end{array} \right.
$$
Hence ${\tilde g}_i^k$ behave the same as the $g_i^k$ in Algorithm \ref{alg:q-cgluon} with no compression. We apply Lemma \ref{lm:sumgammaij} with $\omega=0$ to ${\tilde g}_i^k$, then we can obtain that for $0<q\leq 1$, we have 
\begin{equation}\label{eq:sumgammaijboundq>0-dG}
	\E \left[  \| \sum_{j=1}^k \alpha \beta^{k-j} \gamma_i^j\|_2^2  \right]  \leq \frac{2\alpha}{(2-\alpha)(\alpha+\beta q)} \left(   \frac{2\sigma^2}{Bn}  + \frac{(1-q)\delta_i^2 t_i^2 \eta^2}{qn} \right); 
\end{equation}
for $q=0$, we have 
\begin{equation}\label{eq:sumgammaijboundq=0-dG}
	\E \left[  \| \sum_{j=1}^k \alpha \beta^{k-j} \gamma_i^j\|_2^2  \right]  \leq \frac{2\sigma^2}{Bn}  + \frac{2k \delta_i^2 t_i^2\eta^2}{n}. 
\end{equation}

For $\E \left[  \|e_i^j\|_2^2 \right]$, from Jensen's inequality, we have $\E \left[  \|e_i^j\|_2^2 \right] \leq \frac{1}{n} \sum_{\tau=1}^n \E \left[  \|e_i^{\tau, j}\|_2^2 \right]$. Then the estimation follows from Lemma \ref{lm:eik}.  

First we consider the case where $0<q\leq 1$. Combining (\ref{eq:Mik-dG}), (\ref{eq:sumgammaijboundq>0-dG}), and (\ref{eq:eik}) yields that 
\begin{eqnarray*}
	&& \E \left[  \|M_i^k - \nabla_i f(X^k)\|_{(i)\star}  \right] \\ 
	&\leq& \frac{(1-\alpha)^k \rho \sigma}{\sqrt{Bn}} +  \frac{2\sqrt{\alpha} \rho \sigma}{\sqrt{(2-\alpha)(\alpha+\beta q)Bn}}  + \frac{\sqrt{2\alpha(1-q)}\rho\delta_it_i\eta}{\sqrt{(2-\alpha)(\alpha+\beta q)qn}}  +  \frac{L_i^0t_i\eta}{\alpha} \\ 
	&& + L_i^1 t_i \eta \sum_{j=1}^k \beta^{k+1-j} \E \left[  \|\nabla_i f(X^j)\|_{(i)\star}  \right] +  \frac{\sqrt{2(1-q)(1-\delta)(L_i^2 + (q+\delta-q\delta)\delta_i^2 )} \rho t_i \eta}{q + \delta -q\delta}. 
\end{eqnarray*}

Then from (\ref{eq:sumfgrad-cs}), we can obtain that 
\begin{align*}
	\sum_{i=1}^p \sum_{k=0}^{K-1} t_i \eta \E \left[  \| \nabla_i f(X^k)\|_{(i)\star}  \right] \leq \Delta^0 + \sum_{i=1}^p  & \left[   \sum_{k=0}^{K-1} \frac{2(1-\alpha)^kt_i\eta\rho \sigma}{\sqrt{Bn}} + \sum_{k=0}^{K-1} \frac{4\sqrt{\alpha}t_i \eta \rho \sigma}{\sqrt{(2-\alpha)(\alpha+\beta q)} \sqrt{Bn}} \right. \\ 
	& \quad \left. + \sum_{k=0}^{K-1} \frac{2\sqrt{2(1-q)\alpha}\rho \delta_i t_i^2 \eta^2}{\sqrt{(2-\alpha)(\alpha+\beta q)} \sqrt{qn}}   + \sum_{k=0}^{K-1}  \frac{2L_i^0 t_i^2 \eta^2}{\alpha}   \right. \\ 
	& \quad \left. + \sum_{k=0}^{K-1} 2L_i^1 t_i^2 \eta^2 \sum_{j=1}^k \beta^{k+1-j} \E \left[  \|\nabla_i f(X^j) \|_{(i)\star}  \right] \right.  \\ 
	& \quad \left.  +  \sum_{k=0}^{K-1}   \frac{2\sqrt{2(1-q)(1-\delta)(L_i^2 + (q+\delta-q\delta)\delta_i^2 )} \rho t_i^2 \eta^2}{q + \delta -q\delta}    \right. \\
	& \quad  \left. + \sum_{k=0}^{K-1} \frac{L_i^0 t_i^2 \eta^2}{2} + \sum_{k=0}^{K-1}  \frac{L_i^1 t_i^2 \eta^2}{2} \E \left[  \nabla_i f(X^k)\|_{(i)\star}  \right] \right]. 
\end{align*}
Since 
$$
\sum_{k=0}^{K-1} \sum_{j=1}^k \beta^{k+1-j} \E \left[  \|\nabla_i f(X^j) \|_{(i)\star}  \right]  = \sum_{j=1}^{K-1} \sum_{k=j}^{K-1} \beta^{k+1-j}  \E \left[  \|\nabla_i f(X^j) \|_{(i)\star}  \right]  \leq \frac{1}{\alpha} \sum_{k=0}^{K-1}  \E \left[  \|\nabla_i f(X^k) \|_{(i)\star}  \right], 
$$
which implies that 
\begin{align*}
	\sum_{i=1}^p \sum_{k=0}^{K-1} t_i \eta \E \left[ \|  \nabla_i f(X^k)\|_{(i)\star}  \right] \leq \Delta^0 + \sum_{i=1}^p & \left[   \frac{2t_i \eta \rho \sigma}{\alpha\sqrt{Bn}} +  \frac{4K\sqrt{\alpha}t_i \eta \rho \sigma}{\sqrt{(2-\alpha)(\alpha+\beta q)} \sqrt{Bn}}  + \frac{2KL_i^0 t_i^2 \eta^2}{\alpha}  + \frac{K L_i^0 t_i^2 \eta^2}{2}   \right. \\ 
	& \quad  \left. +   \frac{2K\sqrt{2(1-q)\alpha} \rho \delta_i t_i^2 \eta^2}{\sqrt{(2-\alpha)(\alpha+\beta q)} \sqrt{qn}}     \right. \\ 
	& \quad  \left. +     \frac{2K\sqrt{2(1-q)(1-\delta)(L_i^2 + (q+\delta-q\delta)\delta_i^2 )} \rho t_i^2 \eta^2}{q + \delta -q\delta}    \right. \\ 
	& \quad \left.  +   \sum_{k=0}^{K-1} \left(  \frac{2}{\alpha} + \frac{1}{2}  \right) L_i^1 t_i^2 \eta^2 \E \left[  \|\nabla_i f(X^\tau) \|_{(i)\star} \right]   \right]. 
\end{align*}

Now we consider two options: (1) $L_i^1 = 0$ for all $i \in \{  1, ..., p  \}$ and (2) $L_i^1 \neq 0$, for all $i \in \{  1, ..., p  \}$.  

{\bf Case 1:} $L_i^1 = 0$ for all $i \in \{  1, ..., p  \}$. In this case, 
\begin{eqnarray*}
	&& \min_{k=0, ..., K-1} \sum_{i=1}^p t_i \E \left[  \|\nabla_i f(X^k) \|_{(i)\star} \right] \\ 
	&\leq&  \frac{1}{K} \sum_{k=0}^{K-1} \sum_{i=1}^p t_i \E \left[  \|\nabla_i f(X^k) \|_{(i)\star} \right] \\ 
	&\leq&  \frac{\Delta^0}{\eta K} + \frac{2\sum_{i=1}^p t_i \rho \sigma}{\alpha K \sqrt{Bn}} +   \frac{4\sqrt{\alpha}\sum_{i=1}^pt_i \rho \sigma}{\sqrt{(2-\alpha)(\alpha+\beta q)} \sqrt{Bn}} +    \frac{2\sqrt{2(1-q)\alpha} \sum_{i=1}^p\delta_i t_i^2 \rho \eta}{\sqrt{(2-\alpha)(\alpha+\beta q)} \sqrt{qn}}  \\ 
	&&  + \frac{2\sum_{i=1}^p L_i^0t_i^2 \eta}{\alpha} + \frac{\sum_{i=1}^p L_i^0 t_i^2 \eta}{2} + \frac{2\sqrt{2(1-q)(1-\delta)} \sum_{i=1}^p \sqrt{L_i^2 + (q+\delta-q\delta)\delta_i^2 }  t_i^2 \rho \eta}{q + \delta -q\delta}. 
\end{eqnarray*}

{\bf Case 2:} $L_i^1 \neq 0$, for all $i \in \{  1, ..., p  \}$. First we let $\frac{\eta}{\alpha} \leq \min_{i}  \frac{1}{5L_i^1 t_i}$. Then $\left(  \frac{2}{\alpha} + \frac{1}{2}  \right) L_i^1 t_i \eta \leq \frac{1}{2}$ for all $i$, and 
\begin{eqnarray*}
	&& \min_{k=0, ..., K-1} \sum_{i=1}^p t_i \E \left[  \|\nabla_i f(X^k) \|_{(i)\star} \right] \\ 
	&\leq&  \frac{1}{K} \sum_{k=0}^{K-1} \sum_{i=1}^p t_i \E \left[  \|\nabla_i f(X^k) \|_{(i)\star} \right] \\ 
	&\leq&  \frac{2\Delta^0}{\eta K} + \frac{4\sum_{i=1}^p t_i \rho \sigma}{\alpha K \sqrt{Bn}} +   \frac{8\sqrt{\alpha}\sum_{i=1}^pt_i \rho \sigma}{\sqrt{(2-\alpha)(\alpha+\beta q)} \sqrt{Bn}} +    \frac{4\sqrt{2(1-q)\alpha} \sum_{i=1}^p\delta_i t_i^2 \rho \eta}{\sqrt{(2-\alpha)(\alpha+\beta q)} \sqrt{qn}}  \\ 
	&&  + \frac{4\sum_{i=1}^p L_i^0t_i^2 \eta}{\alpha} + {\sum_{i=1}^p L_i^0 t_i^2 \eta} + \frac{4\sqrt{2(1-q)(1-\delta)} \sum_{i=1}^p \sqrt{L_i^2 + (q+\delta-q\delta)\delta_i^2 }  t_i^2 \rho \eta}{q + \delta -q\delta}. 
\end{eqnarray*}

Now we consider the case where $q=0$. Combining (\ref{eq:Mik-dG}), (\ref{eq:sumgammaijboundq=0-dG}), and (\ref{eq:eik}) implies that 
\begin{eqnarray*}
	&& \E \left[  \|M_i^k - \nabla_i f(X^k)\|_{(i)\star}  \right] \\ 
	&\leq& \frac{(1-\alpha)^k \rho \sigma}{\sqrt{Bn}} +  \frac{\sqrt{2} \rho \sigma}{\sqrt{Bn}}  + 
	\frac{\sqrt{2k} \rho \delta_i t_i \eta}{\sqrt{n}}  +  \frac{L_i^0t_i\eta}{\alpha} \\ 
	&& + L_i^1 t_i \eta \sum_{j=1}^k \beta^{k+1-j} \E \left[  \|\nabla_i f(X^j)\|_{(i)\star}  \right] +  \frac{\sqrt{2(1-\delta)(L_i^2 + \delta \delta_i^2 )} \rho t_i \eta}{ \delta}. 
\end{eqnarray*}
Then similar to the $0<q\leq 1$ case, we can obtain 
\begin{align*}
	\sum_{i=1}^p \sum_{k=0}^{K-1} t_i \eta \E \left[ \|  \nabla_i f(X^k)\|_{(i)\star}  \right] \leq \Delta^0 + \sum_{i=1}^p & \left[   \frac{2t_i \eta \rho \sigma}{\alpha\sqrt{Bn}} +  \frac{2\sqrt{2}K t_i \eta \rho \sigma}{\sqrt{Bn}}  + \frac{2KL_i^0 t_i^2 \eta^2}{\alpha}  + \frac{K L_i^0 t_i^2 \eta^2}{2}   \right. \\ 
	& \quad  \left. +   \frac{2\sqrt{2} K^{\frac{3}{2}} \rho \delta_i t_i^2 \eta^2}{\sqrt{n}}   +     \frac{2K\sqrt{2(1-\delta)(L_i^2 + \delta \delta_i^2 )} \rho t_i^2 \eta^2}{ \delta}    \right. \\ 
	& \quad \left.  +   \sum_{k=0}^{K-1} \left(  \frac{2}{\alpha} + \frac{1}{2}  \right) L_i^1 t_i^2 \eta^2 \E \left[  \|\nabla_i f(X^\tau) \|_{(i)\star} \right]   \right]. 
\end{align*}

Now we consider two options: (1) $L_i^1 = 0$ for all $i \in \{  1, ..., p  \}$ and (2) $L_i^1 \neq 0$, for all $i \in \{  1, ..., p  \}$.  

{\bf Case 1:} $L_i^1 = 0$ for all $i \in \{  1, ..., p  \}$. In this case, 
\begin{eqnarray*}
	&& \min_{k=0, ..., K-1} \sum_{i=1}^p t_i \E \left[  \|\nabla_i f(X^k) \|_{(i)\star} \right] \\ 
	&\leq&  \frac{1}{K} \sum_{k=0}^{K-1} \sum_{i=1}^p t_i \E \left[  \|\nabla_i f(X^k) \|_{(i)\star} \right] \\ 
	&\leq& \frac{\Delta^0}{\eta K} + \frac{2\sum_{i=1}^p t_i \rho \sigma}{\alpha K \sqrt{Bn}} +  \frac{2\sqrt{2}\sum_{i=1}^p t_i \rho \sigma}{\sqrt{Bn}}  +  \frac{2\sqrt{2}\sqrt{K} \sum_{i=1}^p \delta_i t_i^2 \rho \eta}{\sqrt{n}}  \\ 
	&&  + \frac{2\sum_{i=1}^p L_i^0t_i^2 \eta}{\alpha} + \frac{\sum_{i=1}^p L_i^0 t_i^2 \eta}{2} + \frac{2\sqrt{2(1-\delta)} \sum_{i=1}^p \sqrt{L_i^2 + \delta \delta_i^2 }  t_i^2 \rho \eta}{ \delta}. 
\end{eqnarray*}

{\bf Case 2:} $L_i^1 \neq 0$, for all $i \in \{  1, ..., p  \}$. First we let $\frac{\eta}{\alpha} \leq \min_{i}  \frac{1}{5L_i^1 t_i}$. Then $\left(  \frac{2}{\alpha} + \frac{1}{2}  \right) L_i^1 t_i \eta \leq \frac{1}{2}$ for all $i$, and 
\begin{eqnarray*}
	&& \min_{k=0, ..., K-1} \sum_{i=1}^p t_i \E \left[  \|\nabla_i f(X^k) \|_{(i)\star} \right] \\ 
	&\leq&  \frac{1}{K} \sum_{k=0}^{K-1} \sum_{i=1}^p t_i \E \left[  \|\nabla_i f(X^k) \|_{(i)\star} \right] \\ 
	&\leq& \frac{2\Delta^0}{\eta K} + \frac{4\sum_{i=1}^p t_i \rho \sigma}{\alpha K \sqrt{Bn}} +  \frac{4\sqrt{2}\sum_{i=1}^p t_i \rho \sigma}{\sqrt{Bn}}  +  \frac{4\sqrt{2}\sqrt{K} \sum_{i=1}^p \delta_i t_i^2 \rho \eta}{\sqrt{n}}  \\ 
	&&  + \frac{4\sum_{i=1}^p L_i^0t_i^2 \eta}{\alpha} + {\sum_{i=1}^p L_i^0 t_i^2 \eta} + \frac{4\sqrt{2(1-\delta)} \sum_{i=1}^p \sqrt{L_i^2 + \delta \delta_i^2 }  t_i^2 \rho \eta}{ \delta}. 
\end{eqnarray*}

\newpage

\section{Compressed Gluon with Error Feedback and MVR}

\begin{algorithm}
	\caption{Compressed Gluon with Error Feedback and MVR}
	\label{alg:q-deltagluon-mvr}
	\begin{algorithmic}[1]
		\STATE \textbf{Input:} Initial model parameters $X^0 = [X_1^0, \dots, X_p^0] \in \mathcal{S}$, momentum $M^0 = [M_1^0, \dots, M_p^0] \in \mathcal{S}$, momentum decay factors $\beta \in [0, 1)$ for all iterations $k \geq 0$, probability $q \in (0, 1]$, batch size $B$, stepsize parameter $\eta>0$, $u^0=1 \in \R$, $e_i^{\tau, 0} = e_i^{\tau, 1} = 0 \in {\cal S}_i$
		\FOR{ $k = 0, 1, 2, ..., K-1$}
		\FOR{ $\tau = 1, ..., n$} 
		\STATE $u^{k+1}_\tau = 0$ for $\tau = 2, ..., n$, \quad 
		$
		u^{k+1}_1 = \left\{ \begin{array}{rl}
			1 & \mbox{ with probability $q$} \\
			0 &\mbox{ with probability $1-q$}
		\end{array} \right.
		$\\
		\IF{$u^k = 1$}
		\STATE Sample $\xi^k_\tau \sim \mathcal{D_\tau}$ 
		\STATE $y_i^{\tau, k} = e_i^{\tau, k} + \nabla_i f_{\xi^k_\tau} (X^k) - \nabla_i f_{\xi^k_\tau} (X^{k-1})$ for $i = 1, ..., p$ 
		\STATE $e_i^{\tau, k+1} = 0$ for $i = 1, ..., p$ 
		\STATE Sample $\xi^k_{\tau, j} \sim \mathcal{D_\tau}$ independently for $j = 1, ..., B$
		\STATE $g_i^{\tau, k} = \frac{1}{B} \sum_{j=1}^B \nabla_i f_{\xi^k_{\tau, j}} (X^k)$ for $i = 1, ..., p$ 
		\STATE Send $g_i^{\tau, k}$, $y_i^{\tau, k}$, and $u^{k+1}_\tau$ to the other nodes
		\STATE Receive $g_i^{\tau, k}$, $y_i^{\tau, k}$, and $u^{k+1}_\tau$ from the other nodes
		\ELSE
		\STATE Sample $\xi^k_\tau \sim \mathcal{D_\tau}$ 
		\STATE  $y_i^{\tau, k} = \C_i\left(  \nabla_i f_{\xi^k_\tau} (X^k) - \nabla_i f_{\xi^k_\tau} (X^{k-1} )+ e_i^{\tau, k}  \right)$ for $i = 1, ..., p$ 
		\STATE $e_i^{\tau, k+1} = e_i^{\tau, k} + \nabla_i f_{\xi^k_\tau} (X^k) - \nabla_i f_{\xi^k_\tau} (X^{k-1}) - y_i^{\tau, k}$ for $i = 1, ..., p$ 
		\STATE Send $y_i^{\tau, k}$ and $u^{k+1}_\tau$ to the other nodes
		\STATE Receive $y_i^{\tau, k}$ and $u^{k+1}_\tau$ from the other nodes
		\ENDIF
		
		\ENDFOR
		\STATE $y_i^k = \frac{1}{n} \sum_{\tau=1}^n y_i^{\tau, k}$ 
		\STATE 
		$
		g_i^k= \left\{ \begin{array}{cl}
			\frac{1}{n} \sum_{\tau=1}^n g_i^{\tau, k} & \mbox{ if $u^k = 1$} \\
			g_i^{k-1} + y_i^k &\mbox{ otherwise}
		\end{array} \right.
		$
		\STATE $u^{k+1} = \sum_{\tau=1}^n u^{k+1}_\tau$
		\STATE Update momentum $M_i^k = \beta M_i^{k-1} + (1-\beta) g_i^k + \beta y_i^k$ for layer $i$ 
		\STATE Choose adaptive stepsize/radius $t_i \eta > 0$ for layer $i$
		\STATE Update parameters for layer $i$ via LMO over $\mathcal{B}_i^k := \{X_i \in \mathcal{S}_i : \|X_i - X_i^k\|_{(i)} \leq t_i \eta\}$:
		\begin{equation}\label{eq:updateCon-dG-mvr}
			X_i^{k+1} = \text{LMO}_{\mathcal{B}_i^k}(M_i^k) := \arg \min_{X_i \in \mathcal{B}_i^k} \langle M_i^k, X_i \rangle_{(i)}
		\end{equation}
		\STATE Update full parameter vector $X^{k+1} = [X_1^{k+1}, \dots, X_p^{k+1}]$
		\ENDFOR
	\end{algorithmic}
\end{algorithm}

\begin{theorem}\label{th:deltaGluon-mvr}
	Let Assumptions \ref{as:L0L1smooth}, \ref{as:Lismooth}, \ref{as:boundedvariance}, \ref{as:rho}, and \ref{as:HV} hold. Assume each $\C_i$ in Algorithm \ref{alg:q-deltagluon-mvr} is the contraction compressor satisfying (\ref{eq:Ci}) and $q + \delta - q\delta >0$.   Let $X^0, ..., X^{K-1}$ be the iterates of Algorithm \ref{alg:q-deltagluon-mvr}, and $M_i^0 = \frac{1}{Bn} \sum_{\tau=1}^n \sum_{j=1}^B \nabla_i f_{\xi_{\tau, j}^k} (X^0)$. 
	
	1. If $L_i^1 =0$, then for $0<q\leq 1$, 
	\begin{eqnarray}
		&& \min_{k=0, ..., K-1} \sum_{i=1}^p t_i \E \left[  \|\nabla_i f(X^k) \|_{(i)\star} \right] \nonumber \\ 
		&\leq& \frac{\Delta^0}{\eta K} + \frac{2\sum_{i=1}^p t_i \rho \sigma}{\alpha K\sqrt{Bn}}  +   \frac{4\sqrt{\alpha}\sum_{i=1}^p t_i  \rho \sigma}{\sqrt{(2-\alpha)(\alpha+\beta q)} \sqrt{Bn}}  +    \frac{2 \sqrt{2(1-q)\alpha} \sum_{i=1}^p \delta_i t_i^2 \rho \eta}{\sqrt{(2-\alpha)(\alpha+\beta q)} \sqrt{qn}}  +  \frac{2\sum_{i=1}^p \delta_i t_i^2 \rho \eta}{\sqrt{\alpha n}}  \nonumber  \\ 
		&&  +     \frac{2\sqrt{2(1-q)(1-\delta)} \sum_{i=1}^p \sqrt{L_i^2 + (q+\delta-q\delta)\delta_i^2 }  t_i^2 \rho \eta}{q + \delta -q\delta}    + \frac{\sum_{i=1}^p L_i^0 t_i^2 \eta}{2}; \label{eq:th-deltaGluon-mvr-1}
	\end{eqnarray}
	for $q=0$, 
	\begin{eqnarray*}
		&& \min_{k=0, ..., K-1} \sum_{i=1}^p t_i \E \left[  \|\nabla_i f(X^k) \|_{(i)\star} \right] \nonumber   \\ 
		&\leq&  \frac{\Delta^0}{\eta K} + \frac{2\sum_{i=1}^p t_i \rho \sigma}{\alpha K\sqrt{Bn}}   +   \frac{2\sqrt{2} \sum_{i=1}^p t_i \rho \sigma}{\sqrt{Bn}}  +   \frac{2\sqrt{2} \sqrt{K}  \sum_{i=1}^p \delta_i t_i^2 \rho \eta}{\sqrt{n}}   +  \frac{2 \sum_{i=1}^p \delta_i t_i^2 \rho \eta}{\sqrt{\alpha n}}  \nonumber    \\ 
		&&  +     \frac{2\sqrt{2(1-\delta)} \sum_{i=1}^p \sqrt{L_i^2 + \delta \delta_i^2 }  t_i^2 \rho \eta}{ \delta}    + \frac{\sum_{i=1}^p L_i^0 t_i^2 \eta}{2}. \label{eq:th-deltaGluon-mvr-2}
	\end{eqnarray*}

	2. If $L_i^1 \neq 0$, we let $\eta \leq \min_i \frac{1}{L_i^1 t_i}$. Then the inequalities (\ref{eq:th-deltaGluon-mvr-1}) and (\ref{eq:th-deltaGluon-mvr-2}) remain valid when their right-hand sides are multiplied by a factor $2$. 
\end{theorem}

From Theorem \ref{th:deltaGluon-mvr}, we can get the following convergence rate under suitable choices of parameters. 

\begin{theorem}\label{th:deltaGluon-mvr-rate}
Under the premise of Theorem \ref{th:deltaGluon}, let $\alpha = \Theta(1)$, $q>0$, and $\frac{\Delta^0}{\eta^2 K} = {\cal L}_4$, where ${\cal L}_4 \eqdef  \sum_{i=1}^p L_i^0t_i^2 + \frac{\sqrt{1-q} + \sqrt{q}}{\sqrt{qn}} \sum_{i=1}^p \rho \delta_i t_i^2 + \frac{\sqrt{(1-q)(1-\delta)}}{q+\delta-q\delta} \cdot \sum_{i=1}^p \rho \sqrt{L_i^2 + (q+\delta-q\delta)\delta_i^2}t_i^2$. Then the upper-bound in (\ref{eq:th-deltaGluon-mvr-1}) becomes 
$$
{\cal O} \left(  \frac{\sqrt{\Delta^0 {\cal L}_4} }{\sqrt{K}}  +  \frac{\sum_it_i \rho \sigma}{\sqrt{Bn}}  \right). 
$$
\end{theorem}

\subsection{Local Gluon with Error Compensated MVR}

In this subsection, we consider the case where $\C_i (X_i) \equiv 0$ for all $X_i \in {\cal S}_i$, which indicates that $\delta \equiv 0$. First we consider the case where $L_i^1=0$. Assume $n \leq \sqrt{K}$. Then by choosing $\eta = \frac{n^{1/3}}{K^{2/3}}$, $q = \frac{n^{2/3}}{K^{1/3}}$, $B = \frac{K^{1/3}}{n^{2/3}}$, and $\alpha = \frac{n^{1/3}}{K^{2/3}}$, from Theorem \ref{th:deltaGluon-mvr}, we have 
$$
\min_{k=0, ..., K-1} \sum_{i=1}^p t_i \E \left[  \|\nabla_i f(X^k) \|_{(i)\star} \right]  \leq {\cal O} \left(  \frac{1}{n^{1/3} K^{1/3}}  \right). 
$$
To achieve the $\epsilon$ precision, it is sufficient to choose $K = \frac{1}{\epsilon^3 n}$, and the expected totally communication cost is 
$$
\Theta\left( nd + Kn \left(1 + qd  \right) \right)  = \Theta\left(  \frac{1}{ \epsilon^3 } + \frac{n d}{\epsilon^2}  \right). 
$$
In this case, $n\leq \sqrt{K}$ is equivalent to $n\leq \frac{1}{\epsilon}$. The batch size $B = \frac{1}{\epsilon n}$, $\eta = \alpha = n\epsilon^2$, $q = n\epsilon$, and the expected stochastic gradient oracle is $\Theta(\frac{1}{\epsilon^3})$. Assume the data sampled from distribution $\mathcal{D_\tau}$ each time is different. Then the number of data $N$ should be $ \frac{1}{\epsilon^3}$.
In this parameter setting, if we omit the communication cost of $u_\tau^k$, then the worker nodes only need to communicate with each other when $u^k = 1$. And when $u^k=1$, the accumulation of $\nabla_i f_{\xi^k_\tau} (X^k) - \nabla_i f_{\xi^k_\tau} (X^{k-1})$, which is the added term in MVR, is communicated and added to $M_i^k$. Hence we call Algorithm \ref{alg:q-deltagluon-mvr} in this parameter setting Local Gluon with error compensated MVR. 

For the case where $L_i^1 \neq 0$, the above results also hold as long as $\eta \leq \min_{i}  \frac{1}{L_i^1 t_i}$, which generally holds for large $K$ in practice.

\subsection{Communication Cost Analysis}

In this subsection, we consider the compressed case. Similar to the analysis in Subsection \ref{subsec:comcost-dG}, the expected totally communication cost is at least $\Theta\left(  \frac{n}{\epsilon^2 (q+\delta-q\delta) } + \frac{n d}{\epsilon^2}  \right)$. 

First we consider the case where $L_i^1=0$. By choosing $\eta = \frac{\sqrt{c} n^{1/4}}{K^{3/4}}$, $\delta = \frac{cn^{1/2}}{K^{1/2}}$, $B = \frac{cK^{1/2}}{n^{1/2}}$, $\frac{c^2}{K} \leq q \leq \delta$, and $\alpha \geq \{  q, \frac{1}{K}  \}$ for $0<c \leq \frac{K^{1/2}}{n^{1/2}}$, from Theorem \ref{th:deltaGluon-mvr}, we can obtain 
$$
\min_{k=0, ..., K-1} \sum_{i=1}^p t_i \E \left[  \|\nabla_i f(X^k) \|_{(i)\star} \right]  \leq {\cal O} \left(  \frac{1}{\sqrt{c} n^{1/4} K^{1/4}}  \right). 
$$
To achieve the $\epsilon$ precision, it is sufficient to choose $K = \frac{1}{c^2 \epsilon^4 n}$, and the expected totally communication cost is 
$$
\Theta\left( nd + Kn \left(1 + qd + (1-q)\delta d  \right) \right)  = \Theta\left(  \frac{1}{ c^2 \epsilon^4 } + \frac{n d}{\epsilon^2}  \right). 
$$
Compared to the totally communication cost of $\Theta\left(  \frac{d}{\epsilon^3}  \right)$ for the uncompressed algorithm in subsection \ref{subsec:new-vr-alg}, the above communication cost is better when $n \leq \frac{1}{\epsilon}$. In this parameter setting, the batch size $B = \frac{1}{\epsilon^2 n}$ and by choosing $q = \frac{c^2}{K}$, the expected stochastic gradient oracle is 
$$
\Theta\left(  Bn + (1-q + qB)Kn \right) = \Theta \left(  \frac{1}{c^2 \epsilon^4}  +  \frac{c^2}{\epsilon^2}  \right), 
$$
with $0<c^2 \leq \frac{1}{\epsilon^2 n}$. If $n \leq \frac{1}{\epsilon}$,  then by choosing $c^2 = \frac{1}{\epsilon} = (Kn)^{1/3}$, we have $q=n\epsilon^2$, $B=\frac{1}{\epsilon^2 n}$, $\eta=n\epsilon^2$, $\delta=n\epsilon$, $\alpha \geq n \epsilon^2$, $K = \frac{1}{\epsilon^3 n}$, and the above expected stochastic gradient oracle becomes $\Theta \left(  \frac{1}{\epsilon^3}  \right)$. 

We can also use another parameter setting to get $\Theta \left(  \frac{1}{\epsilon^3}  \right)$ expected stochastic gradient oracle. Assume $n\leq \sqrt{K}$. Then by choosing $\eta = \alpha = \frac{n^{1/3}}{K^{2/3}}$, $\delta = \frac{n^{2/3}}{K^{1/3}}$, $\alpha \leq q\leq \delta$, $B = 1/q$, from Theorem \ref{th:deltaGluon-mvr}, we can get 
$$
\min_{k=0, ..., K-1} \sum_{i=1}^p t_i \E \left[  \|\nabla_i f(X^k) \|_{(i)\star} \right]  \leq {\cal O} \left(  \frac{1}{ n^{1/3} K^{1/3}}  \right). 
$$
To achieve the $\epsilon$ precision, it is sufficient to choose $K = \frac{1}{\epsilon^3 n}$, and the expected totally communication cost is 
$$
\Theta\left( nd + Kn \left(1 + qd + (1-q)\delta d \right) \right)  = \Theta\left(  \frac{1}{ \epsilon^3 } + \frac{n d}{\epsilon^2}  \right). 
$$
In this case, $n\leq \sqrt{K}$ is equivalent to $n\leq \frac{1}{\epsilon}$, $\eta = \alpha = n\epsilon^2$, $\delta = n\epsilon$, $n\epsilon^2 \leq q\leq n\epsilon$, $B = 1/q$, and the expected stochastic gradient oracle is $\Theta(\frac{1}{\epsilon^3})$.

When $q=0$, by choosing $\eta = \frac{\sqrt{c} n^{1/4}}{K^{3/4}}$, $\delta = \frac{cn^{1/2}}{K^{1/2}}$, $B = \frac{cK^{1/2}}{n^{1/2}}$, and $\alpha \geq \max\{ \frac{1}{K}, \frac{c^2}{K} \}$ for $0<c \leq \min \{\frac{K^{1/2}}{n^{1/2}}, 1\}$, from Theorem \ref{th:deltaGluon}, we can get 
$$
\min_{k=0, ..., K-1} \sum_{i=1}^p t_i \E \left[  \|\nabla_i f(X^k) \|_{(i)\star} \right]  \leq {\cal O} \left(  \frac{1}{\sqrt{c} n^{1/4} K^{1/4}}  \right). 
$$
To achieve the $\epsilon$ precision, it is sufficient to choose $K = \frac{1}{c^2 \epsilon^4 n}$, and the expected totally communication cost is 
$$
\Theta\left( nd + Kn \left( \delta d  \right) \right)  = \Theta\left(  \frac{n d}{\epsilon^2}  \right). 
$$
The batch size $B = \frac{1}{\epsilon^2 n}$ and the expected stochastic gradient oracle is $\Theta\left(  \frac{1}{c^2 \epsilon^4}  \right)$. 

For the case where $L_i^1 \neq 0$, the above results also hold as long as $\eta \leq \min_{i}  \frac{1}{L_i^1 t_i}$, which generally holds for large $K$ in practice.

\subsection{Proof of Theorem \ref{th:deltaGluon-mvr}}

First we also introduce some auxiliary notations similar to those of Algorithm \ref{alg:q-deltagluon}. We define $e_i^k \eqdef \frac{1}{n} \sum_{\tau=1}^n e_i^{\tau, k}$, ${\tilde M}_i^k \eqdef M_i^k + e_i^{k+1}$, ${\tilde g}_i^k \eqdef g_i^k + e_i^{k+1}$, and $\alpha \eqdef 1-\beta$. Then from the update rule of $M_i^k$ in Algorithm \ref{alg:q-deltagluon-mvr}, we can get 
\begin{eqnarray*}
	{\tilde M}_i^k &=& M_i^k + e_i^{k+1} \\ 
	&=& \beta(M_i^{k-1} + e_i^k) + \alpha (g_i^k + e_i^{k+1}) + \beta (e_i^{k+1} - e_i^k) + \beta y_i^k \\ 
	&=& \beta {\tilde M}_i^{k-1} + \alpha {\tilde g}_i^k +  \beta (y_i^k + e_i^{k+1} - e_i^k) . 
\end{eqnarray*}

If $u^k \neq 1$, we have $e_i^{\tau, k+1} = e_i^{\tau, k} + \nabla_i f_{\xi^k_\tau} (X^k) - \nabla_i f_{\xi^k_\tau} (X^{k-1}) - y_i^{\tau, k}$ for $i = 1, ..., p$. If $u^k=1$, we have $e_i^{\tau, k+1} = 0 = e_i^{\tau, k} + \nabla_i f_{\xi^k_\tau} (X^k) - \nabla_i f_{\xi^k_\tau} (X^{k-1}) - y_i^{\tau, k}$. Thus we always have $y_i^k + e_i^{k+1} - e_i^k = \frac{1}{n} \sum_{\tau=1}^n (\nabla_i f_{\xi^k_\tau} (X^k) - \nabla_i f_{\xi^k_\tau} (X^{k-1}))$, which implies that 
$$
{\tilde M}_i^k =  \beta {\tilde M}_i^{k-1} + \alpha {\tilde g}_i^k   + \beta \frac{1}{n} \sum_{\tau=1}^n (\nabla_i f_{\xi^k_\tau} (X^k) - \nabla_i f_{\xi^k_\tau} (X^{k-1})). 
$$
We also denote $\mu_i^k \eqdef {\tilde M}_i^k - \nabla_i f(X^k)$, $\gamma_i^k \eqdef {\tilde g}_i^k - \nabla_i f(X^k)$, and $Z_i^k \eqdef \frac{1}{n} \sum_{\tau=1}^n (\nabla_i f_{\xi^k_\tau} (X^k) - \nabla_i f_{\xi^k_\tau} (X^{k-1})) -( \nabla_i f(X^k) - \nabla_i f(X^{k-1}))$. Then we can obtain 
\begin{eqnarray*}
	{\tilde M}_i^k - \nabla_i f(X^k) &=&  \beta ({\tilde M}_i^{k-1} - \nabla_i f(X^{k-1})) + \alpha ({\tilde g}_i^k - \nabla_i f(X^k)) + \beta Z_i^k\\ 
	&=& \beta^k \mu_i^0 + \sum_{j=1}^k \beta^{k-j} \alpha \gamma_i^j + \sum_{j=1}^k \beta^{k+1-j} Z_i^j, 
\end{eqnarray*} 
which indicates that 
\begin{eqnarray*}
	M_i^k - \nabla_i f(X^k) &=& {\tilde M}_i^k - \nabla_i f(X^k) - e_i^{k+1} \\ 
	&=& \beta^k \mu_i^0 + \sum_{j=1}^k \beta^{k-j} \alpha \gamma_i^j + \sum_{j=1}^k \beta^{k+1-j} Z_i^j - e_i^{k+1}. 
\end{eqnarray*}
Then similar to (\ref{eq:Mik-dG}), we can get 
\begin{eqnarray*}
	&&\E \left[  \|M_i^k - \nabla_i f(X^k)\|_{(i)\star}  \right] \\ 
	&\leq& \frac{(1-\alpha)^k \rho \sigma}{\sqrt{Bn}} + \rho \sqrt{\E \left[  \| \sum_{j=1}^k \beta^{k-j} \alpha \gamma_i^j \|_2^2 \right]}  + \rho \sqrt{\E \left[  \| \sum_{j=1}^k \beta^{k+1-j} Z_i^j \|_2^2 \right]} + \rho \sqrt{\E \left[  \|e_i^{k+1}\|_2^2  \right]}. 
\end{eqnarray*}
It is easy to verify that $\gamma_i^k$ evolves the same as that of Algorithm \ref{alg:q-deltagluon}. Hence we have the estimations (\ref{eq:sumgammaijboundq>0-dG}) and (\ref{eq:sumgammaijboundq=0-dG}) for $\E \left[  \| \sum_{j=1}^k \alpha \beta^{k-j} \gamma_i^j \|_2^2  \right]$ as well. 

For $\E \left[  \| \sum_{j=1}^k \beta^{k+1-j} Z_i^j \|_2^2 \right]$, we can apply (\ref{eq:sumZij}) with no compression. Then we can obtain 
$$
\E \left[  \| \sum_{j=1}^k \beta^{k+1-j} Z_i^j \|_2^2 \right] \leq \frac{\delta_i^2 t_i^2 \eta^2}{\alpha n}. 
$$
For $\E\left[  \|e_i^{k+1}\|_2^2  \right]$, from Lemma \ref{lm:eik}. we have 
$$
\E\left[  \|e_i^{k+1}\|_2^2  \right] \leq \frac{1}{n} \sum_{\tau=1}^n \E \left[  \|e_i^{\tau, k+1}\|_2^2  \right] \leq \frac{2(1-q)(1-\delta)(L_i^2 + (q+\delta-q\delta)\delta_i^2 ) t_i^2 \eta^2}{(q+\delta-q\delta)^2}. 
$$

Combining the above estimations, we can obtain that 
\begin{eqnarray}
	&&\E \left[  \|M_i^k - \nabla_i f(X^k)\|_{(i)\star}  \right] \nonumber \\ 
	&\leq& \frac{(1-\alpha)^k \rho \sigma}{\sqrt{Bn}} +  \frac{2\sqrt{\alpha} \rho \sigma}{\sqrt{(2-\alpha)(\alpha+\beta q)Bn}}  + \frac{\sqrt{2\alpha(1-q)}\rho\delta_it_i\eta}{\sqrt{(2-\alpha)(\alpha+\beta q)qn}}  \nonumber  \\ 
	&& +  \frac{\rho \delta_i t_i \eta}{\sqrt{\alpha n}} +  \frac{\sqrt{2(1-q)(1-\delta)(L_i^2 + (q+\delta-q\delta)\delta_i^2 )} \rho t_i \eta}{q + \delta -q\delta}, \label{eq:Mik-dG-mvr-q>0}
\end{eqnarray}
for $0<q\leq 1$; and 
\begin{eqnarray}
	&&\E \left[  \|M_i^k - \nabla_i f(X^k)\|_{(i)\star}  \right] \nonumber \\ 
	&\leq& \frac{(1-\alpha)^k \rho \sigma}{\sqrt{Bn}} +  \frac{\sqrt{2} \rho \sigma}{\sqrt{Bn}}  + 
	\frac{\sqrt{2k} \rho \delta_i t_i \eta}{\sqrt{n}}  \nonumber \\
	&&  +  \frac{\rho \delta_i t_i \eta}{\sqrt{\alpha n}} +  \frac{\sqrt{2(1-q)(1-\delta)(L_i^2 + (q+\delta-q\delta)\delta_i^2 )} \rho t_i \eta}{q + \delta -q\delta},  \label{eq:Mik-dG-mvr-q=0}
\end{eqnarray}
for $q=0$. 

First we consider the case where $0<q\leq 1$. From (\ref{eq:sumfgrad-cs}) and (\ref{eq:Mik-dG-mvr-q>0}), we can get that 
\begin{align*}
	\sum_{i=1}^p \sum_{k=0}^{K-1} t_i \eta \E \left[  \| \nabla_i f(X^k)\|_{(i)\star}  \right] \leq \Delta^0 + \sum_{i=1}^p  & \left[   \frac{2 t_i\eta\rho \sigma}{\alpha \sqrt{Bn}} +  \frac{4K \sqrt{\alpha}t_i \eta \rho \sigma}{\sqrt{(2-\alpha)(\alpha+\beta q)} \sqrt{Bn}} \right. \\ 
	& \quad \left. +  \frac{2K \sqrt{2(1-q)\alpha}\rho \delta_i t_i^2 \eta^2}{\sqrt{(2-\alpha)(\alpha+\beta q)} \sqrt{qn}}  +  \frac{2K\rho \delta_i t_i^2 \eta^2}{\sqrt{\alpha n}}    \right. \\ 
	& \quad \left.  +    \frac{2K\sqrt{2(1-q)(1-\delta)(L_i^2 + (q+\delta-q\delta)\delta_i^2 )} \rho t_i^2 \eta^2}{q + \delta -q\delta}    \right. \\
	& \quad  \left. +  \frac{K L_i^0 t_i^2 \eta^2}{2} + \sum_{k=0}^{K-1}  \frac{L_i^1 t_i^2 \eta^2}{2} \E \left[  \nabla_i f(X^k)\|_{(i)\star}  \right] \right]. 
\end{align*}

Now we consider two options: (1) $L_i^1 = 0$ for all $i \in \{  1, ..., p  \}$ and (2) $L_i^1 \neq 0$, for all $i \in \{  1, ..., p  \}$.  

{\bf Case 1:} $L_i^1 = 0$ for all $i \in \{  1, ..., p  \}$. In this case, 
\begin{eqnarray*}
	&& \min_{k=0, ..., K-1} \sum_{i=1}^p t_i \E \left[  \|\nabla_i f(X^k) \|_{(i)\star} \right] \\ 
	&\leq&  \frac{1}{K} \sum_{k=0}^{K-1} \sum_{i=1}^p t_i \E \left[  \|\nabla_i f(X^k) \|_{(i)\star} \right] \\ 
	&\leq&  \frac{\Delta^0}{\eta K} + \frac{2\sum_{i=1}^p t_i \rho \sigma}{\alpha K\sqrt{Bn}}  +   \frac{4\sqrt{\alpha}\sum_{i=1}^p t_i  \rho \sigma}{\sqrt{(2-\alpha)(\alpha+\beta q)} \sqrt{Bn}}  +    \frac{2 \sqrt{2(1-q)\alpha} \sum_{i=1}^p \delta_i t_i^2 \rho \eta}{\sqrt{(2-\alpha)(\alpha+\beta q)} \sqrt{qn}}  +  \frac{2\sum_{i=1}^p \delta_i t_i^2 \rho \eta}{\sqrt{\alpha n}}  \\ 
	&&  +     \frac{2\sqrt{2(1-q)(1-\delta)} \sum_{i=1}^p \sqrt{L_i^2 + (q+\delta-q\delta)\delta_i^2 }  t_i^2 \rho \eta}{q + \delta -q\delta}    + \frac{\sum_{i=1}^p L_i^0 t_i^2 \eta}{2}. 
\end{eqnarray*}

{\bf Case 2:} $L_i^1 \neq 0$, for all $i \in \{  1, ..., p  \}$. First we let $\eta \leq \min_i \frac{1}{L_i^1 t_i}$. Then $\frac{1}{2} L_i^1 t_i \eta \leq \frac{1}{2}$ for all $i$, and 
\begin{eqnarray*}
	&& \min_{k=0, ..., K-1} \sum_{i=1}^p t_i \E \left[  \|\nabla_i f(X^k) \|_{(i)\star} \right] \\ 
	&\leq&  \frac{1}{K} \sum_{k=0}^{K-1} \sum_{i=1}^p t_i \E \left[  \|\nabla_i f(X^k) \|_{(i)\star} \right] \\ 
	&\leq&  \frac{2\Delta^0}{\eta K} + \frac{4\sum_{i=1}^p t_i \rho \sigma}{\alpha K\sqrt{Bn}}  +   \frac{8\sqrt{\alpha}\sum_{i=1}^p t_i  \rho \sigma}{\sqrt{(2-\alpha)(\alpha+\beta q)} \sqrt{Bn}}  +    \frac{4 \sqrt{2(1-q)\alpha} \sum_{i=1}^p \delta_i t_i^2 \rho \eta}{\sqrt{(2-\alpha)(\alpha+\beta q)} \sqrt{qn}}  +  \frac{4\sum_{i=1}^p \delta_i t_i^2 \rho \eta}{\sqrt{\alpha n}}  \\ 
	&&  +     \frac{4\sqrt{2(1-q)(1-\delta)} \sum_{i=1}^p \sqrt{L_i^2 + (q+\delta-q\delta)\delta_i^2 }  t_i^2 \rho \eta}{q + \delta -q\delta}   + \sum_{i=1}^p L_i^0 t_i^2 \eta. 
\end{eqnarray*}

Next we consider the case where $q=0$. From (\ref{eq:sumfgrad-cs}) and (\ref{eq:Mik-dG-mvr-q=0}), we can get that 
\begin{align*}
	\sum_{i=1}^p \sum_{k=0}^{K-1} t_i \eta \E \left[  \| \nabla_i f(X^k)\|_{(i)\star}  \right] \leq \Delta^0 + \sum_{i=1}^p  & \left[   \frac{2 t_i\eta\rho \sigma}{\alpha \sqrt{Bn}} +    \frac{2\sqrt{2}Kt_i \eta \rho \sigma}{\sqrt{Bn}}  + 
	\frac{2\sqrt{2} K^{\frac{3}{2}} \rho \delta_i t_i^2 \eta^2}{\sqrt{n}}       \right. \\ 
	& \quad \left.   +  \frac{2K\rho \delta_i t_i^2 \eta^2}{\sqrt{\alpha n}}  +    \frac{2K\sqrt{2(1-\delta)(L_i^2 + \delta \delta_i^2 )} \rho t_i^2 \eta^2}{ \delta}    \right. \\
	& \quad  \left. +  \frac{K L_i^0 t_i^2 \eta^2}{2} + \sum_{k=0}^{K-1}  \frac{L_i^1 t_i^2 \eta^2}{2} \E \left[  \nabla_i f(X^k)\|_{(i)\star}  \right] \right]. 
\end{align*}

Now we consider two options: (1) $L_i^1 = 0$ for all $i \in \{  1, ..., p  \}$ and (2) $L_i^1 \neq 0$, for all $i \in \{  1, ..., p  \}$.  

{\bf Case 1:} $L_i^1 = 0$ for all $i \in \{  1, ..., p  \}$. In this case, 
\begin{eqnarray*}
	&& \min_{k=0, ..., K-1} \sum_{i=1}^p t_i \E \left[  \|\nabla_i f(X^k) \|_{(i)\star} \right] \\ 
	&\leq&  \frac{1}{K} \sum_{k=0}^{K-1} \sum_{i=1}^p t_i \E \left[  \|\nabla_i f(X^k) \|_{(i)\star} \right] \\ 
	&\leq&  \frac{\Delta^0}{\eta K} + \frac{2\sum_{i=1}^p t_i \rho \sigma}{\alpha K\sqrt{Bn}}  +   \frac{2\sqrt{2} \sum_{i=1}^p t_i \rho \sigma}{\sqrt{Bn}}  +   \frac{2\sqrt{2} \sqrt{K}  \sum_{i=1}^p \delta_i t_i^2 \rho \eta}{\sqrt{n}}   +  \frac{2 \sum_{i=1}^p \delta_i t_i^2 \rho \eta}{\sqrt{\alpha n}}   \\ 
	&&  +     \frac{2\sqrt{2(1-\delta)} \sum_{i=1}^p \sqrt{L_i^2 + \delta \delta_i^2 }  t_i^2 \rho \eta}{ \delta}    + \frac{\sum_{i=1}^p L_i^0 t_i^2 \eta}{2}. 
\end{eqnarray*}

{\bf Case 2:} $L_i^1 \neq 0$, for all $i \in \{  1, ..., p  \}$. First we let $\eta \leq \min_i \frac{1}{L_i^1 t_i}$. Then $\frac{1}{2} L_i^1 t_i \eta \leq \frac{1}{2}$ for all $i$, and 
\begin{eqnarray*}
	&& \min_{k=0, ..., K-1} \sum_{i=1}^p t_i \E \left[  \|\nabla_i f(X^k) \|_{(i)\star} \right] \\ 
	&\leq&  \frac{1}{K} \sum_{k=0}^{K-1} \sum_{i=1}^p t_i \E \left[  \|\nabla_i f(X^k) \|_{(i)\star} \right] \\ 
	&\leq&  \frac{2\Delta^0}{\eta K} + \frac{4\sum_{i=1}^p t_i \rho \sigma}{\alpha K\sqrt{Bn}}   +   \frac{4\sqrt{2} \sum_{i=1}^p t_i \rho \sigma}{\sqrt{Bn}}  +    \frac{4\sqrt{2} \sqrt{K}  \sum_{i=1}^p \delta_i t_i^2 \rho \eta}{\sqrt{n}}   +  \frac{4 \sum_{i=1}^p \delta_i t_i^2 \rho \eta}{\sqrt{\alpha n}}   \\ 
	&&  +     \frac{4\sqrt{2(1-\delta)} \sum_{i=1}^p \sqrt{L_i^2 + \delta \delta_i^2 }  t_i^2 \rho \eta}{ \delta}   + \sum_{i=1}^p L_i^0 t_i^2 \eta. 
\end{eqnarray*}

\newpage

\section{Logistic Regression Ablation}
\label{appendix:logreg}

We make a detailed ablation over the learning rate, momentum parameter $\beta$, probability $q$, compression factor $K\%$ and $B$ for Algorithm \ref{alg:q-cgluon} (Figures \ref{fig:alg1_B1_logreg}, \ref{fig:alg1_B4_logreg}, \ref{fig:alg1_B16_logreg}), Algorithm \ref{alg:q-deltagluon} (Figure \ref{fig:logreg_alg2}) and VR-MARINA (Figure \ref{fig:logreg_marina_ablation})

We make a list of observations related to the Algorithms \ref{alg:q-cgluon} and \ref{alg:q-deltagluon}.

\begin{itemize}
    \item Introducing the parameter $q<1$ results in a significant slowdown in convergence compared to the SGD baseline, even without compression (see $K=100\%$). However, increasing the compression rate does not further degrade convergence speed. Consequently, for all final experiments in the main body of the paper, we employ an aggressive compression setting of $K=1\%$.
    \item Increasing the parameter $B$ enables the model to converge to a more favorable neighborhood.
    \item Across all experimental setups, a larger momentum value $\beta=0.99$ consistently yielded better performance.
    \item The greedy compressor (Top$K\%$) demonstrates no significant advantage over the random sparsifier (Rand$K\%$).
\end{itemize}

\begin{figure}[htbp!]
    \centering
    
    \includegraphics[width=1\linewidth, height=5cm]{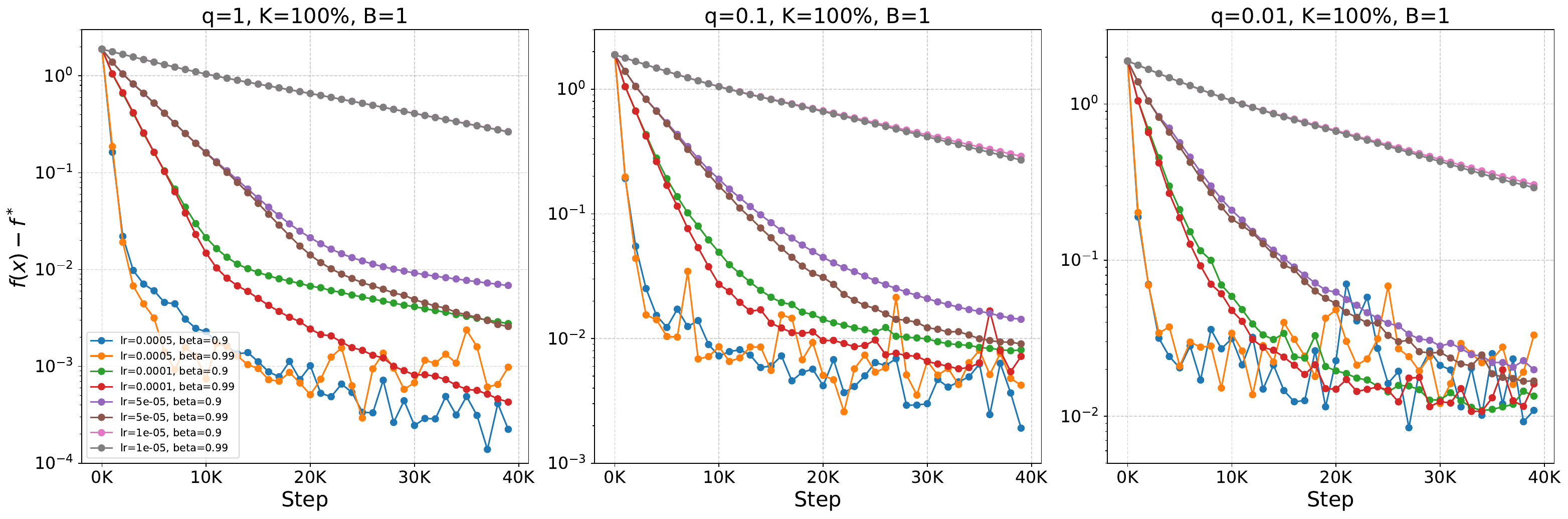}
    \vspace{0.5cm} 
    
    \includegraphics[width=1\linewidth, height=5cm]{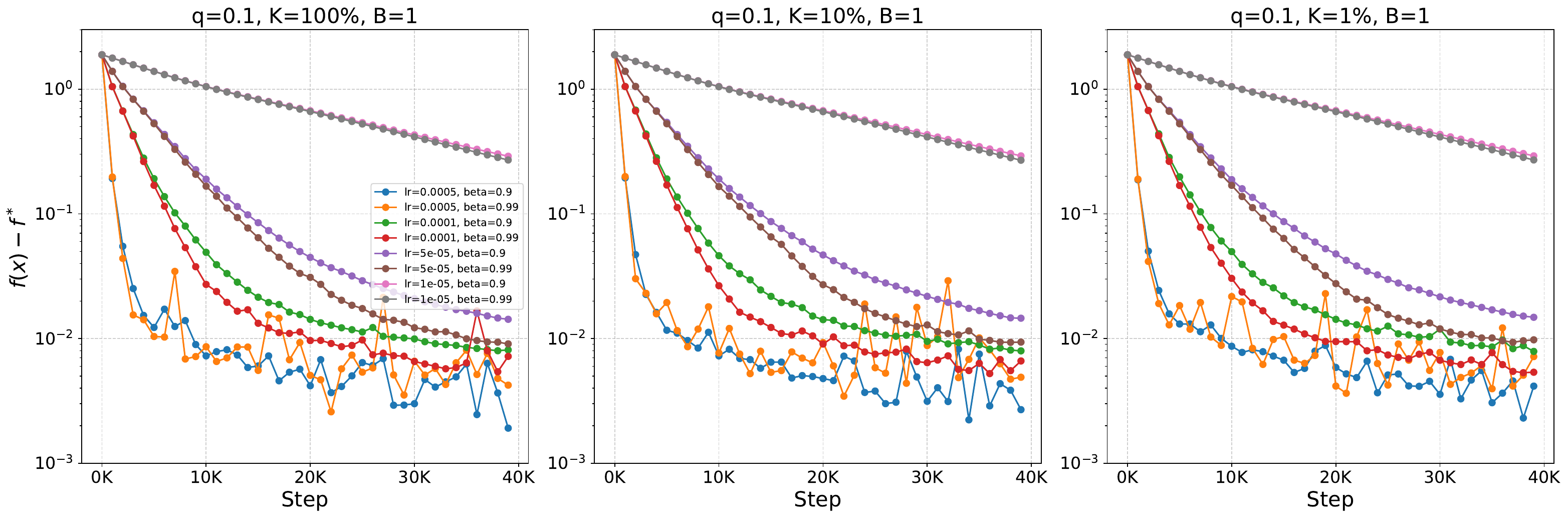} 
    \vspace{0.5cm}
    
    \includegraphics[width=1\linewidth, height=5cm]{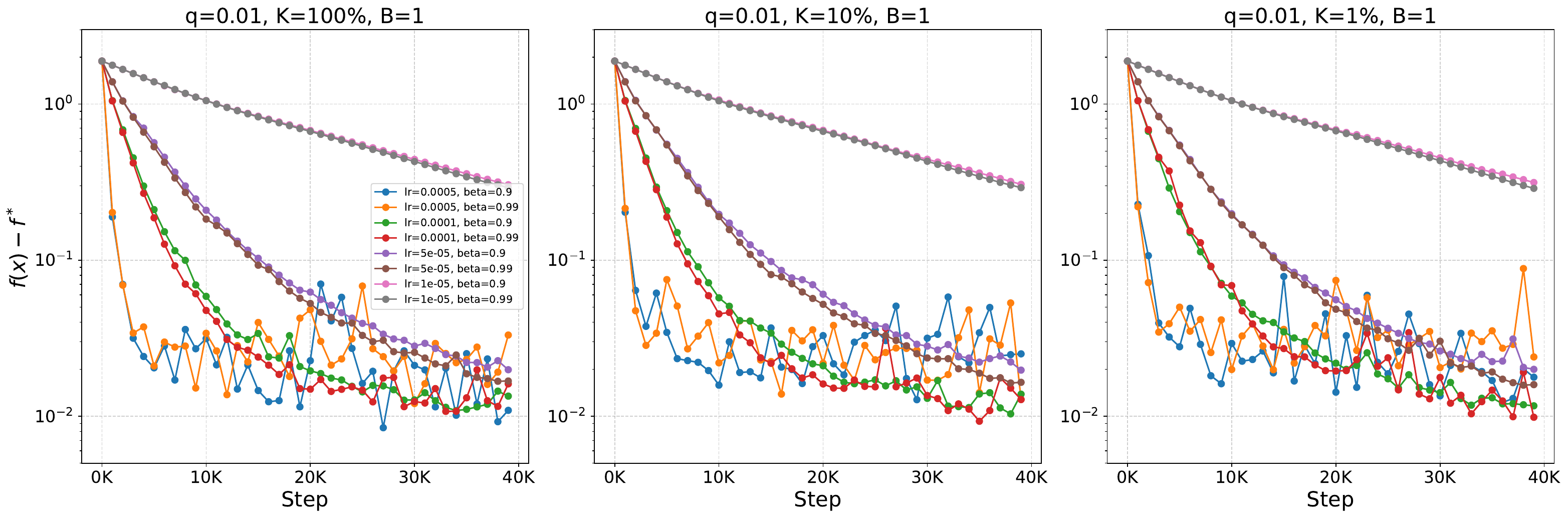}
    \caption{$f(x)-f^*$ vs. Step, Algorithm \ref{alg:q-cgluon}, Logistic Regression, a5a dataset, $B=1$ across various $\beta$, lr, compression factor $K\%$ and $q$.}
    \label{fig:alg1_B1_logreg}
    
\end{figure}

\begin{figure}[htbp!]
    \centering
    
    \includegraphics[width=1\linewidth, height=5cm]{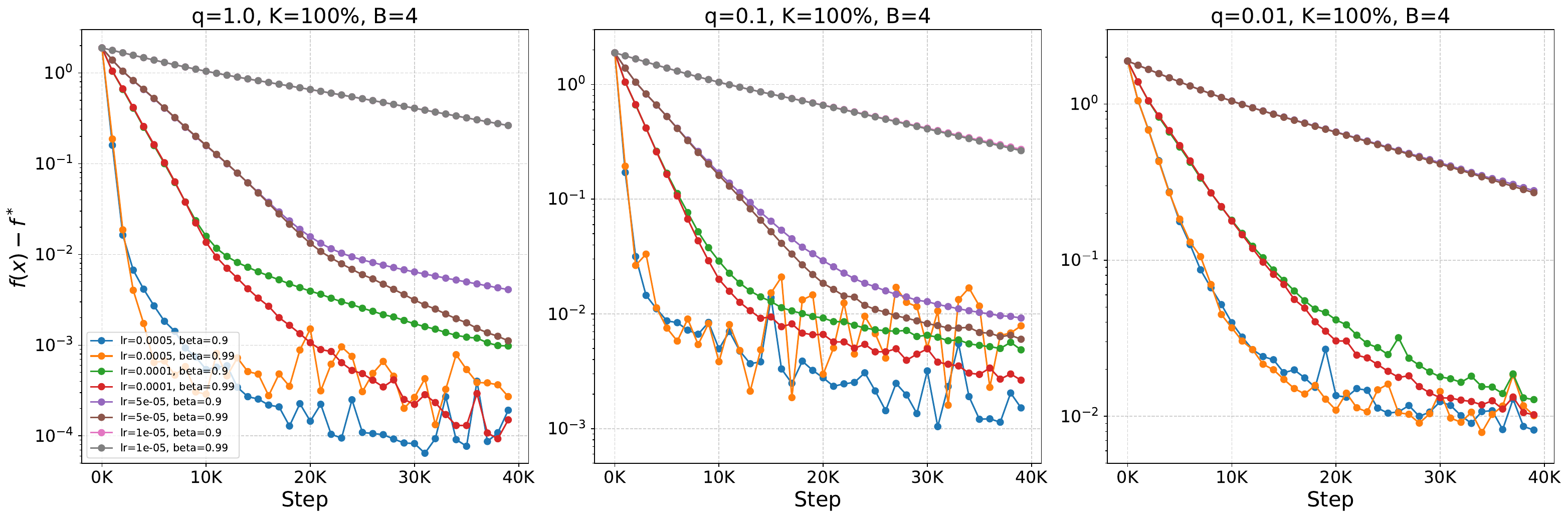}
    \vspace{0.5cm} 
    
    \includegraphics[width=1\linewidth, height=5cm]{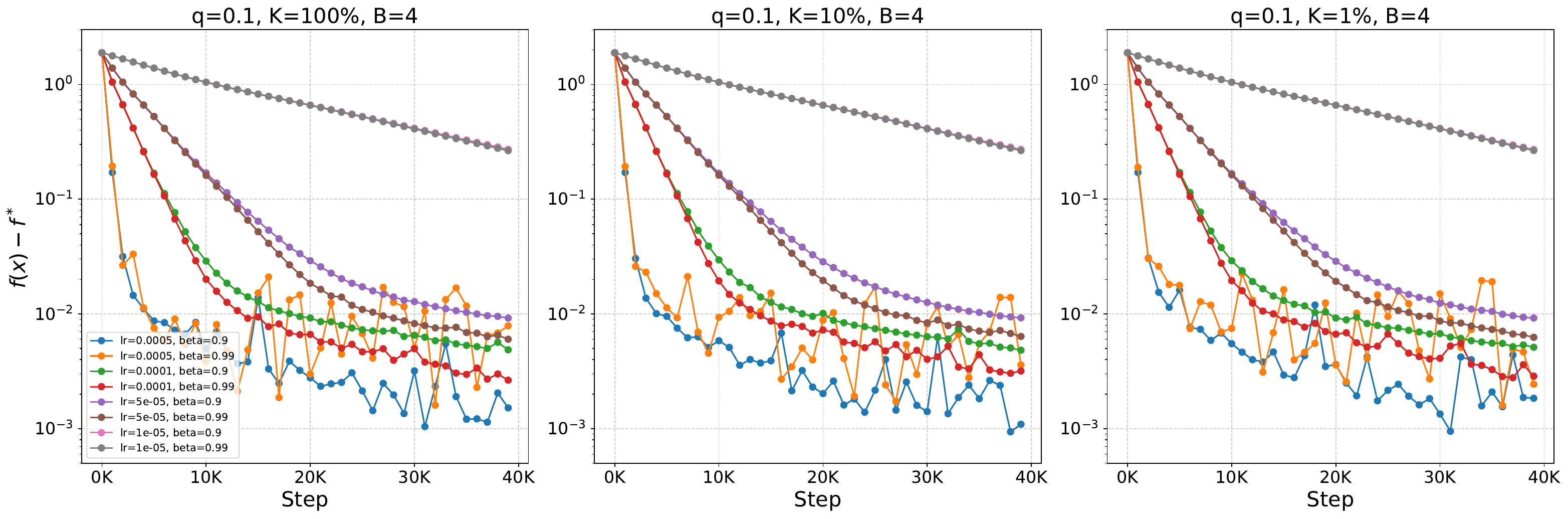} 
    \vspace{0.5cm}
    
    \includegraphics[width=1\linewidth, height=5cm]{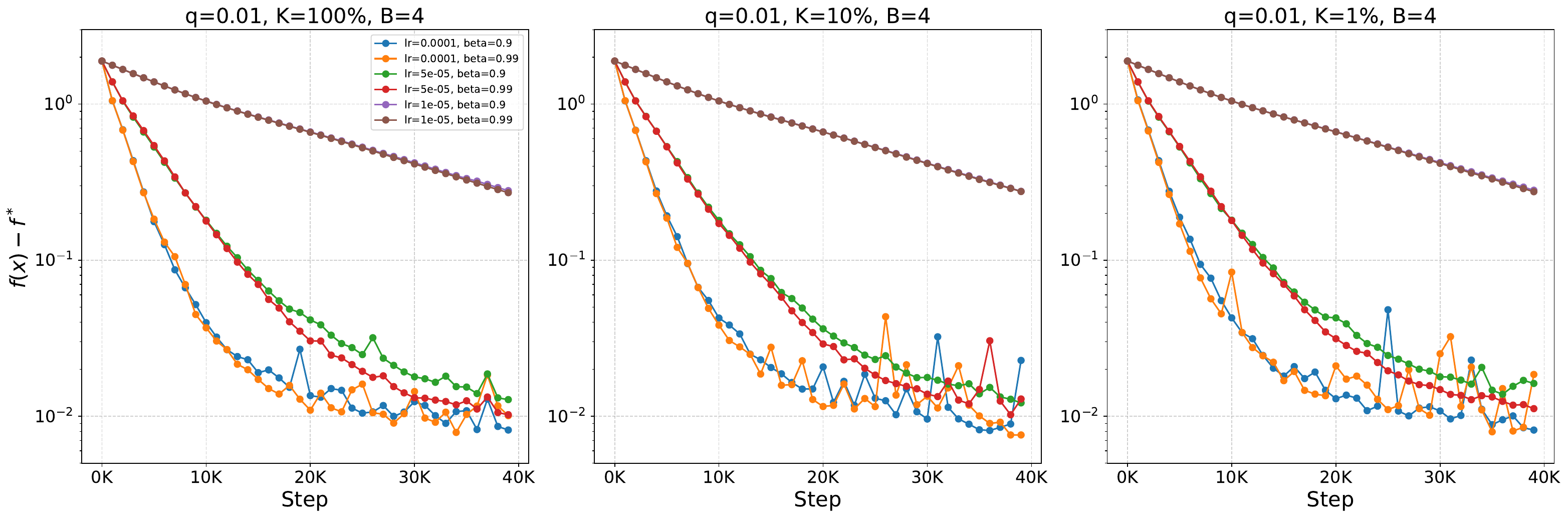}
    \caption{$f(x)-f^*$ vs. Step, Algorithm \ref{alg:q-cgluon}, Logistic Regression, a5a dataset, $B=1$ across various $\beta$, lr, compression factor $K\%$ and $q$.}
    \label{fig:alg1_B4_logreg}
    
\end{figure}

\begin{figure}[htbp!]
    \centering
    
    \includegraphics[width=1\linewidth, height=5cm]{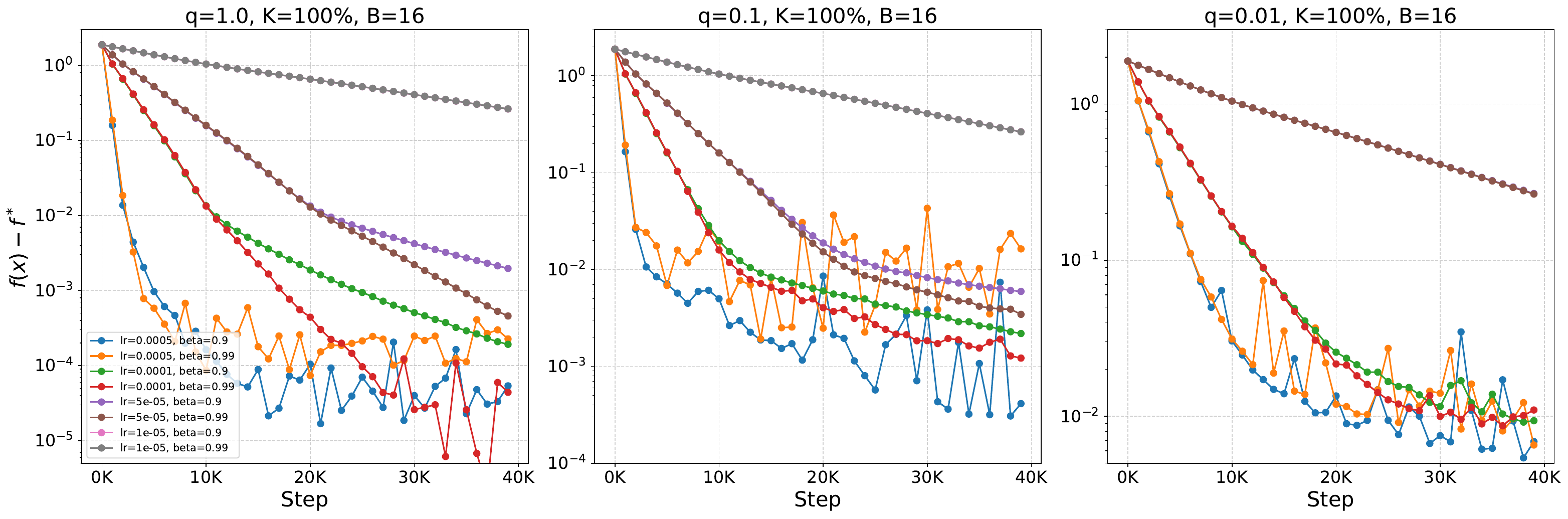}
    \vspace{0.5cm} 
    
    \includegraphics[width=1\linewidth, height=5cm]{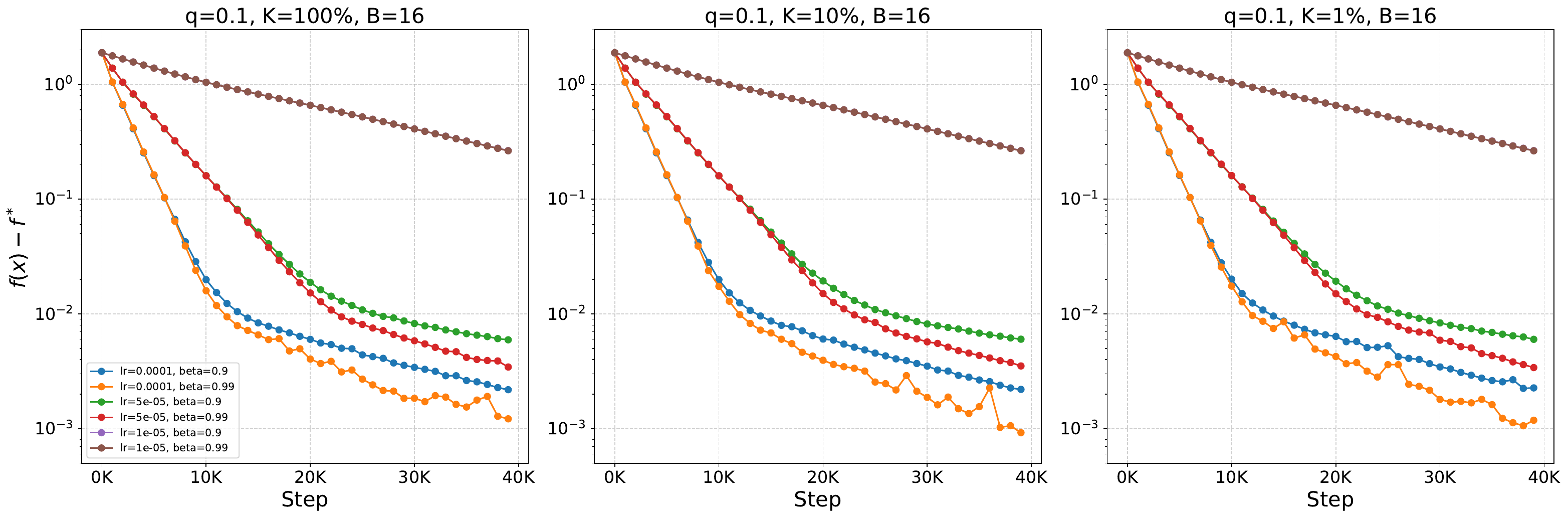} 
    \vspace{0.5cm}
    
    \includegraphics[width=1\linewidth, height=5cm]{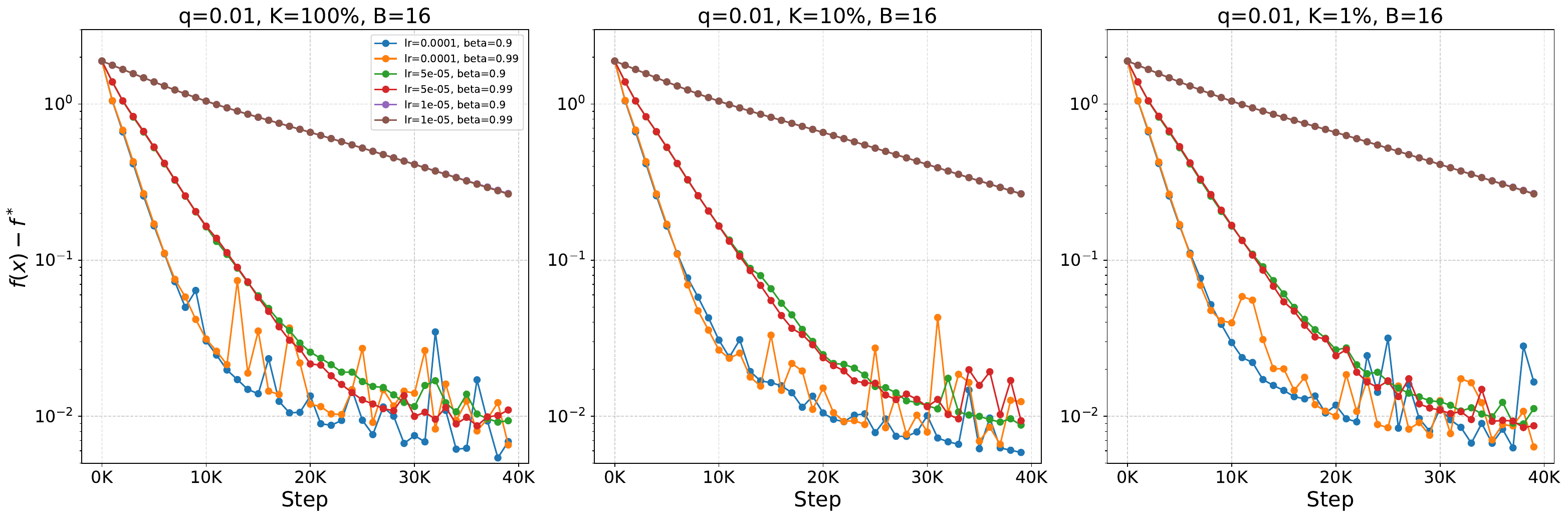}
    \caption{$f(x)-f^*$ vs. Step, Algorithm \ref{alg:q-cgluon}, Logistic Regression, a5a dataset, $B=16$ across various $\beta$, lr, compression factor $K\%$ and $q$.}
    \label{fig:alg1_B16_logreg}

\end{figure}

\begin{figure}[htbp!]
    \centering
    \includegraphics[width=0.48\linewidth]{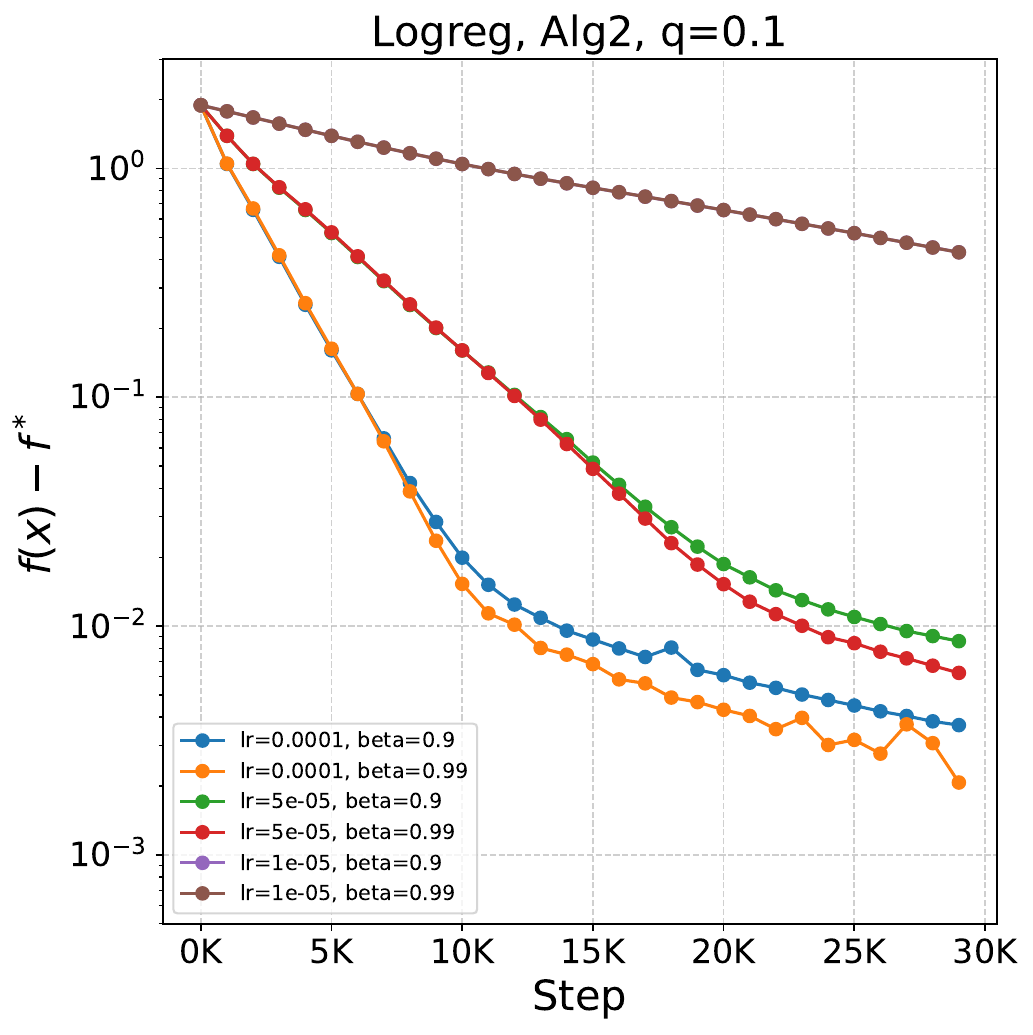}%
    \hfill
    \includegraphics[width=0.48\linewidth]{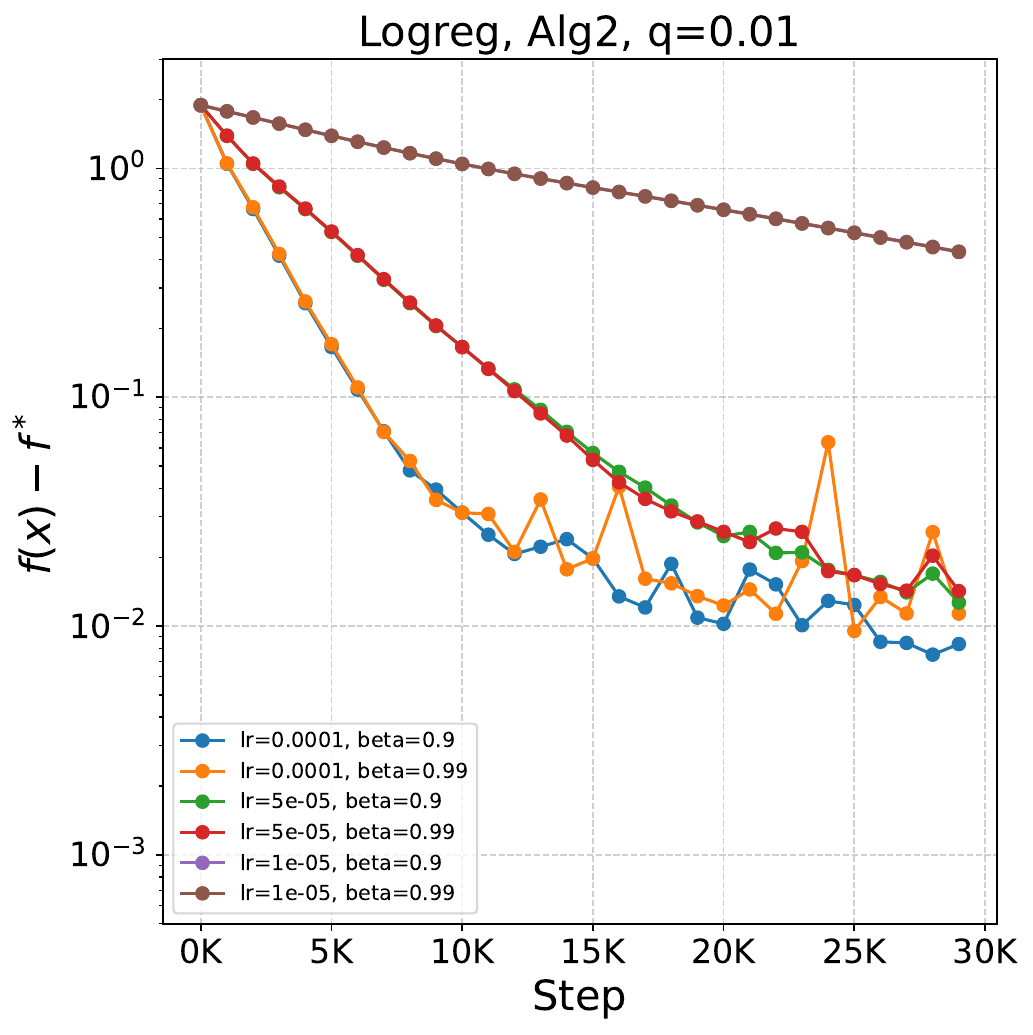}
    \caption{$f(x)-f^*$ vs. training step, Algorithm \ref{alg:q-deltagluon}, logistic regression ablation on the a5a dataset, $B=16$, Top$1\%$ compression. Only stable trajectories are presented.}
    \label{fig:logreg_alg2}
\end{figure}

\begin{figure}[htbp!]
    \centering
    \includegraphics[width=0.48\linewidth]{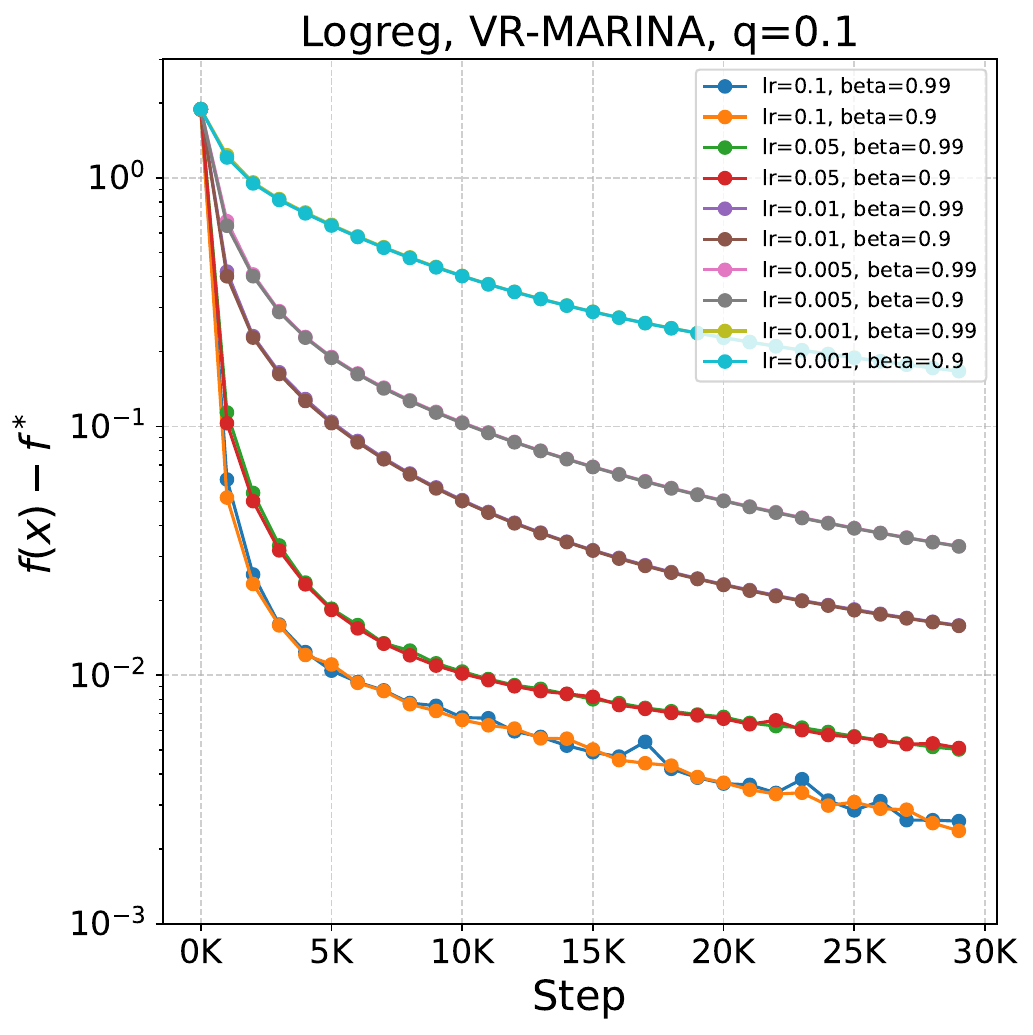}%
    \hfill
    \includegraphics[width=0.48\linewidth]{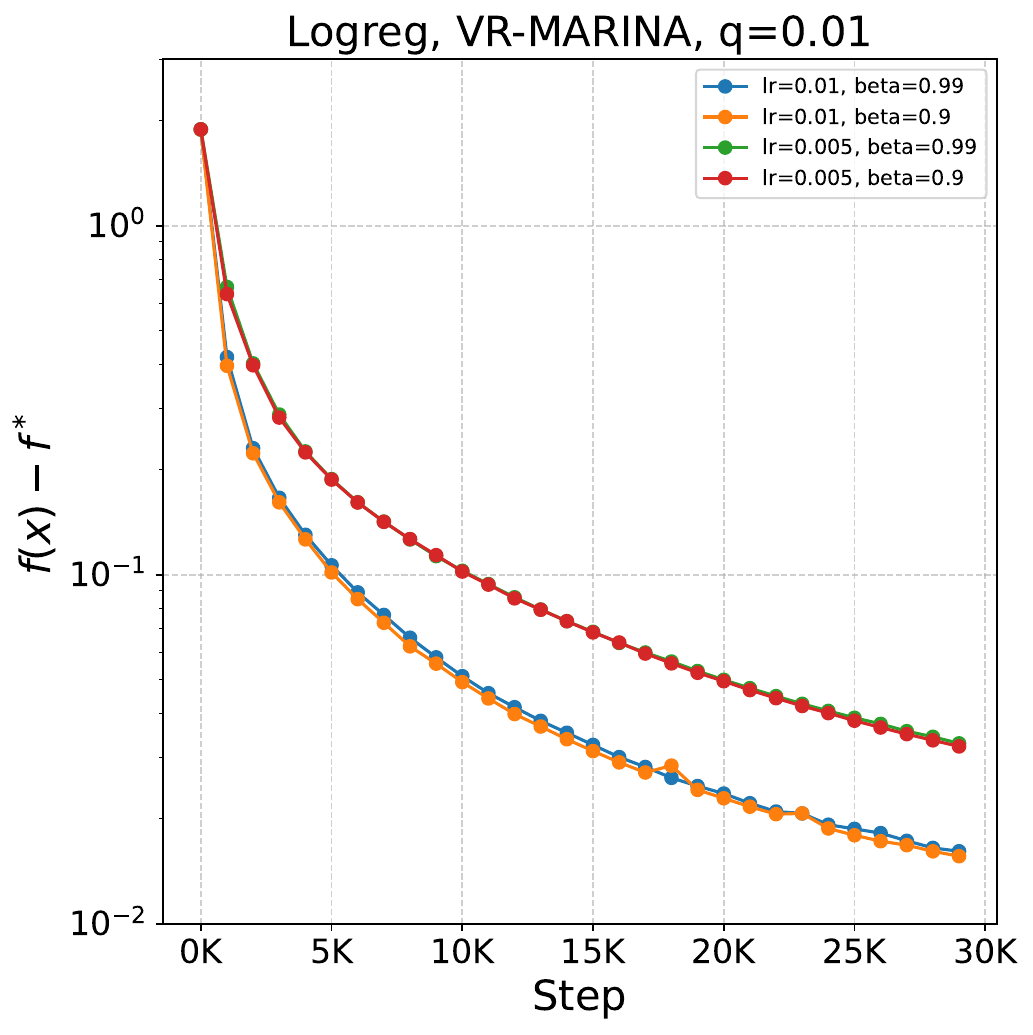}
    \caption{$f(x)-f^*$ vs. training step, VR-MARINA, logistic regression on the a5a dataset, $B=16$ (to ensure a fair comparison with our methods), Rand$1\%$ compression. Only stable trajectories are presented.}
    \label{fig:logreg_marina_ablation}
\end{figure}

\section{CNN Ablation}
\label{appendix:cifar10}

In this section, we incorporate observations from Appendix \ref{appendix:logreg} and set the compression parameter to $K=1\%$. We report only stably converging trajectories across the hyperparameter grid in Figure \ref{cifar10_ablation}. Runs with $q=0.05$ and $q=0.01$ were found not to be competitive with the baseline in terms of communication cost due to very slow convergence.

\begin{figure}[H]
    \centering
    \includegraphics[width=1.0\linewidth]{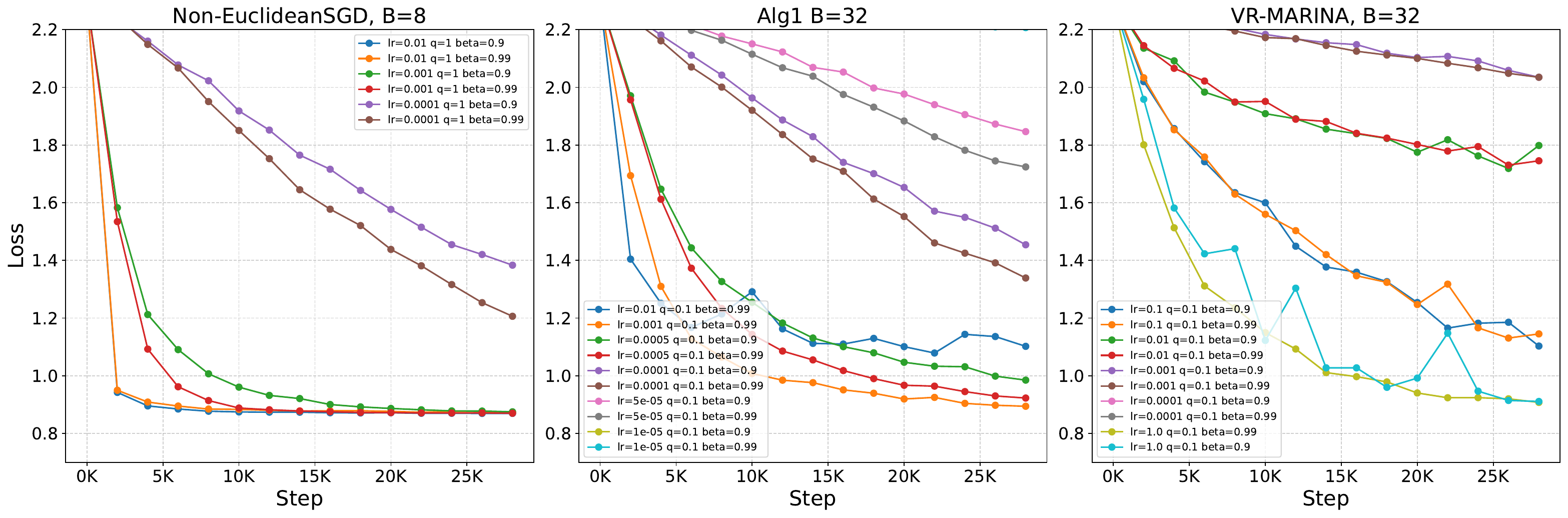}%
    \caption{Loss vs. training step, CIFAR10, compression parameter $K=1\%$, learning rate and momentum tuning for Non-Euclidean SGD baseline, Algorithm \ref{alg:q-cgluon}, and VR-MARINA.}
    \label{cifar10_ablation}
\end{figure}


\end{document}